\documentclass{article} 

\usepackage[utf8]{inputenc} 
\usepackage[T1]{fontenc}    
\usepackage{fullpage}
\usepackage{amsmath,amsfonts,amsthm,commath,amssymb,dsfont,mathtools}
\usepackage[normalem]{ulem}
\usepackage{grffile}
\usepackage{xcolor}
\usepackage{amssymb}
\usepackage{times}

\usepackage{bbm}
\usepackage{physics}
\usepackage{makecell}
\usepackage{multirow}
\usepackage{minitoc}
\usepackage{circledsteps}
\usepackage{enumitem}
\usepackage{tikz}
\usetikzlibrary{positioning, arrows.meta}
\usepackage{subfigure}
\usepackage{graphicx,subcaption}
\usepackage{float}
\usepackage{footnote}  
\usepackage[colorlinks, citecolor=blue]{hyperref} 
\usepackage{footnotebackref}
\usepackage{placeins}

\makesavenoteenv{tabular}

\usepackage[style=authoryear,maxcitenames=2,maxbibnames=99,backend=biber,sorting=nyt,natbib=true,backref=true,uniquelist=false,dashed=false]{biblatex}
\makeatletter
\newcommand\footnoteref[1]{\protected@xdef\@thefnmark{\ref{#1}}\@footnotemark}
\makeatother

\newtheorem{theorem}{Theorem}
\newtheorem{lemma}{Lemma}
\newtheorem{remark}{Remark}
\newtheorem{corollary}{Corollary}

\newcommand{\ba}{\boldsymbol{a}}
\newcommand{\A}{\boldsymbol{A}}

\newcommand{\e}{\boldsymbol{e}}
\newcommand{\x}{\boldsymbol{x}}
\newcommand{\y}{\boldsymbol{y}}
\newcommand{\bv}{\boldsymbol{v}}

\newcommand{\w}{\boldsymbol{w}}

\newcommand{\I}{\boldsymbol{I}}
\newcommand{\D}{\boldsymbol{D}}

\newcommand{\X}{\boldsymbol{X}}

\newcommand{\V}{\boldsymbol{V}}
\newcommand{\z}{\boldsymbol{z}}

\newcommand{\bu}{\boldsymbol{u}}
\newcommand{\blambda}{\boldsymbol{\lambda}}
\newcommand{\bLambda}{\boldsymbol{\Lambda}}
\newcommand{\bdelta}{\boldsymbol{\delta}}

\newcommand{\bTheta}{\boldsymbol{\Theta}}
\newcommand{\bbeta}{\boldsymbol{\beta}}
\newcommand{\bmu}{\boldsymbol{\mu}}
\newcommand{\bepsilon}{\boldsymbol{\epsilon}}
\newcommand{\zero}{\boldsymbol{0}}

\newcommand{\bSigma}{\boldsymbol{\Sigma}}

\newcommand{\balpha}{\boldsymbol{\alpha}}

\newcommand{\R}{\mathbb{R}}

\newcommand{\E}{\mathbb{E}}

\newcommand{\Risk}{\mathcal{R}}

\newcommand{\uargmin}[1]{\underset{#1}{\arg\min\,}}

\newcommand{\cond}[4]{\left\{ \begin{matrix}
      #1 &\; #2 \\
      #3 &\; #4
    \end{matrix}\right.\,}
\newcommand{\condds}[6]{\left\{ \begin{matrix}
      #1 &\; #2 \\[10pt]
      #3 &\; #4 \\[10pt]
      #5 &\; #6
    \end{matrix}\right.\,}
\newcommand{\conddd}[8]{\left\{ \begin{matrix}
      #1 &\; #2 \\
      #3 &\; #4 \\
      #5 &\; #6 \\
      #7 &\; #8
    \end{matrix}\right.\,}
\newcommand{\pts}[1]{\left(#1\right)}
\newcommand{\mts}[1]{\left[#1\right]}

\newcommand{\nnorm}[2][]{\left\|#2\right\|_{#1}}

\newcommand{\expf}[2][]{\exp^{#1}\left(#2\right)}

\newcommand{\XX}{\X\X^\top}

\newcommand{\XXi}{\left(\X\X^\top\right)^{-1}}

\newcommand{\XXX}{\X^\top\left(\X\X^\top\right)^{-1}}

\newcommand{\myquad}[1][1]{\hspace*{#1em}\ignorespaces}
\newtheorem{assumption}{Assumption}
\begin{document}
\setcounter{assumption}{-1}

\doparttoc 
\faketableofcontents 

\begin{center}

    {\bf{\LARGE{How Does the ReLU Activation Affect the Implicit Bias of Gradient Descent on High-dimensional Neural Network Regression?}}}
    
    \vspace*{.2in}

    \renewcommand{\thefootnote}{\fnsymbol{footnote}}
    {\large{
    \begin{tabular}{cccc}
    Kuo-Wei Lai$^{1}$\footnote{\label{foot:foot1}Equal contribution; co-first author.} & Guanghui Wang$^2$\footnoteref{foot:foot1} & Molei Tao$^{3}$ & Vidya Muthukumar$^{1,4}$
    \end{tabular}}}
    
    \vspace*{.2in}
    
    \renewcommand{\thefootnote}{\arabic{footnote}}
    \setcounter{footnote}{0}
    \begin{tabular}{c}
        $^1$School of Electrical and Computer Engineering, Georgia Institute of Technology\\
        $^2$College of Computing, Georgia Institute of Technology\\
        $^3$School of Mathematics, Georgia Institute of Technology\\
        $^4$H. Milton Stewart School of Industrial and Systems Engineering, Georgia Institute of Technology
        \\ \texttt{\{klai36, gwang369, mtao, vmuthukumar8\}@gatech.edu}
    \end{tabular}

\end{center}

\begin{abstract}%
Overparameterized ML models, including neural networks, typically induce underdetermined training objectives with multiple global minima.
The \emph{implicit bias} refers to the limiting global minimum that is attained by a common optimization algorithm, such as gradient descent (GD).
In this paper, we characterize the implicit bias of GD for training a shallow ReLU model with the squared loss on high-dimensional random features. 
Prior work~\citep{vardi2021implicit} showed that the implicit bias does not exist in the worst-case, or corresponds exactly to the minimum-$\ell_2$-norm interpolating solution under \emph{exactly} orthogonal data~\citep{boursier2022gradient}.
Our work interpolates between these two extremes and shows that, for sufficiently high-dimensional random data, the implicit bias approximates the minimum-$\ell_2$-norm solution with high probability with a gap on the order $\Theta\left(\sqrt{n/\|\blambda\|_1}\right)$, where $n$ is the number of training examples and $\blambda$ denotes the spectrum of the data covariance matrix.
Our results are obtained through a novel primal-dual analysis that carefully tracks the evolution of predictions, data-span coefficients, as well as their interactions, and show that the ReLU activation pattern quickly stabilizes with high probability over random data.

\end{abstract}


\section{Introduction}
In many modern machine learning problems, 
the training objectives are typically \emph{underdetermined}, which implies that they may admit (potentially infinitely) many global minima. Despite this, a large body of empirical results \citep{neyshabur2014search,zhang2021understanding} show that optimization algorithms such as gradient descent frequently converge to solutions that generalize well, even in the absence of any explicit regularization. This phenomenon is commonly referred to as the \emph{implicit bias} introduced by gradient descent \citep{ji2018risk,soudry2018implicit}, and understanding the nature of this benign bias has become a central topic of recent research.

The study of the implicit bias of gradient descent originally emerged in the context of linear models. 
For linear classification with separable data, the seminal work of~\citet{soudry2018implicit,ji2018risk} shows that, when minimizing exponentially-tailed losses, gradient descent converges in direction to the max-margin solution that minimizes $\ell_2$-norm. 
For linear regression with the squared loss, gradient descent is known to converge to the zero-loss (interpolating) solution with the minimum-$\ell_2$-norm~\citep{engl1996regularization}. 
Building on these foundational results, several finer characterizations were derived for linear models, including sharper convergence analyses \citep{ji2018risk,nacson2019convergence,ji2021characterizing}, general classes of first-order methods~\citep{gunasekar2018characterizing,sun2022mirror,wang2025faster}, and a deeper understanding of the implicit bias on high-dimensional data~\citep{hsu2021proliferation,lai2025general}. 

Understanding the implicit bias in nonlinear models such as neural networks remains a significant challenge, primarily due to the induced non-convexity of the optimization objective. 
In this work, we focus on regression with a one-hidden-layer ReLU neural network and the squared loss, which represents one of the most fundamental and natural extensions beyond linearity. Remarkably,~\citet{vardi2021implicit} showed that establishing the implicit bias of ReLU models is known to be hard in the worst case, even when the model consists only of a single neuron and assuming global convergence. (On the other hand, the implicit bias of a single neuron with a strictly monotonic activation function (e.g., leaky ReLU) does follow the minimum-$\ell_2$-norm solution.)
They do this through a stylized counterexample of $3$ data points with $3$-dimensional features, raising the natural question of whether the implicit bias becomes characterizable under typical data ensembles.
At the other extreme,~\citet{boursier2022gradient} showed that the implicit bias of gradient flow for \emph{exactly} orthogonal features is exactly the minimum-$\ell_2$-norm solution.
However, an exact orthogonality assumption is restrictive and rarely holds in practice. 
Notably, high-dimensional random features are \emph{near}-orthogonal, raising the question of whether the implicit bias can be characterized in this more realistic but also more challenging case.

\paragraph{Our contributions:}  In this paper, we establish new insights into the implicit bias induced by gradient descent for ReLU networks trained with the squared loss on sufficiently high-dimensional data. Our main contributions are summarized as follows. First, we completely characterize the expression for the implicit bias of gradient descent dynamics on ReLU models with 1 or 2 neurons for high-dimensional data under sufficient conditions (Theorems~\ref{thm:single_relu_gd_high_dim_implicit_bias} and~\ref{thm:multiple_relu_gd_high_dim_implicit_bias}). Second, we quantify the relationship between the implicit bias of gradient descent and the global minimum that achieves the minimum-$\ell_2$-norm. More specifically, we establish both upper and lower bounds on the distance between the gradient descent limit and the minimum-$\ell_2$-norm solution, showing that it scales as $\Theta(\sqrt{{n}/{\nnorm[1]{\blambda}}})$ where $n$ is the number of training examples and $\blambda$ denotes the spectrum of the data covariance matrix (Theorems~\ref{thm:single_relu_approx_to_w_star} and~\ref{thm:multiple_relu_approx_to_w_star}). Consequently, the solutions are very close, but not identical, for high-dimensional features. Interestingly, a similar phenomenon was also shown to occur with exponentially-tailed losses~\citep{frei2023benign,frei2023implicit} for classification. 

\paragraph{Our techniques in a nutshell:} Our main results are obtained through a novel primal-dual formulation of the gradient descent dynamics under the squared loss with ReLU networks,  which is inspired by mirror descent (first studied by~\citet{ji2021characterizing} for linear models). Instead of directly tracking the weight vector in the original parameter space like previous work, we introduce primal variables representing the predictions on training examples, and dual variables capturing the coefficients in the data span. This representation is particularly well-suited for analyzing ReLU networks because the sign of each primal variable directly determines whether the corresponding example is active, and hence whether its dual variable receives a gradient update. Our analysis reveals that understanding the gradient dynamics hinges on tracking (i) the positivity of the primal variables and (ii) the interactions between training examples. We introduce new tools to carefully control the evolution of positive primal variables and sufficiently negative dual variables (Lemmas~\ref{lem:pos_stay_pos} and~\ref{lem:neg_stay_neg}, which may be of independent interest). Underlying the proofs of our approximation results to the minimum-$\ell_2$-norm solutions are novel characterizations of the latter as minimum-$\ell_2$-norm \emph{linear} interpolations of a (possibly data-dependent) subset of training examples. This data-dependent subset selection is a fundamental difference between the implicit bias of linear models and ReLU models.

\subsection{Related Work}
We now briefly discuss the most closely related work and highlight key differences of our approach. We contextualize our results within the most closely related prior studies on implicit bias of regression models in Table~\ref{tab:comparison}.  \citet{boursier2022gradient} study the dynamics of gradient flow on two-layer ReLU networks under an \emph{exact} orthogonality assumption on the data. Exact orthogonality removes interactions between examples and significantly simplifies the activation patterns induced by the ReLU nonlinearity. As a result, their analysis primarily focuses on how the second-layer weights evolve to fit all examples, leading to a multi-phase gradient flow dynamic. Under these assumptions, they show that gradient flow converges to the minimum-$\ell_2$-norm solution (their Theorem 1). In contrast, our work focuses on understanding how interactions between examples, captured through the Gram matrix, shape the active and inactive patterns in ReLU models under more realistic, controllable high-dimensional settings. 
Interestingly, we show that in high dimensions, the implicit bias is no longer exactly the minimum-$\ell_2$-norm solution but remains close to it (Theorems~\ref{thm:single_relu_approx_to_w_star} and~\ref{thm:multiple_relu_approx_to_w_star}). In comparison, \citet{vardi2021implicit} provide only a multiplicative upper bound on the magnitude of the norm of implicit bias in the worst-case data setting, showing that it is at most \emph{twice} that of the minimum-norm solution. \citet{dana2025convergence} also analyze the high-dimensional regime and establish global convergence by showing that each example can be fitted by at least one neuron with high probability and all active examples stay active (their Theorem 1).
However, their analysis does not address the behavior of inactive examples suppressed by the ReLU nonlinearity and does not shed light on the implicit bias. 
As a result, their work provides only a partial view of the gradient dynamics. In contrast, we introduce a novel primal-dual framework that allows us to simultaneously track both active and inactive examples (Lemmas~\ref{lem:pos_stay_pos} and~\ref{lem:neg_stay_neg}). This framework enables a full characterization of the gradient dynamics and, consequently, the implicit bias in high dimensions.
We use some of the observations of~\citet{dana2025convergence} as a starting point for our primal-dual characterizations. More generally, most existing analyses~\citep{vardi2021implicit,boursier2022gradient,dana2025convergence} rely on gradient flow and continuous-time ODE techniques, which assume infinitesimal step sizes. In contrast, our analysis directly studies gradient descent with finite (though still small) step sizes. This distinction is both theoretically and practically important, as gradient descent is the algorithm used in practice. Our primal-dual approach provides a new framework for analyzing discrete-time optimization dynamics in ReLU networks and opens a complementary direction to existing studies based on gradient flow. 

\citet{frei2023benign,frei2023implicit} consider classification in a similar one-hidden-layer setup with the leaky ReLU activation, and also exploit simplified KKT conditions under nearly orthogonal data. However, due to the different training objectives underlying classification and regression, the resulting analyses are fundamentally different. 
\begin{table}[t]

    \centering
    \resizebox{1\textwidth}{!}{\LARGE\begin{tabular}{|c | c | c | c|} 
    \hline
 & Orthogonal data & High dimensional data & Worst-case data \\
    \hline
    \multirow{3.5}{*}{\begin{tabular}{c}ReLU models ($k=1$)\\$h (\x) \coloneqq \sum_{k=1}^m s_k\sigma(\w_k^\top\x)$\end{tabular}} & \multirow{2}{*}{\begin{tabular}{c}Implicit bias characterization\\ \citep{boursier2022gradient} \\ $\w^{(\infty)} = \underset{\w\in\{\w:\X\w=\y\}}{\arg\min} \nnorm[2]{\w - \w^{(0)}}$\end{tabular}} & \begin{tabular}{c}Global convergence only\\ \citep{dana2025convergence} \end{tabular} & \multirow{3.5}{*}{\begin{tabular}{c}No implicit bias in general\\ \citep{vardi2021implicit} \\ $\nnorm[2]{\w^{(\infty)} - \w^{(0)}} \leq 2\cdot \nnorm[2]{\w^\star - \w^{(0)}}$ \end{tabular}} \\ \cline{3-3}
    & & \begin{tabular}{c} {\color{red}This work} \\ {\color{red}$\nnorm[2]{\w^{(\infty)} - \w^\star}\asymp\sqrt{\frac{n}{\nnorm[1]{\blambda}}}$}\end{tabular} &\\
    \hline\begin{tabular}{c}Linear models\\$h (\x) \coloneqq \w^\top\x$\end{tabular} & \multicolumn{2}{ c |}{\begin{tabular}{c}Implicit bias coincides with \\ maximum $\ell_2$ margin SVM \citep{hsu2021proliferation}\end{tabular}} & \begin{tabular}{c}$\w^{(\infty)} = \underset{\w\in\{\w:\X\w=\y\}}{\arg\min} \nnorm[2]{\w - \w^{(0)}}$\\ \citep{engl1996regularization}\end{tabular} \\
    \hline
    \end{tabular}}
    \renewcommand{\arraystretch}{1}
    \caption{Our results contextualized with related literature.
    }
    \label{tab:comparison}
\vspace{0cm}  
\end{table}

\paragraph{Notation:} 
We use lowercase boldface letters (e.g.~$\x$) to denote vectors, lowercase letters (e.g.~$y$) to denote scalars, and uppercase boldface letters (e.g.~$\X$) to denote matrices. We use $\nnorm[p]{\cdot}$ to denote the $\ell_p$-norm of a vector for $p \in [1,\infty)$ and $\nnorm[2]{\cdot}$ to additionally denote the operator norm of a matrix. For a vector $\x\in\R^d$, we use $x_i$ to denote its $i$-th component. We use $[n]$ to denote the set $\{1,\ldots,n\}$. For a matrix $\X\in\R^{n\times d}$, a vector $\y\in\R^n$, and any index set $S \subseteq [n]$, we use $\X_S \in \R^{|S| \times d}$ to denote the submatrix of $\X$ consisting of the rows indexed by $S$, and $\y_S \in \R^{|S|}$ denotes the corresponding subvector. We use $\x \preceq \zero$ (respectively $\x \succeq \zero$) to denote that every entry of vector $\x$ is less than or equal to (respectively greater than or equal to) zero.
We use $C,c$ to denote universal constants that appear in upper and lower bounds, respectively, that may change from line to line.
We also use the notation $C_{(\cdot)}$ to denote universal constants with a specific meaning that \emph{do not} change from line to line. We specifically choose $C_0\gtrsim C_{\alpha}^2$ and $C_{\alpha}\gtrsim \max\{C_g^2, C_yC_g\}$ in our analysis.

\section{Problem Setup}\label{sec:setup}

We consider a regression problem with feature vector $\x \in \mathcal{X} \subset \R^d$ and label $y \in \mathcal{Y} \subset \R$.
We consider random feature vectors drawn according to a distribution $\mathcal{D}$ with zero mean, i.e., $\E[\x] = \zero$, and covariance matrix $\bSigma =\E[\x\x^\top]$. Let $\bSigma = \V\bLambda\V^\top \in \R^{d \times d}$ be the eigendecomposition of the covariance matrix, where $\V\in\R^{d\times d}$ is the matrix of eigenvectors and $\bLambda\in\R^{d\times d}$ is a diagonal matrix whose entries are the eigenvalues of $\bSigma$, arranged in descending order. We make the mild assumption that the feature vector admits the representation $\x = \V \bLambda^{\frac{1}{2}}\z$ where $\z\in \R^d$ has independent, mean-zero, $\sigma_z^2$-subgaussian components. By the definition of a $\sigma_z^2$-subgaussian random variable, each coordinate $z_j$ satisfies
$
    \E\bigl[\exp(\bu^\top\z)\bigr] \leq {\exp}(\sigma_z^2\nnorm[2]{\bu}^2/2)
$
for any $\bu\in\R^d$.
For simplicity, we set $\sigma_z = 1$ throughout the remainder of the analysis. The labels $y$ are only required to be bounded (see Assumption~\ref{asm:label_bounds}) and may be chosen arbitrarily. In particular, $y$ does not need to satisfy any particular relationship with respect to $\x$.

We observe a dataset $\{\x_i, y_i\}_{i=1}^n$ where the features $\{\x_i\}_{i=1}^n$ are drawn i.i.d. from the distribution $\mathcal{D}$. We denote the data matrix by $\X \in \R^{n \times d}$ and the label vector by $\y \in \R^n$. Since we operate in a high-dimensional regime ($d > n$), we make the mild assumption that $\X$ has full row rank, i.e., $\mathrm{rank}(\X)=n$, which is automatically satisfied under the assumptions of all lemmas and theorems in this paper.\footnote{The full-rank assumption holds either almost surely or with high probability under random and high-dimensional features; see, e.g.~\citet{hsu2021proliferation}.}
\begin{assumption}[Full-rank Data Matrix]\label{asm:row_rank}
    The data matrix $\X \in \R^{n \times d}$ satisfies $\mathrm{rank}(\X) = n$.
\end{assumption}

For ease of subsequent notation, we consider without loss of generality the samples with positive labels to appear in the upper block of the data matrix $\X$, while those with negative labels appear in the lower block. Let $n_+$ denote the number of positive labels and $n_- = n - n_+$ denote the number of negative labels.
Accordingly, we write $\X = \begin{bmatrix} \X_+^\top & \X_-^\top\end{bmatrix}^\top$ where $\X_+ \in \R^{n_+ \times d}$ contains the features corresponding to positive labels and $\X_- \in \R^{n_- \times d}$ contains the features corresponding to negative labels. We similarly partition the label vector as $\y = \begin{bmatrix} \y_+^\top & \y_-^\top\end{bmatrix}^\top$.

Next, we introduce our key assumptions on the features and labels. First, we assume that the magnitudes of the labels are bounded away from zero and infinity.
\begin{assumption}[Bounded  Labels]\label{asm:label_bounds}
    For all $i \in [n]$, the labels satisfy $y_{\min} \leq |y_i| \leq y_{\max}$ for some $y_{\min}, y_{\max} \in \R_+$.
\end{assumption}
This assumption ensures that all labels are non-degenerate and have comparable scales, which will be important for controlling the dynamics of gradient-based optimization.

We next impose a high-dimensional assumption on the data features. To characterize the effective dimensionality of the feature distribution, we define two notions of effective dimension based on the spectrum $\blambda\coloneqq\mts{\lambda_1,\cdots,\lambda_d}^\top$ of the feature covariance matrix $\bSigma$, given by $d_2 \coloneqq \frac{\nnorm[1]{\blambda}^2}{\nnorm[2]{\blambda}^2}, d_{\infty} \coloneqq \frac{\nnorm[1]{\blambda}}{\nnorm[\infty]{\blambda}}$.
Note that when the covariance is isotropic, i.e., $\lambda_1=\lambda_2=\cdots=\lambda_d=1$, we have $d_2 =d_{\infty} = d$, i.e., these reduce to the original data dimension. Our high-dimensional assumption requires that these effective dimensions dominate problem-dependent quantities involving the sample size $n$ and the range of label magnitudes $[y_{\min},y_{\max}]$. Similar conditions have also appeared in related global convergence analysis under the squared loss \citep{dana2025convergence} and implicit bias analyses under the logistic/exponentially-tailed losses~\citep{frei2023benign}.
\begin{assumption}[High-dimensional Features]\label{asm:high_dim_data}
    The data features satisfy $d_2 \geq C_0^2\frac{n^2 y_{\max}^2}{y_{\min}^2}$ and $d_{\infty} \geq C_0\frac{n^{1.5} y_{\max}}{y_{\min}}$ for a sufficiently large constant $C_0 > 1$.
\end{assumption}
This assumption places the problem in a sufficiently high-dimensional regime, ensuring strong concentration properties of the empirical Gram matrix and enabling precise control of the gradient dynamics analyzed in the following sections. We note that our techniques would yield similar guarantees on the implicit bias for deterministic feature vectors that satisfy a near-orthogonality condition adapted from~\cite{frei2023benign,frei2023implicit} to real-valued labels.
\begin{assumption}[Deterministic Nearly-orthogonal Features]\label{asm:deterministic_data}  For a sufficiently large constant\\ $C_0 > 1$, the data features satisfy $\min_{i \in [n]} \nnorm[2]{\x_i}^2 \geq C_0 n\frac{y_{\max}}{y_{\min}}\max_{i\neq j}|\x_i^\top\x_j|$.
\end{assumption}
This condition ensures that the desired concentration properties of the empirical Gram matrix $\XX$ hold for our analysis.
%
\paragraph{General ReLU Models and Empirical Risk Minimization.} We denote by $h_{\bTheta} : \mathcal{X} \rightarrow \mathcal{Y}$ the general ReLU model used for the regression task in this work, defined as $h_{\bTheta} (\x) \coloneqq \sum_{k=1}^m s_k\sigma(\w_k^\top\x)$, where $\bTheta$ denotes the collection of model parameters $\{\w_k\}_{k=1}^m$ together with fixed signs $\{s_k\}_{k=1}^m$. Here, $s_k\in\{-1, +1\}$ denotes the sign of the $k$-th ReLU neuron, $\sigma(z) \coloneqq \max\{z, 0\}$ is the ReLU activation function, $\w_k\in\R^d$ is its weight vector, and $m\geq 1$ is the number of neurons. To learn the regression model, we minimize the empirical risk under the squared loss, defined as
\begin{align}\label{eq:empirical_risk}
    \Risk(\bTheta) = \frac{1}{2} \sum_{i=1}^n \bigl(h_{\bTheta}(\x_i) - y_i\bigr)^2 = \frac{1}{2} \nnorm[2]{h_{\bTheta}(\X) - \y}^2,
\end{align}
where we define the vector-valued extension $h_{\bTheta}$ as $h_{\bTheta}(\X)\coloneqq [h_{\bTheta}(\x_1),\cdots,h_{\bTheta}(\x_n)]^\top\in \R^n$.
We employ the gradient descent algorithm to minimize~\eqref{eq:empirical_risk}.
To make the dynamics more tractable, we only update the neuron weights $\{\w_k\}_{k=1}^m$ and fix the signs of the neurons $\{s_k\}_{k=1}^m$\footnote{This is a reasonable approximation for the dynamics when both layers are trained via the well-known \emph{balancedness} condition, but balancedness is typically formally shown only under gradient flow (see, e.g.~\citealt[Theorem 2.1]{du2018algorithmic}).}. When there are $m > 1$ neurons, we will initialize at least one neuron for a positive sign and one neuron for a negative sign to ensure that the neural network can fit arbitrary labels.


\paragraph{Gradient Descent and Primal-dual Representation.}
For the ReLU model $h_{\bTheta}$, the gradient of the empirical risk in~\eqref{eq:empirical_risk} with respect to $\w_k$ is given by
\begin{align*}
    \nabla_{\w_k} \Risk(\bTheta)
    =\sum_{i=1}^n \bigl(h_{\bTheta}(\x_i) - y_i\bigr) s_k \cdot \mathbbm{1}_{\w_k^\top \x_i > 0} \cdot \x_i
    = s_k\X^\top \D(\X\w_k)\bigl(h_{\bTheta}(\X) - \y\bigr),
\end{align*}
where $\D(\z) : \R^n\rightarrow \R^{n\times n}$ denotes the diagonal matrix with entries $D_{ii} \coloneqq \mathbbm{1}_{z_i > 0}$. Accordingly, the gradient descent update for $\w_k$ takes the form
\begin{align}
    \w_k^{(t+1)}
    &= \w_k^{(t)} - \eta \nabla_{\w_k} \Risk(\bTheta^{(t)}) = \w_k^{(t)} - \eta s_k\X^\top \D(\X\w_k^{(t)})\bigl(h_{\bTheta^{(t)}}(\X) - \y\bigr).\label{eq:gd_update_wk}
\end{align}

To analyze these updates more transparently, we introduce a primal-dual representation used in mirror descent~\citep{ji2021characterizing}. For all $k\in[m]$, we define the primal variable $\bbeta_k\in\R^n$ and the dual variable $\balpha_k\in\R^n$ as
\begin{align}
    \bbeta_k \coloneqq \X \w_k, \myquad[2] \balpha_k \coloneqq \XXi\X\w_k, \quad\text{and note that}\quad \bbeta_k = \XX \balpha_k. \label{eq:primal_dual_def}
\end{align}
This representation restricts attention to the components of $\w_k$ that lie in the span of the data matrix $\X$.\footnote{In general, $\w_k$ may contain components orthogonal to $\mathrm{span}(\{\x_i\}_{i=1}^n)$, i.e., $\w_k = \X^\top\balpha_k +\sum_{j=n+1}^d \tilde{\alpha}_{k,j}\tilde{\x}_j$ where $\tilde{\alpha}_{k,j}\in\R$ and we define the vector $\tilde{\x}_j\perp\x_i$ for all $i\in[n]$ such that $\{\x_i\}_{i=1}^n \cup \{\tilde{\x}_j\}_{j=n+1}^d$ forms a complete basis of $\R^d$. However, since the gradient updates act only within $\mathrm{span}(\{\x_i\}_{i=1}^n)$, the orthogonal components remain unchanged throughout training. Our results can be easily extended by adding back in this orthogonal component.} For ease of notation, we further define $\bbeta_{k,+}\coloneqq \X_+\w_k$ and decompose the dual variable as $\balpha_k \coloneqq \begin{bmatrix} \balpha_{k,+}^\top & \balpha_{k,-}^\top \end{bmatrix}^\top$, consistent with the partition on labels $\y = \begin{bmatrix} \y_+^\top & \y_-^\top\end{bmatrix}^\top$. Under this parameterization, the gradient descent update~\eqref{eq:gd_update_wk} can be expressed in primal-dual form as
\begin{subequations}\label{eq:primal_dual_update}
\begin{alignat}{2}
    &\text{(Primal) }\myquad[6] &&\bbeta_k^{(t+1)} = \bbeta_k^{(t)} - \eta s_k\X \X^\top \D(\bbeta_k^{(t)})(h_{\bTheta^{(t)}}(\X) - \y),\myquad[6]\label{eq:primal_update} \\
    &\text{(Dual) }\myquad[6] &&\balpha_k^{(t+1)} = \balpha_k^{(t)} - \eta s_k\D(\bbeta_k^{(t)})(h_{\bTheta^{(t)}}(\X) - \y)\label{eq:dual_update}.\myquad[6]
\end{alignat}
\end{subequations}
This primal-dual formulation plays a central role in our analysis. In particular, the sign of each primal coordinate $\beta_{k,i}^{(t)}$ determines whether the corresponding dual variable $\alpha_{k,i}^{(t+1)}$ is updated through the diagonal matrix $\D(\bbeta_k^{(t)})$. Consequently, understanding the positivity pattern of $\bbeta_k^{(t)}$ and the resulting dynamics of $\balpha_k^{(t)}$ is key to characterizing the behavior and implicit bias of gradient descent.

\paragraph{Minimum-\texorpdfstring{$\ell_2$}{l2}-norm Solution.}\label{sec:minimum_l2_sol}
It is well known that, for linear regression with zero initialization, i.e., $h(\x)\coloneqq \w^\top\x$ with $\w^{(0)} = \zero$, gradient descent converges to the minimum-$\ell_2$-norm interpolation, which is given by $\w_{\text{linear-MNI}} = \arg\min_{\w}\frac{1}{2}\nnorm[2]{\w}^2$, $\text{s.t. } \w^\top\x_i = y_i, \text{ for all } i\in [n]$.
This solution admits the closed-form expression $\w_{\text{linear-MNI}} = \X^\top(\X\X^\top)^{-1}\y$.
Motivated by this classical result, we consider the minimum-$\ell_2$-norm solution for the general ReLU regression problem that we study, defined as:
%
\begin{align}\label{eq:w_star_def}
    &\{\w_k^\star\}_{k=1}^m = {} \uargmin{\{\w_k\}_{k=1}^m}\frac{1}{2}\sum_{k=1}^m \nnorm[2]{\w_k}^2\\
    \quad\quad \text{s.t. } &\sum_{k=1}^m s_k\sigma(\w_k^\top \x_i) = y_i, \text{ for all } i\in[n].\nonumber
\end{align}
%
Let $\bTheta_g$ denote the set of global minimizers of the empirical risk~\eqref{eq:empirical_risk}.
Note that for $m = 1$, the empirical risk can often not be driven to zero; labels that are opposite in sign to the sign of the neuron $s_1$ cannot be fit.
For networks with $m > 1$ neurons, we will consider at least one neuron to be positively signed and negatively signed respectively, ensuring that the global minimizers will achieve zero empirical risk and interpolate the training data (i.e. $h_{\bTheta}(\x_i) = y_i, \; \forall i\in[n]$).

\section{Implicit Bias of Single ReLU Models (\texorpdfstring{$m=1$}{m=1}) Under Gradient Descent}\label{sec:single_relu}

We begin by analyzing the case of the single positive ReLU neuron model ($m = 1$).
Specifically, we consider $h_{\bTheta}(\x) \coloneqq s_1 \sigma(\w^\top \x)$ where $\w\in\R^d$ is the only model parameter.
We will also assume that $s_1 = +1$ as will become clear through this section, the single neuron is only capable of fitting positive labels in this case.
A symmetric version of our results in this section will hold in the opposite case where $s_1 = -1$, with all instances of positive labels replaced by negative labels.
We omit this case for brevity.

\subsection{Gradient Descent Updates and Convergence}

For the single ReLU model ($m=1$), the gradient descent update in~\eqref{eq:gd_update_wk} simplifies to
\begin{align}\label{eq:single_relu_gd_update}
\w^{(t+1)} = \w^{(t)} - \eta \nabla_{\w} \Risk(\w^{(t)}) &= \w^{(t)} - \eta \X^\top \D(\X\w^{(t)})\bigl(\sigma(\X\w^{(t)}) - \y\bigr)\nonumber\\
    &= \w^{(t)} - \eta\X^\top \D(\X\w^{(t)})(\X\w^{(t)} - \y),
\end{align}
where we write the vector-valued extension of the ReLU as $\sigma(\z)\coloneqq [\sigma(z_1),\cdots,\sigma(z_n)]^\top\in \R^n$, and the second equality follows from the fact that the diagonal matrix $\D(\X\w^{(t)})$ enforces the ReLU activation pattern. Specifically, since $\D(\X\w^{(t)})$ contains indicators of positive pre-activations, the explicit nonlinearity $\sigma(\cdot)$ can be removed once it is applied. 

Compared to linear regression, the key difference in the gradient update for a single ReLU model is the presence of the diagonal matrix $\D(\X\w^{(t)})$. This matrix effectively selects a subset of examples --- those with positive pre-activations --- to contribute to each gradient update. As a result, the optimization trajectory becomes both data-dependent and time-varying, with the active set of samples evolving during training.


\subsubsection{Sufficient Conditions for Gradient Descent Convergence}
According to Equation~\eqref{eq:single_relu_gd_update}, convergence of gradient descent occurs when $\nabla_{\w}\Risk(\w^{(t)})=\zero$. This condition implies that, for every $i \in [n]$, either $\x_i^\top \w^{(t)} \leq 0$ or $\x_i^\top \w^{(t)} = y_i$. In other words, at convergence, each training example is either inactive due to the ReLU nonlinearity or is fit exactly. These criteria can define either a global or local minimum depending on the activation pattern.


In general, the optimization trajectory and loss landscape induced by gradient descent, even on a single ReLU model, are difficult to characterize, primarily due to this data-dependent activation pattern. However, suppose there exists an iteration $t_0 \geq 0$ such that, for all $t \geq t_0$, the set of active examples $S \coloneqq \{ i\in [n] : \x_i^\top\w^{(t_0)} > 0\}$, remains unchanged. In this ``final phase'', we expect the gradient descent dynamics of the single ReLU model to reduce to those of linear regression restricted to the active subset of samples.
We formalize this observation in the following lemma, which is proved in Appendix~\ref{app:proof_single_relu_convergence_mni_in_general}.
\begin{lemma}
\label{lem:final_phase}
    Suppose there exists $t_0 \ge 0$ such that $\D(\X\w^{(t_0)})=\D(\X\w^{(t)})$ for all $t \geq t_0$. Define the subset of examples $S \coloneqq \{ i\in [n] : \x_i^\top\w^{(t_0)} > 0\}$. Then, for all $t \geq t_0$, the gradient descent dynamics of the single ReLU model are equivalent to gradient descent applied to a linear model initialized at $\w^{(t_0)}$ and trained only on the subset $S$.
\end{lemma}

As a direct consequence, convergence of the single ReLU model in the final phase follows from classical convergence guarantees for linear regression. In particular, it is straightforward to show that if the step size $\eta \leq \frac{1}{\mu_1(\XX)}$, then gradient descent converges in the final phase under our conditions on the training data. We state and prove this result for completeness in Appendix~\ref{app:proof_single_relu_convergence_mni_in_general}.
\begin{lemma}\label{lem:conv_step_size}
    Suppose there exists $t_0 \geq 0$ such that $\D(\X\w^{(t_0)})=\D(\X\w^{(t)})$ for all $t \geq t_0$, if the step size satisfies $\eta \leq \frac{1}{\mu_1(\XX)}$, gradient descent applied to the single ReLU model converges to $\w^{(\infty)} =\uargmin{\w\in\{\w:\X_S\w =\y_S\}} \nnorm[2]{\w - \w^{(t_0)}}$, where $S \coloneqq \{ i\in [n] : \x_i^\top\w^{(t_0)} > 0\}$.
\end{lemma}

\subsection{Minimum-\texorpdfstring{$\ell_2$}{l2}-norm Solution of Single ReLU Models}
In Section~\ref{sec:setup}, we discussed the minimum-$\ell_2$-norm solution for linear regression (called $\w_{\text{linear-MNI}}$). In contrast, due to the presence of the ReLU activation, single ReLU models can only produce nonnegative outputs. As a result, such models can minimize the empirical risk only by exactly fitting all samples with positive labels and outputting zero on samples with nonpositive labels.
It is natural to consider the minimum-$\ell_2$-norm solution for the single ReLU model subject to these constraints.
Interestingly, this can be written as a convex optimization problem (despite the empirical risk itself being nonconvex) in which the constraints associated with nonpositive labels are written as linear inequalities, as below:
\begin{align}\label{eq:single_relu_minimum_norm_sol}
    \w^\star = &\uargmin{\w}\frac{1}{2}\nnorm[2]{\w}^2\\ \text{ s.t. } \w^\top\x_i &= y_i, \text{ for all } y_i > 0,\nonumber\\
    \w^\top\x_j &\leq 0, \text{ for all } y_j \leq 0.\nonumber
\end{align}
%
We show that the solution of~\eqref{eq:single_relu_minimum_norm_sol} coincides with the minimum-$\ell_2$-norm solution associated with \emph{linearly} fitting a suitable subset of training examples, after setting all negative labels to zero.
We define the linear MNI solution associated with the subset $S\subseteq [n]$ as $\w_{\text{linear-MNI},S} = \X_S^\top(\X_S\X_S^\top)^{-1}\tilde{\y}_S$, where $\tilde{\y}_S \in \R^{|S|}$ denotes the corresponding modified label subvector with all negative entries replaced by zero.

\begin{lemma}
\label{lem:relu_mni_subset}
Consider a single ReLU model $h_{\bTheta}(\x) = \sigma(\w^\top \x)$. The minimum-$\ell_2$-norm solution $\w^\star$ of $h_{\bTheta}(\x)$ defined in Equation~\eqref{eq:single_relu_minimum_norm_sol} satisfies $\w^\star = \w_{\text{linear-MNI}, S}$ for some index subset $S \subseteq [n]$ that necessarily contains all indices $i$ such that $y_i > 0$, where the corresponding labels are given by $\tilde{y}_{S,i} = \max\{y_i, 0\}$.
\end{lemma}
Lemma~\ref{lem:relu_mni_subset} is proved in Appendix~\ref{app:proof_single_relu_convergence_mni_in_general} through the Karush-Kahn-Tucker (KKT) conditions.
It is important to note that $\w^\star$ is a fundamentally different inductive bias from $\w_{\text{linear-MNI}}$ as the subset $S$ does not have an explicit formula, and is training data-dependent.

\subsection{High-dimensional Implicit Bias of Single ReLU Models}
Our first main result, stated below, characterizes the gradient descent dynamics of single ReLU models on high-dimensional data.

\begin{theorem}\label{thm:single_relu_gd_high_dim_implicit_bias}
    Consider Assumptions~\ref{asm:label_bounds} and~\ref{asm:high_dim_data}, suppose the initialization is $\w^{(0)} = \X^{\top}(\XX)^{-1} \bepsilon$, where $0 < \epsilon_i \leq \frac{1}{C_{\alpha}} y_{\min}$ for all $i \in [n]$, and the step size to satisfy $\frac{1}{CC_g\nnorm[1]{\blambda}} \leq \eta \leq \frac{1}{C_g\nnorm[1]{\blambda}}$. Then, the gradient descent limit $\w^{(\infty)}$ for the single ReLU model coincides with the solution obtained by linear regression trained only on the positively labeled examples with initialization $\w^{(1)} = \eta\X^\top\Bigl(\y-\bepsilon+\frac{1}{\eta}(\XX)^{-1}\bepsilon\Bigr)$ with probability at least $1 - 2 \exp(-
    cn)$. Formally, we have $\w^{(\infty)} =\uargmin{\w\in\{\w:\X_+\w =\y_+\}} \nnorm[2]{\w - \w^{(1)}}$ and $\X_- \w^{(\infty)} \preceq \zero$.
\end{theorem}
Theorem~\ref{thm:single_relu_gd_high_dim_implicit_bias} is proved in Appendix~\ref{app:proof_single_relu_gd_high_dim_implicit_bias} and characterizes a regime of gradient descent in which the convergence behavior is tractable. 
Due to the presence of the ReLU activation, the main challenge lies in monitoring which examples are active and which are inactive during gradient descent. Under our assumption of sufficiently high-dimensional data, we show, through careful tracking of the primal and dual variables, that examples with positive labels remain active throughout the optimization process (see Lemma~\ref{lem:pos_stay_pos}), while examples with negative labels eventually become and remain inactive (see Lemma~\ref{lem:neg_stay_neg}).
Therefore, the limiting solution fits all positive labels exactly and produces predictions equal to zero for samples with negative labels. Consequently, this solution achieves the minimum empirical risk, i.e. is a specific global minimizer of~\eqref{eq:empirical_risk}. 

\begin{remark}
    In addition to Assumptions~\ref{asm:label_bounds} and~\ref{asm:high_dim_data}, Theorem~\ref{thm:single_relu_gd_high_dim_implicit_bias} assumes a sufficiently small initialization where all the training examples are active (to see this, note that $\X \w^{(0)} = \bepsilon \succ \zero$)\footnote{The initialization expression $\w^{(0)} = \X^\top(\XX)^{-1}\bepsilon$ with $\bepsilon\in\R^n$ is isomorphic to an arbitrary initialization in the span of the data $\{\x_i\}_{i=1}^n$, owing to the full-row-rank assumption on $\X$. Note in particular that such an initialization can be arbitrarily far from $\w^\star$, and does not constitute a ``local'' initialization.}. 
    Essentially, the primal variables are initialized in the positive orthant.
    The sufficiently small initialization is also assumed in previous work~\citep{boursier2022gradient,dana2025convergence}.
    The positivity assumption is made to ensure high-probability convergence to a global minimum.
    On the other hand, a random initialization would map to both positive and negative primal variables.
    Our simulations in Appendix~\ref{app:simulation} (in particular, Figure~\ref{fig:bad_init_high_dim}) demonstrate that in this case, a positively labeled but initially inactive example remains inactive, meaning that gradient descent can only converge to a local minimum\footnote{In fact, this is why~\citet{dana2025convergence} need to assume a sufficiently large number of neurons $m$ to ensure global convergence under random initialization.}. Additionally, we provide a simple explicit counterexample showing that such initializations result in convergence to a local minimum that is not globally optimal in Appendix~\ref{app:counter_example}.
\end{remark}

\subsection{Approximation to Minimum-\texorpdfstring{$\ell_2$}{l2}-norm Solution in High Dimensions}
We now show that the limiting solution obtained from Theorem~\ref{thm:single_relu_gd_high_dim_implicit_bias} is different from, but close to the minimum-$\ell_2$-norm solution in high dimensions. 
Specifically, the following theorem upper and lower bounds the Euclidean distance between $\w^{(\infty)}$ and $\w^\star$ as a function of the number of negative examples $n_-$, effective dimension and label magnitude.

\begin{theorem}\label{thm:single_relu_approx_to_w_star}
    Consider Assumptions~\ref{asm:label_bounds} and~\ref{asm:high_dim_data}, suppose the initialization is $\w^{(0)} = \X^{\top}(\XX)^{-1} \bepsilon$, where $0 < \epsilon_i \leq \frac{1}{C_{\alpha}} y_{\min}$ for all $i \in [n]$, and the step size to satisfy $\frac{1}{CC_g\nnorm[1]{\blambda}} \leq \eta \leq \frac{1}{C_g\nnorm[1]{\blambda}}$. Then, we have $\sqrt{\frac{n_-y_{\min}^2}{C C_g\nnorm[1]{\blambda}}}\leq \nnorm[2]{\w^{(\infty)} - \w^{\star}} \leq \sqrt{\frac{16n_-y_{\max}^2}{C_g\nnorm[1]{\blambda}}}$  with probability at least $1 - 2\exp(-cn)$.
\end{theorem}
Theorem~\ref{thm:single_relu_approx_to_w_star} is proved in Appendix~\ref{app:proof_single_relu_approx_to_w_star} and heavily leverages our characterization of the minimum-$\ell_2$-norm solution in Lemma~\ref{lem:relu_mni_subset}.
Our simulation in Figure~\ref{fig:convex_prog} shows excellent agreement with Theorem~\ref{thm:single_relu_approx_to_w_star}.
Note that Theorem~\ref{thm:single_relu_approx_to_w_star} implies that $\w^{(\infty)} = \w^\star=\w_{\text{linear-MNI}}$ in the case where all labels are positive!

\section{Implicit Bias of Two ReLU Models (\texorpdfstring{$m=2$}{m=2}) Under Gradient Descent}\label{sec:multiple_relu_gd}

We now extend our analysis to a $2$-ReLU model ($m=2$), which combines one positive ReLU neuron and one negative ReLU neuron. More specifically, we define $h_{\bTheta}(\x) = \sigma(\w_{\oplus}^\top\x) - \sigma(\w_{\ominus}^\top\x)$, where $\bTheta$ is a set of model parameters such that $\bTheta \coloneqq \{\w_{\oplus}, \w_{\ominus}\}$ and $\w_{\oplus}, \w_{\ominus}\in\R^d$. As mentioned in Section~\ref{sec:minimum_l2_sol}, this model is more expressive than the single ReLU model as it can perfectly fit arbitrary labels (both positive and negative).
For the $2$-ReLU model, the gradient descent update in~\eqref{eq:gd_update_wk} simplifies to
\begin{align*}
\w_{\oplus}^{(t+1)}
&= \w_{\oplus}^{(t)} - \eta \nabla_{\w_{\oplus}} \Risk(\bTheta^{(t)}) = \w_{\oplus}^{(t)} - \eta \X^\top \D(\X\w_{\oplus}^{(t)})\bigl(h_{\bTheta^{(t)}}(\X) - \y\bigr).\\
\w_{\ominus}^{(t+1)}
&= \w_{\ominus}^{(t)} - \eta \nabla_{\w_{\ominus}} \Risk(\bTheta^{(t)}) = \w_{\ominus}^{(t)} + \eta \X^\top \D(\X\w_{\ominus}^{(t)})\bigl(h_{\bTheta^{(t)}}(\X) - \y\bigr).
\end{align*}

\subsection{Minimum-\texorpdfstring{$\ell_2$}{l2}-norm Solution of \texorpdfstring{$2$}{2}-ReLU Models}

First, we characterize the minimum-$\ell_2$-norm solution for the $2$-ReLU model, defined below:
\begin{align}\label{eq:multiple_relu_minimum_norm_sol}
    \w_{\oplus}^\star, \w_{\ominus}^\star  = &\uargmin{\w_{\oplus}, \w_{\ominus}}\frac{1}{2}\nnorm[2]{\w_{\oplus}}^2 + \frac{1}{2}\nnorm[2]{\w_{\ominus}}^2\\
   \text{s.t. } \sigma(\w_{\oplus}^\top\x_i)&- \sigma(\w_{\ominus}^\top\x_i) = y_i, \text{ for all } i\in[n].\nonumber
\end{align}
Unlike in the case of the single ReLU model,~\eqref{eq:multiple_relu_minimum_norm_sol} cannot be written as a convex program.
To analyze~\eqref{eq:multiple_relu_minimum_norm_sol}, we show that the optimal solution is also the optimal solution to a \emph{restricted convex program} obtained by fixing the activation pattern of the two ReLU units across the training examples. To state this result, we define some additional notation.
Let $S_+ = \{i: y_i > 0, \text{ for all }i\in[n]\}$, $S_- = \{j: y_j < 0, \text{ for all }j\in[n]\}$, so that $S_+ \cup S_- = [n]$ and $S_+ \cap S_- =\varnothing$. 
\begin{lemma}\label{lem:multiple_relu_w_star_eqivalence}
    The feasible set of~\eqref{eq:multiple_relu_minimum_norm_sol} is nonempty, and there exist partitions $S_1 \cup S_2 = S_+, \; S_1\cap S_2 =\varnothing, \text{ and } S_3 \cup S_4 = S_-, \; S_3\cap S_4 =\varnothing$
    such that the optimal solution $\{\w_{\oplus}^\star, \w_{\ominus}^\star\}$ of~\eqref{eq:multiple_relu_minimum_norm_sol} is also an optimal solution of the following convex program:
\begin{alignat}{3}\label{eq:multiple_relu_w_star_eq_form}
    &\myquad[2]\w_{\oplus}^\star, \w_{\ominus}^\star &= \uargmin{\w_{\oplus}, \w_{\ominus}}\frac{1}{2}\nnorm[2]{\w_{\oplus}}^2 + \frac{1}{2}\nnorm[2]{\w_{\ominus}}^2&\\
    \text{s.t. } &\w_{\oplus}^\top\x_i \phantom{-\w_{\ominus}^\top\x_i} &= y_i, \phantom{-\w_{\oplus}^\top\x_i < 0, -}\w_{\ominus}^\top\x_i\leq 0, &\quad\text{ for all } i\in S_1,\nonumber\\
    &\w_{\oplus}^\top\x_i - \w_{\ominus}^\top\x_i &= y_i,  \phantom{-\w_{\oplus}^\top\x_i < 0} -\w_{\ominus}^\top\x_i \leq 0, &\quad\text{ for all } i\in S_2,\nonumber\\
    &\phantom{\w_{\oplus}^\top\x_i} - \w_{\ominus}^\top\x_i &= y_i, \phantom{-}\w_{\oplus}^\top\x_i\leq 0, \phantom{-\w_{\ominus}^\top\x_i < 0,} &\quad\text{ for all } i\in S_3,\nonumber\\
    &\w_{\oplus}^\top\x_i - \w_{\ominus}^\top\x_i &= y_i, -\w_{\oplus}^\top\x_i \leq 0, \phantom{-\w_{\ominus}^\top\x_i < 0,} &\quad\text{ for all } i\in S_4.\nonumber
\end{alignat}
\end{lemma}
Lemma~\ref{lem:multiple_relu_w_star_eqivalence} is proved in Appendix~\ref{app:proof_multiple_relu_w_star_eqivalence}.
Note that, in general, it is not possible to explicitly identify or characterize the exact activation patterns and corresponding partitions. However, the mere existence of such a partition is sufficient for our purposes and allows us to derive the desired approximation results for the implicit bias in Section~\ref{sec:multiple_relu_approx_w_star}.

\subsection{High-dimensional Implicit Bias of \texorpdfstring{$2$}{2}-ReLU Models}

Next, we characterize the gradient descent dynamics of two ReLU models in the high-dimensional regime in a manner analogous to the single-ReLU model (Theorem~\ref{thm:single_relu_gd_high_dim_implicit_bias}).

\begin{theorem}\label{thm:multiple_relu_gd_high_dim_implicit_bias}
    Consider Assumptions~\ref{asm:label_bounds} and \ref{asm:high_dim_data}, suppose the initialization $\w_{\oplus}^{(0)} = \X^\top\bigl(\X\X^\top\bigr)^{-1}\bepsilon_{\oplus}$ and $\w_{\ominus}^{(0)} = \X^\top\bigl(\X\X^\top\bigr)^{-1}\bepsilon_{\ominus}$, where $0 < \epsilon_{\oplus,i}, \epsilon_{\ominus,i} \leq \frac{1}{2C_{\alpha}}y_{\min}$ for all $i \in [n]$, and the step size to satisfy $\frac{1}{CC_g\nnorm[1]{\blambda}}\leq \eta \leq \frac{1}{C_g\nnorm[1]{\blambda}}$. Then, with probability at least $1 - 2\exp(-cn)$, we have: The gradient descent limit $\w_{\oplus}^{(\infty)}$ coincides with the solution obtained by linear regression trained only on the positively labeled examples, with the initialization $\w_{\oplus}^{(1)} = \eta\X^\top\Bigl(\y - \bepsilon_{\oplus} + \bepsilon_{\ominus} + \frac{1}{\eta}(\X\X^\top)^{-1}\bepsilon_{\oplus}\Bigr)$, and $\w_{\oplus}^{(\infty)} =\uargmin{\w\in\{\w:\X_+\w =\y_+\}} \nnorm[2]{\w - \w_{\oplus}^{(1)}}$; the gradient descent limit $\w_{\ominus}^{(\infty)}$ coincides with the solution obtained by linear regression trained only on the negatively labeled examples, with the initialization $\w_{\ominus}^{(1)} = \eta\X^\top\Bigl(-\y + \bepsilon_{\oplus} - \bepsilon_{\ominus} + \frac{1}{\eta}(\X\X^\top)^{-1}\bepsilon_{\ominus}\Bigr)$ and $\w_{\ominus}^{(\infty)} =\uargmin{\w\in\{\w:\X_-\w =-\y_-\}} \nnorm[2]{\w - \w_{\ominus}^{(1)}}$, with $\X_- \w_{\oplus}^{(\infty)} \preceq \zero$ and $\X_+ \w_{\ominus}^{(\infty)} \preceq \zero$.
\end{theorem}
Theorem~\ref{thm:multiple_relu_gd_high_dim_implicit_bias} is proved in Appendix~\ref{app:proof_multiple_relu_gd_high_dim_implicit_bias} in a manner similar to the proof of Theorem~\ref{thm:single_relu_gd_high_dim_implicit_bias}.
Our main additional insight is that, in high dimensions, the optimization dynamics naturally decouple: $\w_{\oplus}$ learns to fit all positively labeled examples, while $\w_{\ominus}$ learns to fit all negatively labeled examples. 

\subsection{Approximation to Minimum-\texorpdfstring{$\ell_2$}{l2}-norm Solution in High Dimensions}\label{sec:multiple_relu_approx_w_star}

Finally, we show, in a result analogous to Theorem~\ref{thm:single_relu_approx_to_w_star}, that the limiting solution of Theorem~\ref{thm:multiple_relu_gd_high_dim_implicit_bias} is close to the minimum-$\ell_2$-norm solution $\{\w_{\oplus}^\star, \w_{\ominus}^\star\}$.

\begin{theorem}\label{thm:multiple_relu_approx_to_w_star}
    Consider Assumptions~\ref{asm:label_bounds} and \ref{asm:high_dim_data}, suppose the initialization is $\w_{\oplus}^{(0)} = \X^\top\bigl(\X\X^\top\bigr)^{-1}\bepsilon_{\oplus}
    $, $\w_{\ominus}^{(0)} = \X^\top\bigl(\X\X^\top\bigr)^{-1}\bepsilon_{\ominus},$  where $0 < \epsilon_{\oplus,i},\epsilon_{\ominus,i} \leq \frac{1}{2C_{\alpha}}y_{\min}$ for all $i \in [n]$, and the step size satisfies $\frac{1}{CC_g\nnorm[1]{\blambda}}\leq \eta \leq \frac{1}{C_g\nnorm[1]{\blambda}}$. Then, we have $\sqrt{\frac{n_-y_{\min}^2}{C C_g\nnorm[1]{\blambda}}}\leq\nnorm[2]{\w_{\oplus}^{(\infty)} - \w_{\oplus}^{\star}} \leq \sqrt{\frac{16n_-y_{\max}^2}{C_g\nnorm[1]{\blambda}}}$ and $\sqrt{\frac{n_+y_{\min}^2}{C C_g\nnorm[1]{\blambda}}}\leq\nnorm[2]{\w_{\ominus}^{(\infty)} - \w_{\ominus}^{\star}} \leq \sqrt{\frac{16n_+y_{\max}^2}{C_g\nnorm[1]{\blambda}}}$  with probability at least $1 - 2 \exp(-
    cn)$.
\end{theorem}

Theorem~\ref{thm:multiple_relu_approx_to_w_star} is proved in Appendix~\ref{app:proof_multiple_relu_approx_to_w_star} and leverages the restricted convex program that we derived in Lemma~\ref{lem:multiple_relu_w_star_eqivalence}.
Due to the relative complexity of~\eqref{eq:multiple_relu_w_star_eq_form}, the proof becomes more involved than that of Theorem~\ref{thm:single_relu_approx_to_w_star}, but the underlying basic ideas are similar.
Note that, because one of $n_+, n_- > 0$, the implicit bias of $2$-ReLU cannot exactly coincide with the minimum-$\ell_2$-norm solution.
\section{Main Proof Ideas}\label{sec:proof-sketch}

\begin{figure}
\centering 
\begin{minipage}[t]{0.45\textwidth}
\resizebox{\textwidth}{!}{
 \begin{tikzpicture}[
    node distance=1.5cm and 2cm,
    every node/.style={circle, draw, minimum size=0.8cm, font=\small},
    arrow/.style={-Stealth, thick}
]

\node[draw=none] at (0, 1.3) {\large $(t)$};
\node[draw=none] at (4, 1.3) {\large $(t+1)$};

\node[draw=none] at (-2.5, 0) {\large $y_i > 0$:};
\node (A) at (0, 0) [minimum size=1.6cm, inner sep=1pt] {$\beta_i^{(t)}>0$};
\node (B) at (4, 0) [minimum size=1.6cm, inner sep=1pt] {$\beta_i^{(t+1)}>0$};

\draw[arrow] (A) -- node[above, draw=none, inner sep=-15pt] {Lemma~\ref{lem:pos_stay_pos}} (B);

\node[draw=none] at (-2.5, -2.9) {\large $y_j < 0$:};
\node (D) at (0, -1.8) [minimum size=1.6cm, inner sep=1pt] {$\beta_j^{(t)}<0$};
\node (E) at (0, -3.8) [minimum size=1.6cm, inner sep=1pt] {$\alpha_j^{(t)}<0$};
\node (G) at (4, -3.8) [minimum size=1.6cm, inner sep=1pt, font=\tiny] {$\alpha_j^{(t+1)}=\alpha_j^{(t)}$};

\draw[arrow] (E) --  node[left, draw=none, inner sep=2pt]{Lemma~\ref{lem:neg_stay_neg}} (D);
\draw[arrow] (D) -- node[above right, draw=none, pos=0.4, yshift=-15pt] {Equation~\eqref{eq:dual_update}} (G);

\end{tikzpicture}
}

 \caption{Gradient descent transition diagram for the $k$-th neuron.} 
 \label{fig:gd_transition}
\end{minipage}
\hfill
\begin{minipage}[t]{.45\textwidth}
 \includegraphics[width=65mm]{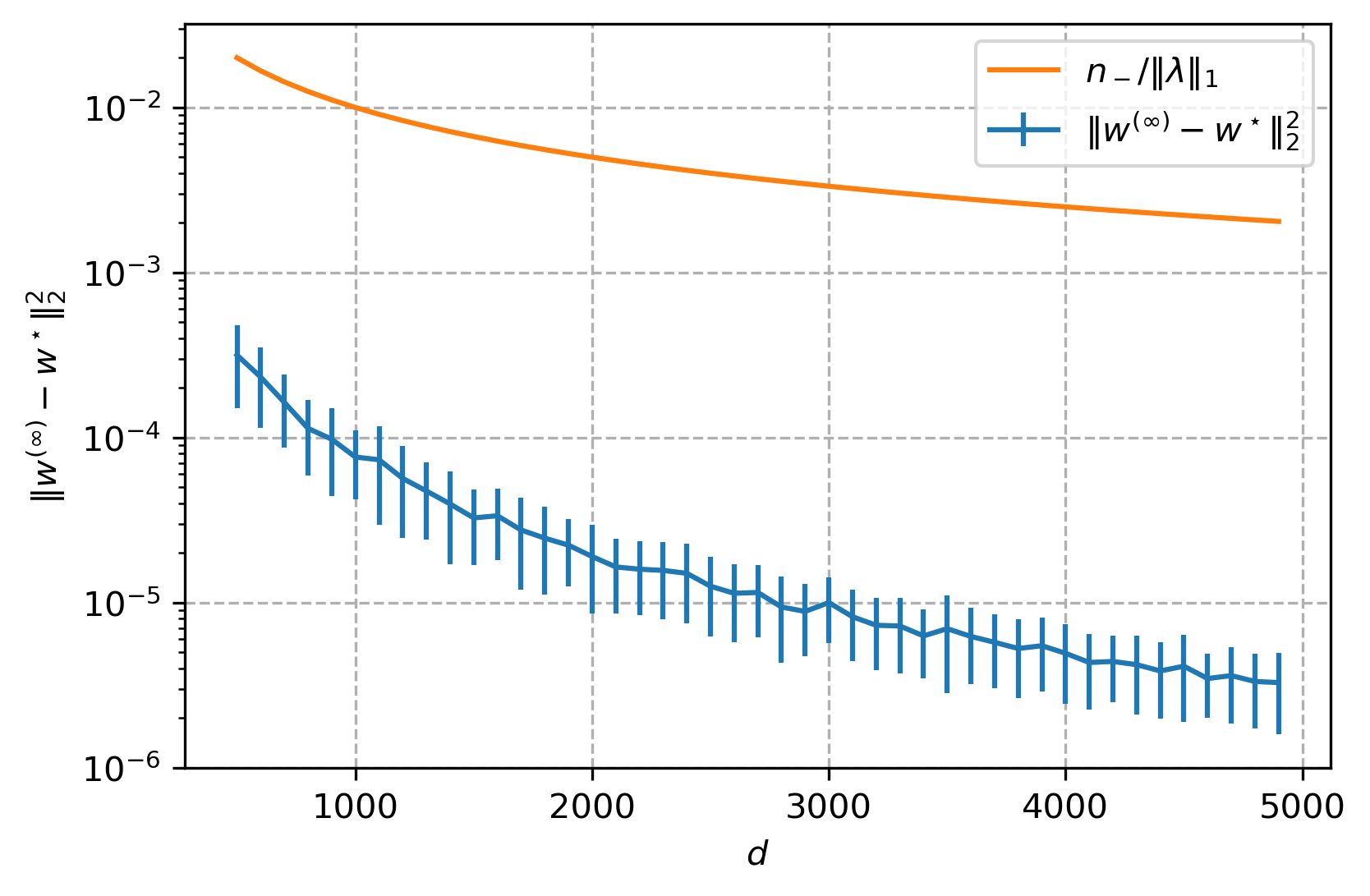}
 \caption{Approximation error between the implicit bias of the single ReLU model $\w^{(\infty)}$ and the minimum-$\ell_2$-norm solution $\w^\star$.} 
  \label{fig:convex_prog}
\end{minipage}
\end{figure}

Our analysis hinges on \textbf{precisely tracking the activation patterns of ReLU neurons across all training examples.} By controlling which examples remain active or inactive throughout training, we are able to understand the resulting gradient dynamics and, consequently, the implicit bias of the converged solution. To establish these results, we introduce two key lemmas. Lemma~\ref{lem:pos_stay_pos}, inspired by ideas in~\citet{dana2025convergence}, shows that once the primal variable $\beta_{k,i}$ corresponding to the $k$-th neuron and the $i$-th example is active—and the sign of the neuron $s_k$ agrees with the label $y_i$—it remains active in the next iteration. This ensures that such an example is not suppressed by the ReLU nonlinearity and continues to contribute to the gradient updates.

\begin{lemma}\label{lem:pos_stay_pos}
    Under Assumptions~\ref{asm:label_bounds} and~\ref{asm:high_dim_data}, suppose the gradient descent step size satisfies $\eta \leq \frac{1}{C_g\nnorm[1]{\blambda}}$. Consider the $k$-th ReLU neuron in $h_{\bTheta}$. For any $t \geq 0$ and any index $i\in[n]$ such that $s_k \cdot y_i > 0$, if $\beta_{k,i}^{(t)} > 0$, $\beta_{k,i}^{(t)} \geq s_k \cdot h_{\bTheta^{(t)}}(\x_i)$, and $\nnorm[2]{h_{\bTheta^{(t)}}(\X) - \y} \leq C_y\nnorm[2]{\y}$, then $\beta_{k,i}^{(t+1)} > 0$ with probability at least $1 - 2 \exp(-cn)$.
\end{lemma}
This lemma is proved in Appendix~\ref{app:proof_primal}.
The main idea behind Lemma~\ref{lem:pos_stay_pos} is that as long as the primal variable $\beta_{k,i}^{(t)}$ is positive and the empirical risk remains uniformly bounded, the gradient update of $\beta_{k,i}^{(t)}$ is dominated by its self-interaction term for high-dimensional data. The reason, at a high level, is that cross-sample interactions can be bounded in high dimensions (due to the concentration of the random Gram matrix $\X \X^\top$ around $\norm{\blambda}_1 \I$). As a result, the magnitude of the update is strictly smaller than $\beta_{k,i}^{(t)}$, ensuring that $\beta_{k,i}^{(t+1)}$ remains positive. 

Lemma~\ref{lem:neg_stay_neg} concerns the behavior of inactive examples. It shows that once a dual variable $\alpha_{k,j}$ associated with the 
$k$-th neuron and the $j$-th example becomes sufficiently negative, the corresponding primal variable $\beta_{k,j}$ remains inactive. Consequently, the dual variable is no longer updated and stays frozen throughout training. This mechanism effectively removes certain examples from the optimization dynamics.

\begin{lemma}\label{lem:neg_stay_neg}
    Under Assumptions~\ref{asm:label_bounds} and~\ref{asm:high_dim_data}, consider the $k$-th ReLU neuron in $h_{\bTheta}$. For any $t\geq 0$ and any index $j\in[n]$, if $\alpha_{k,j}^{(t)} \leq -\frac{y_{\min}}{C_{\alpha}\nnorm[1]{\blambda}}$ and $\nnorm[2]{\balpha_k^{(t)}} \leq \frac{C_{\alpha}\sqrt{n}y_{\max}}{\nnorm[1]{\blambda}}$, then $\beta_{k,j}^{(t)} \leq 0$ and $\alpha_{k,j}^{(t+1)} = \alpha_{k,j}^{(t)}$ with probability at least $1 - 2 \exp(-cn)$.
\end{lemma}
The proof of Lemma~\ref{lem:neg_stay_neg} (see Appendix~\ref{app:proof_dual}) relies on the primal-dual relationship $\bbeta_k=\XX\balpha_k$ from Equation~\eqref{eq:primal_dual_def}, together with concentration results for the Gram matrix. Specifically, if a dual variable is sufficiently negative, then the corresponding primal variable $\beta_{k,j}^{(t)}$ is strictly negative. According to the dual update rule in Equation~\eqref{eq:dual_update}, a negative $\beta_{k,j}^{(t)}$ implies that the ReLU is inactive and the dual coordinate receives no further updates. As a result, $\alpha_{k,j}^{(t+1)} = \alpha_{k,j}^{(t)}$, and sufficiently negative dual variables remain frozen throughout training. Figure~\ref{fig:gd_transition} depicts the transition of primal-dual updates in Lemma~\ref{lem:pos_stay_pos} and Lemma~\ref{lem:neg_stay_neg}. We also provide deterministic-feature counterparts of Lemmas~\ref{lem:pos_stay_pos} and~\ref{lem:neg_stay_neg}, stated as Lemmas~\ref{lem:pos_stay_pos_fixed} and~\ref{lem:neg_stay_neg_fixed}, respectively. Their proofs are given in Appendix~\ref{app:primal_fixed}.

In the following paragraphs, we outline the proof sketch for single ReLU models. The proof ideas for the $2$-ReLU case follow analogously.

\paragraph{Proof Sketch of Theorem~\ref{thm:single_relu_gd_high_dim_implicit_bias}:}
The proof combines the insights from Lemma~\ref{lem:pos_stay_pos} and Lemma~\ref{lem:neg_stay_neg} to obtain a complete picture of how activation patterns evolve during training. Together, these lemmas allow us to track which examples remain active or inactive throughout gradient descent. Our goal is to reach—and maintain—a configuration in which positive-labeled examples remain active while negative-labeled examples remain inactive, as formalized by the sufficient conditions in Lemma~\ref{lem:primal_label_same_sign} in Appendix~\ref{app:proof_single_relu_gd_high_dim_implicit_bias}.
To achieve this, we leverage two key properties of the initialization. First, the positive initialization guarantees that every example initially has at least one active neuron capable of fitting it. Second, the small initialization ensures that, after the first gradient step, positive-labeled examples remain in the active regime while negative-labeled examples acquire sufficiently negative dual variables and become inactive. Together, these properties place positive and negative examples into their respective regimes after a single update.
We then apply Lemma~\ref{lem:primal_label_same_sign} to show that this configuration is stable under subsequent iterations. As a result, the activation pattern becomes fixed after the first step, and the dynamics enter the final phase described in Lemma~\ref{lem:final_phase}. 

\paragraph{Proof Sketch of Theorem~\ref{thm:single_relu_approx_to_w_star}:}
To compare the gradient descent limit $\w^{(\infty)}$ with the minimum-$\ell_2$-norm solution $\w^{\star}$, we relate their distance in parameter space to their distance in prediction space. Since both solutions interpolate all positive-labeled examples exactly, any discrepancy between them must arise from their predictions on negative-labeled examples.
We bound this discrepancy using the KKT conditions characterizing $\w^{\star}$, as established in Lemma~\ref{lem:relu_mni_subset}. These conditions precisely describe how $\w^{\star}$ treats negative-labeled examples and allow us to control the prediction distance in terms of the distance between the primal and dual variables. In particular, the KKT conditions imply that this gap is nonzero, showing that $\w^{(\infty)}\neq \w^{\star}$. Translating our bounds back to parameter space yields matching upper and lower bounds on $\|\w^{(\infty)} - \w^{\star}\|_2$.

\section{Discussion}

We showed that the implicit bias of single and $2$-ReLU models, under appropriate initialization, is remarkably close to the minimum-norm solution if the features are sufficiently high-dimensional (and under appropriate initialization to ensure global convergence).
Natural open questions include: 1) characterizing the dynamics for $m > 2$ neurons, and 2) studying the effect of moderate dimension where $d > n$ but not $d \gg n$.
We provide partial extensions of our results to $m > 2$ neurons in Appendices~\ref{sec:multiple_relu_gd_main} and~\ref{app:multiple_relu_gd_m} that require a specific ``disjoint" initialization, i.e.,~neurons are partitioned into sets such that they are active on disjoint examples.
Handling more realistic initializations is an important direction for future work.
We also simulate the effect of moderate-dimensional data on the dynamics in Appendix~\ref{app:simulation} and observe that the primal and dual variables intricately influence each other.
We hope to characterize these more complex dynamics in future work, for which we will likely require different mathematical tools.

\section*{Acknowledgements}
KL gratefully acknowledges the support of the ARC-ACO Fellowship provided by Georgia Tech. GW gratefully acknowledges the Apple Scholars in AI/ML PhD fellowship by Apple and ARC-ACO
fellowship provided by Georgia Tech. 
MT gratefully acknowledges the partial support of NSF Grant DMS-2513699, DOE Grants NA0004261, SC0026274, and Richard Duke Fellowship.
VM gratefully acknowledges the support of the NSF (through award CCF-2239151 and award IIS-2212182), an Adobe Data Science Research Award, and an Amazon Research Award.
\newpage

\printbibliography
\newpage

\appendix
\addcontentsline{toc}{section}{Appendix} 
\part{Appendix} 
\parttoc 


    

    
    


\newpage
\section{Proofs of Key Lemmas Tracking Primal-Dual Gradient Dynamics}\label{app:primal_dual_key_lemma}
In this section, we present the proofs of the key lemmas used to track the gradient dynamics of the primal and dual variables. The central factor governing these dynamics is the sign pattern of the primal variables, which determines whether individual examples are active or inactive under the ReLU nonlinearity and, consequently, whether the corresponding dual variables are updated. 

Before presenting the proofs, we first recall two key technical lemmas: 1) a concentration result on the eigenvalues of random Gram matrices in high dimensions from~\citet{bartlett2020benign}; 2) a concentration bound on the operator norm of random Gram matrices from~\citet{hsu2021proliferation}.
Both these lemmas play a crucial role throughout the analysis.

\subsection{Concentration of Random Gram Matrices in High Dimensions}
Our analysis relies heavily on properties of the Gram matrix on high-dimensional data. These concentration results allow us to control cross-sample interactions and isolate the dominant self-interaction terms that drive the gradient updates. As a result, we can rigorously characterize how positivity and negativity patterns in the primal and dual variables evolve over time.

In Lemma~\ref{lem:concentration_eigenvalues}, we characterize the typical behavior of the eigenvalues of a weighted sum of outer products of independent subgaussian vectors. Recall from Section~\ref{sec:setup} that the feature vector $\x\in\R^d$ admits the representation $\x = \V\bLambda^{\frac{1}{2}}\z$, where $\z\in \R^d$ has independent, mean-zero, $\sigma_z^2$-subgaussian components, and we take $\sigma_z=1$. Under this model, the empirical Gram matrix can be written as $\XX = \sum_{j=1}^d \lambda_j \bv_j\bv_j^\top$ where each $\bv_j\in\R^n$ is an independent random vector with independent, mean-zero, subgaussian entries. Concretely, Lemma~\ref{lem:concentration_eigenvalues} provides high-probability bounds on the extreme eigenvalues of $\XX$.
\begin{lemma}[{\citealp[Lemma~9]{bartlett2020benign},~\citealp[Lemma~12]{wang2022binary}}]\label{lem:concentration_eigenvalues}
    There exists a constant $c$ such that with probability at least $1-2e^{-n/c}$, we have
    \begin{align*}
        \frac{1}{c}\sum_{j=1}^d\lambda_j -c\lambda_1n \leq \mu_n(\XX) \leq \mu_1(\XX) \leq c\pts{ \sum_{j=1}^d\lambda_j + \lambda_1 n}.
    \end{align*}
    Moreover, if the effective dimension satisfies $d_{\infty} = \frac{\sum_{j=1}^d \lambda_j}{\lambda_1} \geq bn$ for some constant $b \geq 1$, then there exists a constant $C_g\geq 1$ such that
    \begin{align*}
        \frac{1}{C_g}\sum_{j=1}^d\lambda_j \leq \mu_n(\XX) \leq \mu_1(\XX) \leq C_g \sum_{j=1}^d\lambda_j.
    \end{align*}
    with probability at least $1-2e^{-n/C_g}$.
\end{lemma}
Next, Lemma~\ref{lem:concentration_op_norm} provides a high-probability bound on the operator norm deviation between the Gram matrix $\X \X^\top$ and $\nnorm[1]{\blambda} \I$, which is fruitful for high-dimensional data, and Corollary~\ref{cor:concentration_op_norm} shows that the typical value of this deviation can be expressed in terms of $n$ and effective dimensions $d_2,d_\infty$.

\begin{lemma}[{\citealp[Lemma 8]{hsu2021proliferation}}]\label{lem:concentration_op_norm}
    There exists a universal constant $c > 0$, for any $\tau > 0$,
    \begin{align*}
        \Pr\pts{\nnorm[2]{\XX - \nnorm[1]{\blambda}\I} \geq \tau} \leq 2 \cdot 9^n \cdot \expf{-c\cdot \min\left\{\frac{\tau^2}{\nnorm[2]{\blambda}^2}, \frac{\tau}{\nnorm[\infty]{\blambda}}\right\}},
    \end{align*}
    where $\nnorm[1]{\blambda}\coloneqq\sum_{j=1}^d\lambda_j$, $\nnorm[2]{\blambda}^2\coloneqq\sum_{j=1}^d\lambda_j^2$, and $\nnorm[\infty]{\blambda}\coloneqq\max_{j\in[d]}\lambda_j$.
\end{lemma}
\begin{corollary}\label{cor:concentration_op_norm} 
    With the choice of $\tau = C \cdot \max (\nnorm[2]{\blambda}\sqrt{n}, \nnorm[\infty]{\blambda} n)$ and the constant $C \cdot c > \ln 9$, we have
    \begin{align*}
        \nnorm[2]{\frac{1}{\nnorm[1]{\blambda}}\XX - \I} \leq C \cdot \max \pts{\sqrt{\frac{n}{d_2}}, \frac{n}{d_{\infty}}},
    \end{align*}
    with probability at least $1 - 2 \exp(-n(Cc - \ln 9))$, where we have defined $d_2 \coloneqq\frac{\nnorm[1]{\blambda}^2}{\nnorm[2]{\blambda}^2}, d_{\infty} \coloneqq \frac{\nnorm[1]{\blambda}}{\nnorm[\infty]{\blambda}}$. Similarly, we have
    \begin{align*}
        \nnorm[2]{\nnorm[1]{\blambda}\XXi - \I} \leq C_gC \cdot \max \pts{\sqrt{\frac{n}{d_2}}, \frac{n}{d_{\infty}}},
    \end{align*}
    with probability at least $1 - 2 \exp(-n(Cc - \ln 9))$.
\end{corollary}

    

\subsection{Proof of Lemma~\ref{lem:pos_stay_pos} (Primal Variable Gradient Dynamics in High Dimensions)}\label{app:proof_primal}
In this proof, we show that under the assumptions of the lemma, if the sign of any ReLU neuron agrees with the label of an example, then the corresponding primal variable remains positive after one gradient descent step.
\begin{proof}(Lemma~\ref{lem:pos_stay_pos})
     According to the primal gradient descent update in Equation~\eqref{eq:primal_update} for the $k$-th neuron, we have
    \begin{align*}
        \bbeta_k^{(t+1)} = \bbeta_k^{(t)} - \eta s_k\X \X^\top \D(\bbeta_k^{(t)})(h_{\bTheta^{(t)}}(\X) - \y).
    \end{align*}
    We aim to separate the gradient contribution arising from the diagonal and off-diagonal components of the Gram matrix and to show that the updated primal coordinate remains positive, i.e., $\beta_{k,i}^{(t+1)} > 0$. Fix any $t\geq 0$ and any index $i$ such that $s_k\cdot y_i > 0$ and $\beta_{k,i}^{(t)} > 0$. Then, the update of the $i$-th coordinate can be written as
    \begin{align}
        \beta_{k,i}^{(t+1)} &= \beta_{k,i}^{(t)}-\eta s_k\e_i^\top\XX \D(\bbeta_k^{(t)})(h_{\bTheta^{(t)}}(\X) - \y)\nonumber\\
        &= \beta_{k,i}^{(t)}-\eta s_k\e_i^\top\mts{\nnorm[1]{\blambda}\I + \pts{\XX - \nnorm[1]{\blambda}\I}} \D(\bbeta_k^{(t)})(h_{\bTheta^{(t)}}(\X) - \y)\nonumber\\
        &= \mts{\beta_{k,i}^{(t)} - \eta \nnorm[1]{\blambda}\pts{s_k h_{\bTheta^{(t)}}(\x_i) - s_k y_i}} -\eta s_k\e_i^\top\pts{\XX - \nnorm[1]{\blambda}\I} \D(\bbeta_k^{(t)})(h_{\bTheta^{(t)}}(\X) - \y),\label{eq:beta_update0}
    \end{align}
    where the last equality uses the assumption $\beta_{k,i}^{(t)}>0$, which implies $D_{ii}=\mathbbm{1}_{\beta_{k,i}^{(t)} > 0}=1$. We now lower bound $\beta_{k,i}^{(t+1)}$. By the step size condition $\eta \leq \frac{1}{C_g\nnorm[1]{\blambda}}$ and the assumption $\beta_{k,i}^{(t)} \geq s_k \cdot h_{\bTheta^{(t)}}(\x_i)$, the first term in Equation~\eqref{eq:beta_update0} satisfies
    \begin{align*}
        \beta_{k,i}^{(t)} - \eta \nnorm[1]{\blambda}\pts{s_k h_{\bTheta^{(t)}}(\x_i) - s_k y_i} \geq \eta \nnorm[1]{\blambda}|y_i|.
    \end{align*}
    Substituting this into Equation~\eqref{eq:beta_update0} yields
    \begin{align}
        \eqref{eq:beta_update0}&\geq \eta \nnorm[1]{\blambda}|y_i| -\eta s_k\e_i^\top\pts{\XX - \nnorm[1]{\blambda}\I} \D(\bbeta_k^{(t)})(h_{\bTheta^{(t)}}(\X) - \y)\nonumber\\
        &\geq \eta \nnorm[1]{\blambda}|y_i| -\eta\nnorm[2]{\XX - \nnorm[1]{\blambda}\I} \nnorm[2]{h_{\bTheta^{(t)}}(\X) - \y},\label{eq:beta_update1}
    \end{align}
    where the last inequality follows from the Cauchy–Schwarz inequality and the sub-multiplicativity of the operator norm. Next, we upper bound the second term in Equation~\eqref{eq:beta_update1} using Corollary~\ref{cor:concentration_op_norm}. With probability at least $1 - 2 \exp(-n(Cc - \ln 9))$, we obtain 
    \begin{align*}
        \eqref{eq:beta_update1} &\geq \eta \nnorm[1]{\blambda}\mts{|y_i| - C\cdot \max \pts{\sqrt{\frac{n}{d_2}}, \frac{n}{d_{\infty}}} \nnorm[2]{h_{\bTheta^{(t)}}(\X) - \y}}\\
        &\stackrel{(\mathrm{i})}{\geq} \eta \nnorm[1]{\blambda}\mts{y_{\min}- C\cdot \max \pts{\sqrt{\frac{n}{d_2}}, \frac{n}{d_{\infty}}} \cdot C_y\sqrt{n} y_{\max}}\\
        &\stackrel{(\mathrm{ii})}{\geq} \eta \nnorm[1]{\blambda}\mts{y_{\min}- C\cdot C_y\cdot \frac{y_{\min}}{C_0 y_{\max}}\cdot y_{\max}}\\
        &>0.
    \end{align*}
    Inequality (i) applies the lemma assumption that $\nnorm[2]{h_{\bTheta^{(t)}}(\X) - \y} \leq C_y\nnorm[2]{\y} \leq C_y\sqrt{n}y_{\max}$.
    Inequality (ii) follows from Assumption~\ref{asm:high_dim_data}, which guarantees that $d_2 \geq C_0^2\frac{n^2 y_{\max}^2}{y_{\min}^2}$ and $d_{\infty} \geq C_0\frac{n^{1.5} y_{\max}}{y_{\min}}$ with large enough $C_0 > C \cdot C_y$. This completes the proof of the lemma.
\end{proof}

\subsection{Proof of Lemma~\ref{lem:neg_stay_neg} (Dual Variable Gradient Dynamics in High Dimensions)}\label{app:proof_dual}
In this proof, we show that under the assumptions of the lemma, if the dual variable $\alpha_{k,j}^{(t)}$ for the $k$-th neuron and $j$-th example is sufficiently negative, then it remains unchanged in the next iteration, i.e., $\alpha_{k,j}^{(t+1)} = \alpha_{k,j}^{(t)}$.

\begin{proof}(Lemma~\ref{lem:neg_stay_neg})
    By the definition of primal and dual variables in Equation~\eqref{eq:primal_dual_def}, we have
    \begin{align*}
        \bbeta_k^{(t)} = \XX\balpha_k^{(t)}.
    \end{align*}
    According to the dual gradient update in Equation~\eqref{eq:dual_update}, we have
    \begin{align*}
        \balpha_k^{(t+1)} = \balpha_k^{(t)} - \eta \D(\bbeta_k^{(t)})(h_{\bTheta^{(t)}}(\X) - \y).
    \end{align*}
    This update reveals that each coordinate $\alpha_{k,j}^{(t)}$ evolves independently and is governed by the sign of the corresponding primal variable $\beta_{k,j}^{(t)}$. In particular, if $\beta_{k,j}^{(t)} \leq 0$, then the $j$-th diagonal entry of $\D(\bbeta_k^{(t)})$ vanishes, and consequently $\alpha_{k,j}^{(t+1)} = \alpha_{k,j}^{(t)}$.

    We therefore establish a sufficient condition under which $\beta_{k,j}^{(t)} \leq 0$ in terms of the dual variable $\alpha_{k,j}^{(t)}$. Specifically, we separate the diagonal and off-diagonal components of the Gram matrix as
    \begin{align}
        \beta_{k,j}^{(t)} &= \e_j^\top\XX\balpha_k^{(t)}\nonumber\\
        &= \e_j^\top\mts{\nnorm[1]{\blambda}\I +\pts{\XX - \nnorm[1]{\blambda}\I}}\balpha_k^{(t)}\nonumber\\
        &= \nnorm[1]{\blambda}\alpha_{k,j}^{(t)} + \e_j^\top\pts{\XX - \nnorm[1]{\blambda}\I}\balpha_k^{(t)}\nonumber\\
        &\leq \nnorm[1]{\blambda}\alpha_{k,j}^{(t)} + \nnorm[2]{\XX - \nnorm[1]{\blambda}\I}\nnorm[2]{\balpha_k^{(t)}}\label{eq:beta_upperbound1},
    \end{align}
    where the last inequality follows from the sub-multiplicativity of the operator norm. Next, we upper bound the two terms $\nnorm[2]{\X \X^\top - \nnorm[1]{\blambda} \I}$ and $\nnorm[2]{\balpha_k^{(t)}}$ appearing in Equation~\eqref{eq:beta_upperbound1}. Following the same argument as in the proof of Lemma~\ref{lem:pos_stay_pos}, we apply Corollary~\ref{cor:concentration_op_norm}.
    Consequently, with probability at least $1 - 2 \exp(-n(Cc - \ln 9))$, we obtain
    \begin{align}
        \eqref{eq:beta_upperbound1} 
        &\leq \nnorm[1]{\blambda}\mts{\alpha_{k,j}^{(t)} + C\cdot \max \pts{\sqrt{\frac{n}{d_2}}, \frac{n}{d_{\infty}}}\nnorm[2]{\balpha_k^{(t)}}}.\label{eq:beta_upperbound2}
    \end{align}
    Finally, substituting the upper bound of $\alpha_{k,j}^{(t)}$ and $\nnorm[2]{\balpha_k^{(t)}}$ in lemma assumptions into Equation~\eqref{eq:beta_upperbound2}, we obtain
    \begin{align*}
        \eqref{eq:beta_upperbound2} &\leq \nnorm[1]{\blambda}\mts{-\frac{y_{\min}}{C_{\alpha}\nnorm[1]{\blambda}} + C\cdot \max \pts{\sqrt{\frac{n}{d_2}}, \frac{n}{d_{\infty}}}\frac{C_{\alpha}\sqrt{n}y_{\max}}{\nnorm[1]{\blambda}}}\\
        &= \nnorm[1]{\blambda}\mts{-\frac{y_{\min}}{C_{\alpha}\nnorm[1]{\blambda}} + C\cdot\frac{y_{\min}}{C_0 y_{\max}}\frac{C_{\alpha}y_{\max}}{\nnorm[1]{\blambda}}}\\
        &\leq 0,
    \end{align*}
    where the last inequality follows from Assumption~\ref{asm:high_dim_data}, which ensures $d_2 \geq C_0^2\frac{n^2 y_{\max}^2}{y_{\min}^2}$ and $d_{\infty} \geq C_0\frac{n^{1.5} y_{\max}}{y_{\min}}$ with large enough $C_0 > C \cdot C_{\alpha}^2$. We have thus shown that if $\alpha_{k,j}^{(t)}$ is sufficiently negative, then $\beta_{k,j}^{(t)} \leq 0$, and consequently $\alpha_{k,j}^{(t+1)} = \alpha_{k,j}^{(t)}$. This completes the proof of the lemma.
\end{proof}

\subsection{Lemma~\ref{lem:pos_stay_pos_fixed} and Lemma~\ref{lem:neg_stay_neg_fixed} (Gradient Dynamics under Deterministic Features)}\label{app:primal_fixed}
In this section, we present the gradient dynamics of the primal and dual variables under deterministic feature assumptions. Lemma~\ref{lem:pos_stay_pos_fixed} serves as the deterministic-feature counterpart of Lemma~\ref{lem:pos_stay_pos}.

\begin{lemma}\label{lem:pos_stay_pos_fixed}
    Under Assumptions~\ref{asm:label_bounds} and~\ref{asm:deterministic_data}, suppose the gradient descent step size satisfies $\eta \leq \frac{1}{\mu_1\pts{\XX}}$. Consider the $k$-th ReLU neuron in $h_{\bTheta}$. For any $t \geq 0$ and any index $i\in[n]$ such that $s_k \cdot y_i > 0$, if $\beta_{k,i}^{(t)} > 0$, $\beta_{k,i}^{(t)} \geq s_k \cdot h_{\bTheta^{(t)}}(\x_i)$, and $\nnorm[2]{h_{\bTheta^{(t)}}(\X) - \y} \leq C_y\nnorm[2]{\y}$, then $\beta_{k,i}^{(t+1)} > 0$.
\end{lemma}

\begin{proof}(Lemma~\ref{lem:pos_stay_pos_fixed})
     According to the primal gradient descent update in Equation~\eqref{eq:primal_update} for the $k$-th neuron, we have
    \begin{align*}
        \bbeta_k^{(t+1)} = \bbeta_k^{(t)} - \eta s_k\X \X^\top \D(\bbeta_k^{(t)})(h_{\bTheta^{(t)}}(\X) - \y).
    \end{align*}
    We aim to separate the gradient contribution arising from the diagonal and off-diagonal components of the Gram matrix and to show that the updated primal coordinate remains positive, i.e., $\beta_{k,i}^{(t+1)} > 0$. Fix any $t\geq 0$ and any index $i$ such that $s_k\cdot y_i > 0$ and $\beta_{k,i}^{(t)} > 0$. Then, the update of the $i$-th coordinate can be written as
    \begin{align}
        \beta_{k,i}^{(t+1)} &= \beta_{k,i}^{(t)}-\eta s_k\e_i^\top\XX \D(\bbeta_k^{(t)})(h_{\bTheta^{(t)}}(\X) - \y)\nonumber\\
        &= \mts{\beta_{k,i}^{(t)} - \eta s_k\nnorm[2]{\x_i}^2 \mathbbm{1}_{\beta_{k,i}^{(t)}>0}\pts{h_{\bTheta^{(t)}}(\x_i) - y_i}} - \eta s_k\sum_{j\neq i}\x_i^\top\x_j \mathbbm{1}_{\beta_{k,j}^{(t)}>0}\pts{h_{\bTheta^{(t)}}(\x_j) - y_j},\label{eq:beta_fixed_update0}
    \end{align}
    where we have $\mathbbm{1}_{\beta_{k,i}^{(t)}>0} = 1$ according to the assumption $\beta_{k,i}^{(t)}>0$. We now lower bound $\beta_{k,i}^{(t+1)}$. By the step size condition $\eta \leq \frac{1}{\mu_1\pts{\XX}}\leq \frac{1}{\nnorm[2]{\x_i}^2}$ for all $i\in[n]$ and the assumption $\beta_{k,i}^{(t)} \geq s_k \cdot h_{\bTheta^{(t)}}(\x_i)$, the first term in Equation~\eqref{eq:beta_fixed_update0} satisfies
    \begin{align*}
        \beta_{k,i}^{(t)} - \eta \nnorm[2]{\x_i}^2\pts{s_k h_{\bTheta^{(t)}}(\x_i) - s_k y_i} \geq \eta \nnorm[2]{\x_i}^2|y_i|.
    \end{align*}
    Substituting this into Equation~\eqref{eq:beta_fixed_update0} yields
    \begin{align}
        \eqref{eq:beta_fixed_update0}&\geq \eta \nnorm[2]{\x_i}^2|y_i| - \eta s_k\sum_{j\neq i}\x_i^\top\x_j \mathbbm{1}_{\beta_{k,j}^{(t)}>0}\pts{h_{\bTheta^{(t)}}(\x_j) - y_j}\nonumber\\
        &\geq \eta \nnorm[2]{\x_i}^2|y_i| -\eta \max_{i\neq j}|\x_i^\top\x_j|\sum_{j\neq i}\left|h_{\bTheta^{(t)}}(\x_j) - y_j\right|\nonumber\\
        &\geq \eta \nnorm[2]{\x_i}^2|y_i| -\eta \max_{i\neq j}|\x_i^\top\x_j|\sqrt{n}\nnorm[2]{h_{\bTheta^{(t)}}(\X) - \y},\label{eq:beta_fixed_update1}
    \end{align}
    where we take absolute values in the second inequality, and the last inequality follows from the inequality between $\ell_1$ and $\ell_2$ norms. Next, we upper bound the second term in Equation~\eqref{eq:beta_fixed_update1} by the lemma assumption $\nnorm[2]{h_{\bTheta^{(t)}}(\X) - \y} \leq C_y\nnorm[2]{\y} \leq C_y\sqrt{n}y_{\max}$.  We obtain 
    \begin{align*}
        \eqref{eq:beta_fixed_update1} &\geq \eta \nnorm[2]{\x_i}^2|y_i| -\eta \max_{i\neq j}|\x_i^\top\x_j|C_y n y_{\max}\\
        &\geq \eta \mts{\min_{i\in[n]}\nnorm[2]{\x_i}^2 y_{\min}-  C_y n \max_{i\neq j}|\x_i^\top\x_j| y_{\max}}\\
        &>0.
    \end{align*}
    The second inequality follows from taking the minimum for $i\in[n]$, and the last inequality follows from Assumption~\ref{asm:deterministic_data} with $C_0 > C_y$. This completes the proof of the lemma.
\end{proof}

Next, we present the gradient dynamics of the dual variables under deterministic feature assumptions. Lemma~\ref{lem:neg_stay_neg_fixed} serves as the deterministic-feature counterpart of Lemma~\ref{lem:neg_stay_neg}.

\begin{lemma}\label{lem:neg_stay_neg_fixed}
    Under Assumptions~\ref{asm:label_bounds} and~\ref{asm:deterministic_data}, consider the $k$-th ReLU neuron in $h_{\bTheta}$. For any $t\geq 0$ and any index $j\in[n]$, if $\alpha_{k,j}^{(t)} \leq -\frac{y_{\min}}{C_{\alpha}\nnorm[1]{\blambda}}$ and $\nnorm[2]{\balpha_k^{(t)}} \leq \frac{C_{\alpha}\sqrt{n}y_{\max}}{\nnorm[1]{\blambda}}$, then $\beta_{k,j}^{(t)} \leq 0$ and $\alpha_{k,j}^{(t+1)} = \alpha_{k,j}^{(t)}$.
\end{lemma}

\begin{proof}(Lemma~\ref{lem:neg_stay_neg_fixed})
    By the definition of primal and dual variables in Equation~\eqref{eq:primal_dual_def}, we have
    \begin{align*}
        \bbeta_k^{(t)} = \XX\balpha_k^{(t)}.
    \end{align*}
    According to the dual gradient update in Equation~\eqref{eq:dual_update}, we have
    \begin{align*}
        \balpha_k^{(t+1)} = \balpha_k^{(t)} - \eta \D(\bbeta_k^{(t)})(h_{\bTheta^{(t)}}(\X) - \y).
    \end{align*}
    This update reveals that each coordinate $\alpha_{k,j}^{(t)}$ evolves independently and is governed by the sign of the corresponding primal variable $\beta_{k,j}^{(t)}$. In particular, if $\beta_{k,j}^{(t)} \leq 0$, then the $j$-th diagonal entry of $\D(\bbeta_k^{(t)})$ vanishes, and consequently $\alpha_{k,j}^{(t+1)} = \alpha_{k,j}^{(t)}$.

    We therefore establish a sufficient condition under which $\beta_{k,j}^{(t)} \leq 0$ in terms of the dual variable $\alpha_{k,j}^{(t)}$. Specifically, we separate the diagonal and off-diagonal components of the Gram matrix as
    \begin{align}
        \beta_{k,j}^{(t)} &= \e_j^\top\XX\balpha_k^{(t)}\nonumber\\
        &= \nnorm[2]{\x_j}^2\alpha_{k,j}^{(t)} + \sum_{j\neq i} \x_j^\top\x_i \alpha_{k,i}^{(t)}\nonumber\\
        &\leq \nnorm[2]{\x_j}^2\alpha_{k,j}^{(t)} + \max_{j\neq i}|\x_j^\top\x_i|\sum_{j\neq i}|\alpha_{k,i}^{(t)}|\label{eq:beta_fixed_upperbound1},
    \end{align}
    where the last inequality takes the maximum and absolute values for the second term. 
    Furthermore, by applying the inequality between $\ell_2$ and $\ell_1$ norms and taking the minimum for $j\in[n]$ in the first term, we obtain
    \begin{align}
        \eqref{eq:beta_fixed_upperbound1} 
        &\leq \pts{\min_{j\in[n]}\nnorm[2]{\x_j}^2}\alpha_{k,j}^{(t)} + \max_{j\neq i}|\x_j^\top\x_i|\sqrt{n}\nnorm[2]{\balpha_k^{(t)}}.\label{eq:beta_fixed_upperbound2}
    \end{align}
    Finally, substituting the upper bound of $\alpha_{k,j}^{(t)}$ and $\nnorm[2]{\balpha_k^{(t)}}$ in lemma assumptions into Equation~\eqref{eq:beta_fixed_upperbound2}, we obtain
    \begin{align*}
        \eqref{eq:beta_fixed_upperbound2} &\leq -\pts{\min_{j\in[n]}\nnorm[2]{\x_j}^2}\frac{y_{\min}}{C_{\alpha}\nnorm[1]{\blambda}} + \max_{j\neq i}|\x_j^\top\x_i| \sqrt{n}\frac{C_{\alpha}\sqrt{n}y_{\max}}{\nnorm[1]{\blambda}}\\
        &\leq 0,
    \end{align*}
    where the last inequality follows from Assumption~\ref{asm:deterministic_data} with large enough $C_0 > C \cdot C_{\alpha}^2$. We have thus shown that if $\alpha_{k,j}^{(t)}$ is sufficiently negative, then $\beta_{k,j}^{(t)} \leq 0$, and consequently $\alpha_{k,j}^{(t+1)} = \alpha_{k,j}^{(t)}$. This completes the proof of the lemma.
\end{proof}

\newpage
\section{Proofs for the Single ReLU model (\texorpdfstring{$m=1$}{m=1}) Trained with Gradient Descent}\label{app:single_relu_gd}

In this section, we present the proofs concerning the behavior of the single ReLU model trained with gradient descent.

\subsection{Proofs of Lemmas~\ref{lem:final_phase},~\ref{lem:conv_step_size} and~\ref{lem:relu_mni_subset} (Gradient Descent Convergence and \texorpdfstring{$\w^\star$}{w*})}\label{app:proof_single_relu_convergence_mni_in_general}

We present complete proofs of the gradient descent convergence for single ReLU models in Lemmas~\ref{lem:final_phase} and~\ref{lem:conv_step_size}, as well as a characterization of the minimum-$\ell_2$-norm solution in Lemma~\ref{lem:relu_mni_subset}.

\begin{proof}(Lemma~\ref{lem:final_phase})
    We prove this lemma by showing that after iteration $t_0 \geq 0$, since the activation pattern is fixed, the gradient of the single ReLU model is equivalent to the gradient of a linear model using only a subset of examples. Consider a linear model
    \begin{align*}
        h(\x) =  \w^\top\x,
    \end{align*}
    where $\w \in\R^d$ is the linear model parameter (also called weight). Let $S\subseteq [n]$ denote the active set for the single ReLU model at iteration $t_0$, defined by $S \coloneqq \{ i\in [n] : \x_i^\top\w^{(t_0)} > 0\}$. We write the empirical risk with the linear model using only the examples in $S$ as
    \begin{align*}
        \Risk_S(\w) = \frac{1}{2}\sum_{i\in S} ( \w^\top\x_i - y_i)^2.
    \end{align*}
    The gradient descent update for this linear model is
    \begin{align}
        \w^{(t+1)} &= \w^{(t)} - \eta\nabla\Risk_S(\w^{(t)})\nonumber\\
        &= \w^{(t)} -\eta\sum_{i\in S} ( \w^{(t)\top}\x_i - y_i)\x_i\label{eq:final_phase_linear}.
    \end{align}
    On the other hand, the original gradient descent dynamic for the single ReLU model (Equation~\ref{eq:single_relu_gd_update}) tells us that
    \begin{align*}
        \w^{(t+1)} &= \w^{(t)} - \eta \X^\top \D(\X\w^{(t)})(\X\w^{(t)} - \y).
    \end{align*}
    Under the lemma assumption, $\D(\X\w^{(t_0)})=\D(\X\w^{(t)})$ for all $t \geq t_0$. Thus, we know that $D_{ii} = \mathbbm{1}_{i\in S}$ for all $t \geq t_0$. Therefore, for $t \geq t_0$, we can write the gradient update of the original single ReLU model as
    \begin{align*}
        \w^{(t+1)} &= \w^{(t)} - \eta \X^\top \D(\X\w^{(t)})(\X\w^{(t)} - \y)\\
        &=\w^{(t)} -\eta\sum_{i\in S} ( \w^{(t)\top}\x_i - y_i)\x_i.
    \end{align*}
    This gradient update is equivalent to the gradient update of the linear model in Equation~\eqref{eq:final_phase_linear} for all $t \geq t_0$. As a result, for $t \geq t_0$, the gradient update of the single ReLU model is equivalent to a linear model using only data in $S$. This completes the proof of the lemma.
\end{proof}

\begin{proof}(Lemma~\ref{lem:conv_step_size})
    By Lemma~\ref{lem:final_phase}, the activation pattern is fixed for all $t \geq t_0$, so the gradient descent update reduces to linear regression restricted to the active subset $S$, given by
    \begin{align*}
        \w^{(t+1)} &= \w^{(t)} - \eta \X^\top \D(\X\w^{(t)})(\X\w^{(t)} - \y)\\
        &=\w^{(t)} -\eta\sum_{i\in S} ( \w^{(t)\top}\x_i - y_i)\x_i\\
        &= \w^{(t)} -\eta\X_S^\top (\X_S\w^{(t)} - \y_S).
    \end{align*}
    The final phase empirical risk is given by
    \begin{align*}
       \Risk(\w) = \frac{1}{2}\nnorm[2]{\X_S\w-\y_S}^2 + \frac{1}{2} \nnorm[2]{\y_{S^{\mathsf{c}}}}^2,
    \end{align*}
    where the second term comes from the examples in $S^{\mathsf{c}}$ with negative pre-activations, and it does not depend on $\w$ because the activation pattern does not change after $t_0$. Note that $\Risk(\w)$ is a convex quadratic with
    \begin{align*}
        \nabla\Risk(\w) = \X_S^\top(\X_S\w-\y_S), \quad\quad  \nabla^2\Risk(\w) = \X_S^\top\X_S.
    \end{align*}
    Therefore, $\Risk$ is $L$-smooth with 
    \begin{align*}
        L =\nnorm[2]{\nabla^2\Risk(\w)} = \nnorm[2]{\X_S^\top\X_S} = \mu_1(\X_S\X_S^\top).
    \end{align*}
    A standard smoothness/descent result (e.g., \citealt[Equation~9.17]{boyd2004convex}) implies that for any $\eta \leq \frac{1}{L}$,
    \begin{align*}
        \Risk(\w^{(t+1)}) \leq \Risk(\w^{(t)}) - \frac{\eta}{2} \nnorm[2]{\nabla\Risk(\w^{(t)})}^2,
    \end{align*}
    and in particular, $\Risk(\w^{(t)})$ is non-increasing for all $t \geq t_0$.

    It remains to upper bound $L$. Since $S$ is a subset of the training indices, $|S| \leq n$. Since $L = \mu_1(\X_S\X_S^\top) \leq \mu_1(\X\X^\top)$,
    choosing $\eta \leq \frac{1}{\mu_1(\X\X^\top)}$ guarantees that $\Risk(\w^{(t)})$ is non-increasing for all $t \geq t_0$. This establishes the desired step size condition in the final phase (and thus convergence in function value for the single ReLU dynamics after $t_0$).

    Finally, according to~\citet[Section~2.1]{gunasekar2018characterizing}, the set of minimizers of $\Risk(\w)$ is the affine subspace,
    \begin{align*}
        \mathcal{W}_S = \{\w:\X_S\w =\y_S\},
    \end{align*}
    and gradient descent with constant step size converges to the Euclidean projection of the initialization $\w^{(t_0)}$ onto this subspace $\w^{(\infty)} =\uargmin{\w\in\mathcal{W}_S} \nnorm[2]{\w - \w^{(t_0)}}$. This completes the proof of the lemma.
\end{proof}

\begin{proof}(Lemma~\ref{lem:relu_mni_subset})
    We prove the lemma by showing that the optimal solution $\w^\star$ of the original convex program for single ReLU models also solves a reduced convex program whose solution is the minimum-$\ell_2$-norm interpolation (MNI) over an index subset $S\subseteq[n]$ with modified labels. First, we restate the convex program~\eqref{eq:single_relu_minimum_norm_sol} and its KKT conditions below:
    \begin{align*}
        \w^\star \in {} &\uargmin{\w}\frac{1}{2}\nnorm[2]{\w}^2\\
        \text{s.t. } \w^\top\x_i &= y_i, \text{ for all } i\in S_1\nonumber,\\
        \w^\top\x_j &\leq 0, \text{ for all } j\in S_2\nonumber,
    \end{align*}
    where we denote $S_1 = \{i: y_i > 0, \text{ for all }i\in[n]\}$, $S_2 = \{j: y_j \leq 0, \text{ for all }j\in[n]\}$ and $S_1 \cup S_2 = [n]$. Since $n\leq d$ and we have assumed $\mathrm{rank}(\X)=n$, we can always find a feasible solution satisfying all $n$ equality constraints. This implies that the solution set is nonempty, and $\w^\star$ always exists. Hence, the following KKT conditions are necessary (and also sufficient) to $\w^\star$ for some $\blambda^\star\in \R^{|S_1|}$ and $\bmu^\star\in \R^{|S_2|}$:\\
    \textbf{Stationarity:}
    \begin{align*}
        \w^\star + \sum_{i\in S_1}\lambda^\star_i\x_i + \sum_{j\in S_2}\mu^\star_j\x_j = 0 \Leftrightarrow \w^\star = -\sum_{i\in S_1}\lambda^\star_i\x_i - \sum_{j\in S_2}\mu^\star_j\x_j.
    \end{align*}
    \textbf{Primal feasibility:}
    \begin{align*}
        \w^{\star\top}\x_i &= y_i, \text{ for all } i\in S_1,\\
        \w^{\star\top}\x_j &\leq 0, \text{ for all } j\in S_2.
    \end{align*}
    \textbf{Dual feasibility:}
    \begin{align*}
        \lambda^\star_i \in \R, \text{ for all } i\in S_1,\\
        \mu^\star_j \geq 0, \text{ for all }j\in S_2.
    \end{align*}
    \textbf{Complementary slackness:}
    \begin{align*}
        \sum_{j\in S_2}\mu^\star_j\pts{\w^{\star\top}\x_j} = 0.
    \end{align*}
    Next, we further denote a subset $\tilde{S}_2 \subseteq S_2$ such that $\tilde{S}_2 = \{j: \mu^\star_j > 0 \text{ for all } j\in S_2\}$ (note that $\tilde{S}_2$ can be empty). By the KKT conditions, it is necessary for $\w^\star$ to satisfy the following: 
    \begin{subequations}\label{eq:kkt_nec_cond}
    \begin{align}
        \w^\star = -\sum_{i\in S_1}\lambda^\star_i\x_i &- \sum_{j\in \tilde{S}_2}\mu^\star_j\x_j, \text{ with } \lambda^\star_i \in \R \text{ and } \mu^\star_j > 0,\\
        \w^{\star\top}\x_i &= y_i, \text{ for all } i\in S_1,\\
        \w^{\star\top}\x_j &= 0, \text{ for all } j\in \tilde{S}_2.
    \end{align}
    \end{subequations}
    Now, we consider a reduced convex program:
    \begin{align}\label{eq:reduced_kkt}
        \tilde{\w} \in {}&\uargmin{\w}\frac{1}{2}\nnorm[2]{\w}^2\\
        \text{s.t. } \w^\top\x_i &= y_i, \text{ for all } i\in S_1,\nonumber\\
        \w^\top\x_j &= 0, \text{ for all } j\in \tilde{S}_2\nonumber.
    \end{align}
    Its KKT conditions are given below.\\
    \textbf{Stationarity:}
    \begin{align*}
        \tilde{\w} + \sum_{i\in S_1}\tilde{\lambda}_i\x_i + \sum_{j\in \tilde{S}_2}\tilde{\lambda}_j\x_j = 0 \Leftrightarrow \tilde{\w} = -\sum_{i\in S_1}\tilde{\lambda}_i\x_i - \sum_{j\in \tilde{S}_2}\tilde{\lambda}_j\x_j.
    \end{align*}
    \textbf{Primal feasibility:}
    \begin{align*}
        {\tilde{\w}}^\top\x_i &= y_i, \text{ for all } i\in S_1,\\
        {\tilde{\w}}^\top\x_j &= 0, \text{ for all } j\in \tilde{S}_2.
    \end{align*}
    \textbf{Dual feasibility:}
    \begin{align*}
        \tilde{\lambda}_i \in \R, \text{ for all } i\in S_1,\\
        \tilde{\lambda}_j \in \R, \text{ for all }j\in \tilde{S}_2.
    \end{align*}
    Since $\w^\star$ satisfies all the conditions in Equation~\eqref{eq:kkt_nec_cond}, it also satisfies the KKT conditions for the reduced convex program~\eqref{eq:reduced_kkt}. Thus, $\w^\star$ is also the optimal solution of the reduced convex program. Finally, we have a closed-form solution for the reduced convex program such that $\w^\star=\tilde{\w} = \w_{\text{linear-MNI},S} = \X_S^\top(\X_S\X_S^\top)^{-1}\tilde{\y}_S$ where $S = S_1 \cup \tilde{S}_2$ and $\tilde{\y}_S$ denotes the corresponding label subvector with all negative entries replaced by zero. This completes the proof of the lemma.
\end{proof}

\subsection{Proof of Theorem~\ref{thm:single_relu_gd_high_dim_implicit_bias} (High-dimensional Implicit Bias)}\label{app:proof_single_relu_gd_high_dim_implicit_bias}

In this section, we present the proof of Theorem~\ref{thm:single_relu_gd_high_dim_implicit_bias}. For the single ReLU model ($m=1$), the primal-dual gradient update in~\eqref{eq:primal_dual_update} simplifies to
\begin{subequations}
\begin{alignat}{2}
    &\text{(Primal) }\myquad[7] &&\bbeta^{(t+1)} = \bbeta^{(t)} - \eta \X \X^\top \D(\bbeta^{(t)})(\bbeta^{(t)} - \y),\myquad[6]\label{eq:single_primal_update} \\
    &\text{(Dual) }\myquad[7] &&\balpha^{(t+1)} = \balpha^{(t)} - \eta \D(\bbeta^{(t)})(\bbeta^{(t)} - \y)\label{eq:single_dual_update}.\myquad[6]
\end{alignat}
\end{subequations}

Before proceeding to the proof, we introduce a set of sufficient conditions under which the signs of the primal variables agree with the signs of the labels at iteration $t$. Moreover, these conditions are preserved at iteration $t+1$.

\begin{lemma}\label{lem:primal_label_same_sign}
    Under Assumptions~\ref{asm:label_bounds} and~\ref{asm:high_dim_data}, suppose the gradient descent step size satisfies $\eta \leq \frac{1}{C_g\nnorm[1]{\blambda}}$. For any single ReLU model, if the following six conditions hold at some iteration $t\geq 0$, then they also hold at iteration $t+1$. 
    \begin{enumerate}[label=\alph*., ref=\alph*]
        \item $\beta_i^{(t)} > 0$, for all $i\in[n]$ with $y_i >0$\label{eq:cond1}.
        \item $-\frac{3y_{\max}}{C_g\nnorm[1]{\blambda}}\leq \alpha_j^{(t)} \leq -\frac{y_{\min}}{C_{\alpha}\nnorm[1]{\blambda}}$, for all $j\in[n]$ with $y_j < 0$\label{eq:cond2}.
        \item $\nnorm[2]{\bbeta_+^{(t)} - \y_+} \leq C_y\nnorm[2]{\y_+}$\label{eq:cond3}.
        \item $\nnorm[2]{\balpha^{(t)}} \leq \frac{C_{\alpha}\sqrt{n}y_{\max}}{\nnorm[1]{\blambda}}$\label{eq:cond4}.
        \item $\beta_j^{(t)} \leq 0$, for all $j\in[n]$ with $y_j < 0$\label{eq:cond5}.
        \item $\sigma(\bbeta^{(t)}) = \begin{bmatrix} \bbeta_+^{(t)} \\ \zero \end{bmatrix}$\label{eq:cond6}.
    \end{enumerate}
    Consequently, the set of active examples consists exactly of the positively labeled examples, and the activation pattern remains unchanged, i.e., $\D(\bbeta^{(t)}) = \D(\bbeta^{(t+1)})$.
\end{lemma}

\begin{proof}(Lemma~\ref{lem:primal_label_same_sign})
    In the following, we show that if the six sufficient conditions hold at some iteration $t\geq 0$, then they also hold at iteration $t+1$.
    \begin{enumerate}[label=Part (\alph*):,leftmargin=3\parindent]
        \item By conditions~\eqref{eq:cond3} and~\eqref{eq:cond6} at iteration $t$, we have $\nnorm[2]{h_{\bTheta^{(t)}}(\X) - \y}^2 = \nnorm[2]{\sigma(\bbeta^{(t)})- \y}^2 = \nnorm[2]{\bbeta_+^{(t)} - \y_+}^2 +\nnorm[2]{\y_-}^2 \leq C_y^2\nnorm[2]{\y}^2$. Together with $h_{\bTheta^{(t)}}(\x_i) = \beta_i^{(t)}$ and condition~\eqref{eq:cond1}, all the assumptions of Lemma~\ref{lem:pos_stay_pos} are satisfied for all $i$ with $y_i > 0$. Consequently, we obtain $\beta_i^{(t+1)} > 0$ for all $i\in[n]$ with $y_i >0$, and thus condition~\eqref{eq:cond1} holds at iteration $t+1$.
        
        \item According to the dual gradient update in Equation~\eqref{eq:single_dual_update}, and using condition~\eqref{eq:cond5} at iteration $t$, we conclude that the dual variables corresponding to negatively labeled examples remain unchanged, i.e., $\alpha_j^{(t+1)} = \alpha_j^{(t)}$ for all $j\in[n]$ with $y_j < 0$. Therefore, condition~\eqref{eq:cond2} continues to hold at iteration $t+1$.

        \item By conditions~\eqref{eq:cond1} and~\eqref{eq:cond5}, the gradient update at iteration $t$ depends only on the positively labeled examples. Consequently, the update is equivalent to a linear regression gradient descent step using only the positive-labeled subset. As similarly argued in the proof of Lemma~\ref{lem:conv_step_size}, since the step size satisfies $\eta \leq \frac{1}{C_g\nnorm[1]{\blambda}}$, the squared loss is monotonically non-increasing, and we obtain $\nnorm[2]{\bbeta_+^{(t+1)} -\y_+} \leq \nnorm[2]{\bbeta_+^{(t)} -\y_+} \leq C_y\nnorm[2]{\y_+}$ by condition~\eqref{eq:cond3} at iteration $t$.
        Therefore, condition~\eqref{eq:cond3} holds at iteration $t+1$.

        \item For this part, we use conditions~\eqref{eq:cond2} and~\eqref{eq:cond3} at iteration $t+1$. By the triangle inequality, we have
        \begin{align*}
            \nnorm[2]{\balpha^{(t+1)}} \leq \nnorm[2]{\balpha_+^{(t+1)}} + \nnorm[2]{\balpha_-^{(t+1)}}.
        \end{align*}
        By condition~\eqref{eq:cond2} at iteration $t+1$, it follows that $\nnorm[2]{\balpha_-^{(t+1)}} \leq \frac{3\sqrt{n}y_{\max}}{C_g\nnorm[1]{\blambda}}$. It therefore remains to upper bound $\nnorm[2]{\balpha_+^{(t+1)}}$. By condition~\eqref{eq:cond3} at iteration $t+1$, we have $\nnorm[2]{\bbeta_+^{(t+1)}} \leq C_y\nnorm[2]{\y_+} + \nnorm[2]{\y_+} \leq (C_y+1)\nnorm[2]{\y}$. Moreover, we have
        \begin{align*}
            \nnorm[2]{\bbeta_+^{(t+1)}} &= \nnorm[2]{\X_+\X^\top\balpha^{(t+1)}}\\
            &= \nnorm[2]{\X_+\begin{bmatrix} \X_+^\top  \X_-^\top\end{bmatrix}\begin{bmatrix} \balpha_+^{(t+1)}\\ \balpha_-^{(t+1)}\end{bmatrix}}\\
            &= \nnorm[2]{\X_+\X_+^\top\balpha_+^{(t+1)} + \X_+\X_-^\top\balpha_-^{(t+1)}}.
        \end{align*}
        Applying the triangle inequality yields
        \begin{align*}
            \nnorm[2]{\X_+\X_+^\top\balpha_+^{(t+1)}} &\leq \nnorm[2]{\bbeta_+^{(t+1)}} + \nnorm[2]{\X_+\X_-^\top\balpha_-^{(t+1)}}\\
            &\leq (C_y+1)\nnorm[2]{\y} + \nnorm[2]{\X_+\X_-^\top\balpha_-^{(t+1)}}.
        \end{align*}
        Since $\X_+\X_+^\top\in\R^{n_+\times n_+}$ is full rank, we obtain
        \begin{align*}
             \nnorm[2]{\balpha_+^{(t+1)}} &\leq \frac{(C_y+1)\nnorm[2]{\y} + \nnorm[2]{\X_+\X_-^\top\balpha_-^{(t+1)}}}{\mu_{n_+}(\X_+\X_+^\top)}.
        \end{align*}
        For the denominator, the variational formulation for eigenvalues of a submatrix and Lemma~\ref{lem:concentration_eigenvalues} imply that 
        \begin{align*}
            \mu_{n_+}(\X_+\X_+^\top) \geq \mu_{n}(\XX) \geq \frac{1}{C_g}\sum_{j=1}^d\lambda_j=\frac{\nnorm[1]{\blambda}}{C_g},
        \end{align*}
        with probability at least $1 -2e^{-n/C_g}$.
        For the numerator, we have $(C_y+1)\nnorm[2]{\y} \leq (C_y+1)\sqrt{n}y_{\max}$. Moreover, by~\citet[Theorem 1]{bhatia1990singular}, we have
        \begin{align*}
            \nnorm[2]{\X_+\X_-^\top} &\leq \frac{1}{2} \nnorm[2]{\X_+\X_+^\top + \X_-\X_-^\top}\\
            &\leq \frac{1}{2} \pts{\nnorm[2]{\X_+\X_+^\top} + \nnorm[2]{\X_-\X_-^\top}}\\
            &\leq C_g\sum_{j=1}^d\lambda_j\\
            &= C_g\nnorm[1]{\blambda},
        \end{align*}
        where the last inequality follows from Lemma~\ref{lem:concentration_eigenvalues}. Combining these bounds yields
        \begin{align*}
            \nnorm[2]{\balpha_+^{(t+1)}} \leq \frac{(C_y+1)\sqrt{n}y_{\max} + C_g\nnorm[1]{\blambda} \cdot \frac{3\sqrt{n}y_{\max}}{C_g\nnorm[1]{\blambda}}}{\nnorm[1]{\blambda}/C_g} = ((C_y+1)C_g+3C_g)\frac{\sqrt{n}y_{\max}}{\nnorm[1]{\blambda}}.
        \end{align*}
        Consequently, we have
        \begin{align*}
            \nnorm[2]{\balpha^{(t+1)}} \leq  ((C_y+1)C_g+3C_g)\frac{\sqrt{n}y_{\max}}{\nnorm[1]{\blambda}} + \frac{3\sqrt{n}y_{\max}}{C_g\nnorm[1]{\blambda}} \leq \frac{C_{\alpha}\sqrt{n}y_{\max}}{\nnorm[1]{\blambda}},
        \end{align*}
        with $C_{\alpha} \gtrsim \max\{C_g^2, C_yC_g\}$, and thus condition~\eqref{eq:cond4} holds at iteration $t+1$.

        \item By Lemma~\ref{lem:neg_stay_neg}, and since conditions~\eqref{eq:cond2} and~\eqref{eq:cond4} hold at iteration $t+1$, we conclude that $\beta_j^{(t+1)} \leq 0$ for all $j\in[n]$ with $y_j < 0$. Thus, condition~\eqref{eq:cond5} holds at iteration $t+1$.

        \item By conditions~\eqref{eq:cond1} and~\eqref{eq:cond5} at iteration $t+1$, the signs of the primal variables continue to agree with the signs of the labels. Consequently, $\sigma(\bbeta^{(t+1)}) = \begin{bmatrix} \bbeta_+^{(t+1)} \\ \zero \end{bmatrix}$, and thus condition~\eqref{eq:cond6} holds at iteration $t+1$.
    \end{enumerate}
    We have shown that the six sufficient conditions hold at iteration $t+1$. Consequently, the signs of the primal variables continue to agree with the signs of the labels, and hence $\D(\bbeta^{(t)}) = \D(\bbeta^{(t+1)})$. This completes the proof.
\end{proof}

Equipped with Lemma~\ref{lem:primal_label_same_sign}, we are now ready to prove Theorem~\ref{thm:single_relu_gd_high_dim_implicit_bias}.
\begin{proof}(Theorem~\ref{thm:single_relu_gd_high_dim_implicit_bias})
    In the proof, we first show that after the first gradient step, the iterate at $t=1$ satisfies the conditions in Lemma~\ref{lem:primal_label_same_sign}. Next, since the conditions hold at $t=1$ and are preserved from $t=\tilde{t}$ to $t=\tilde{t} +1$ by Lemma~\ref{lem:primal_label_same_sign}, we fully characterize the gradient descent dynamics by induction.
    
    We begin by verifying that the iterate at $t=1$ satisfies the sufficient conditions in Lemma~\ref{lem:primal_label_same_sign}. With the initialization $\w^{(0)} = \X^\top(\XX)^{-1} \bepsilon$, we have $\bbeta^{(0)} = \X\w^{(0)} = \bepsilon$. Therefore, using the primal gradient update in Equation~\eqref{eq:single_primal_update}, we obtain
    \begin{align}
        \bbeta^{(1)} &= \bbeta^{(0)} - \eta \X \X^\top \D(\bbeta^{(0)})(\bbeta^{(0)} - \y)\nonumber\\
        &= \bepsilon -\eta \X \X^\top (\bepsilon - \y)\nonumber\\
        &= \X \X^\top\mts{\underbrace{\eta\pts{\y-\bepsilon+\frac{1}{\eta}\XXi\bepsilon}}_{\eqqcolon\balpha^{(1)}}}.\label{eq:single_beta_one}
    \end{align}
    We denote $\balpha^{(1)} \coloneqq\eta\pts{\y-\bepsilon+\frac{1}{\eta}(\XX)^{-1}\bepsilon}$ according to the primal-dual formulation $\bbeta^{(1)} =\XX\balpha^{(1)}$ in Equation~\eqref{eq:primal_dual_def}. In the below, we show that at iteration $t=1$, the variables $\bbeta^{(1)}$ and $\balpha^{(1)}$ satisfy all the conditions in Lemma~\ref{lem:primal_label_same_sign}.
    \begin{enumerate}[label=Part (\alph*):,leftmargin=3\parindent]
        \item For all $i\in[n]$ with $y_i > 0$, we apply Lemma~\ref{lem:pos_stay_pos}. Since $\beta_i^{(0)} = \epsilon_i > 0$, $\beta_i^{(0)} = h_{\bTheta^{(0)}}(\x_i)$ and $\nnorm[2]{\sigma(\bbeta^{(0)}) -\y} \leq \nnorm[2]{\bepsilon} + \nnorm[2]{\y} \leq \frac{\sqrt{n}}{C_{\alpha}}y_{\min} + \nnorm[2]{\y} \leq C_y\nnorm[2]{\y}$ with $C_y > 1 + \frac{1}{C_{\alpha}}$, it follows that $\beta_i^{(1)}>0$ for all $i\in[n]$ with $y_i >0$.

        \item For all $j\in[n]$ with $y_j <0$, we verify that $\alpha_j^{(1)}$ satisfies the required upper and lower bounds. For the upper bound, recall that
        \begin{align*}
            \alpha_j^{(1)} &= \eta\pts{y_j-\epsilon_j+\frac{1}{\eta}\e_j^\top\XXi\bepsilon}\\
            &= \eta\pts{y_j-\epsilon_j+\frac{1}{\eta}\e_j^\top\mts{\frac{1}{\nnorm[1]{\blambda}}\I + \pts{\XXi - \frac{1}{\nnorm[1]{\blambda}}\I}}\bepsilon}\\
            &= \eta\pts{y_j-\epsilon_j+ \frac{\epsilon_j}{\eta\nnorm[1]{\blambda}} + \frac{1}{\eta}\e_j^\top \pts{\XXi - \frac{1}{\nnorm[1]{\blambda}}\I}\bepsilon}\\
            &\stackrel{(\mathrm{i})}{\leq} \eta\pts{y_j + \frac{\epsilon_j}{\eta\nnorm[1]{\blambda}} + \frac{1}{\eta}\e_j^\top \pts{\XXi - \frac{1}{\nnorm[1]{\blambda}}\I}\bepsilon}\\
            &\stackrel{(\mathrm{ii})}{\leq} \eta\pts{y_j + \frac{\epsilon_j}{\eta\nnorm[1]{\blambda}} + \frac{1}{\eta} \nnorm[2]{\XXi - \frac{1}{\nnorm[1]{\blambda}}\I}\nnorm[2]{\bepsilon}},
        \end{align*}
        where inequality (i) drops the negative term $-\epsilon_j$, and inequality (ii) follows from the submultiplicativity of the operator norm. By Corollary~\ref{cor:concentration_op_norm}, we have
        \begin{align*}
            \nnorm[2]{\XXi - \frac{1}{\nnorm[1]{\blambda}}\I} \leq \frac{C_gC}{\nnorm[1]{\blambda}} \cdot \max \pts{\sqrt{\frac{n}{d_2}}, \frac{n}{d_{\infty}}},
        \end{align*}
        with probability at least $1 - 2 \exp(-n(Cc - \ln 9))$. Moreover, by the theorem assumptions, $\nnorm[2]{\bepsilon} \leq \frac{\sqrt{n}}{C_{\alpha}}y_{\min}$ and $\frac{1}{\eta} \leq CC_g\nnorm[1]{\blambda}$. Combining these bounds yields
        \begin{align*}
            \alpha_j^{(1)} &\leq \frac{1}{CC_g\nnorm[1]{\blambda}}\pts{-y_{\min} + \frac{CC_g}{C_{\alpha}}y_{\min} + C^2C_g^2\cdot \max \pts{\sqrt{\frac{n}{d_2}}, \frac{n}{d_{\infty}}} \cdot \frac{\sqrt{n}}{C_{\alpha}}y_{\min}}\\
            &\leq \frac{1}{CC_g\nnorm[1]{\blambda}}\pts{-y_{\min} + \frac{CC_g}{C_{\alpha}}y_{\min} + C^2C_g^2\cdot \frac{y_{\min}}{C_0y_{\max}} \cdot \frac{1}{C_{\alpha}}y_{\min}}\\
            & = -\frac{y_{\min}}{C_{\alpha}\nnorm[1]{\blambda}}\pts{\frac{C_{\alpha}}{CC_g} -1 - \frac{CC_gy_{\min}}{C_0y_{\max}}}\\
            &\leq -\frac{y_{\min}}{C_{\alpha}\nnorm[1]{\blambda}}.
        \end{align*}
        The second inequality follows from $d_2 \geq C_0^2\frac{n^2 y_{\max}^2}{y_{\min}^2}$ and $d_{\infty} \geq C_0\frac{n^{1.5} y_{\max}}{y_{\min}}$ in Assumption~\ref{asm:high_dim_data}, and the last inequality uses the following relationships between constants: $C_0 > C \cdot C_{\alpha}^2$ and $C_{\alpha} > C \cdot \max\{C_g^2, C_yC_g\}$. For the lower bound, we have
        \begin{align*}
            \alpha_j^{(1)} &= \eta\pts{y_j-\epsilon_j+ \frac{\epsilon_j}{\eta\nnorm[1]{\blambda}} + \frac{1}{\eta}\e_j^\top \pts{\XXi - \frac{1}{\nnorm[1]{\blambda}}\I}\bepsilon}\\
            &\geq \eta\pts{-y_{\max} -\epsilon_j - \frac{1}{\eta}\nnorm[2]{\XXi - \frac{1}{\nnorm[1]{\blambda}}\I}\nnorm[2]{\bepsilon}}\\
            &\geq \frac{1}{C_g\nnorm[1]{\blambda}}\pts{-y_{\max} - \frac{1}{C_{\alpha}}y_{\min} - C^2C_g^2\cdot \max \pts{\sqrt{\frac{n}{d_2}}, \frac{n}{d_{\infty}}}\cdot \frac{\sqrt{n}}{C_{\alpha}}y_{\min}}\\
            &\geq \frac{1}{C_g\nnorm[1]{\blambda}}\pts{-y_{\max} - \frac{1}{C_{\alpha}}y_{\min} - C^2C_g^2\cdot \frac{y_{\min}}{C_0 y_{\max}}\cdot \frac{1}{C_{\alpha}}y_{\min}}\\
            &\geq \frac{-3y_{\max}}{C_g\nnorm[1]{\blambda}},
        \end{align*}
        by the same arguments.
        Thus, $\alpha_j^{(1)}$ satisfies both the required upper and lower bounds for all $j$ with $y_j < 0$.

        \item We now verify that the primal variables corresponding to positively labeled examples minus $\y_+$ satisfy the norm bound in Lemma~\ref{lem:primal_label_same_sign}. Specifically, we show that\\ $\nnorm[2]{\bbeta_+^{(1)} -\y_+}^2 \leq C_y^2\nnorm[2]{\y_+}^2$. According to Equation~\eqref{eq:single_beta_one}, we have
        \begin{align}
            \nnorm[2]{\bbeta_+^{(1)} -\y_+}^2 &= \sum_{i: y_i>0} \pts{\beta_i^{(1)} - y_i}^2\nonumber\\
            &= \sum_{i: y_i>0} \pts{\underbrace{\epsilon_i -\eta\e_i^\top\XX\pts{\bepsilon -\y} - y_i}_{\eqqcolon T_i}}^2.\label{eq:single_pos_diff}
        \end{align}
        Next, we bound the term $T_i\coloneqq\epsilon_i -\eta\e_i^\top\XX\pts{\bepsilon -\y} - y_i$ for all $i\in[n]$ with $y_i > 0$. We have
        \begin{align*}
            T_i&=\epsilon_i -\eta\e_i^\top\XX\pts{\bepsilon -\y} - y_i\\
            &= (\epsilon_i - y_i) -\eta\e_i^\top\mts{\nnorm[1]{\blambda}\I + \pts{\XX - \nnorm[1]{\blambda}\I}}\pts{\bepsilon -\y}\\
            &= (1 - \eta\nnorm[1]{\blambda})(\epsilon_i - y_i) - \eta\e_i^\top\pts{\XX - \nnorm[1]{\blambda}\I}\pts{\bepsilon -\y}.
        \end{align*}
        Since the step size assumption guarantees that $\frac{1}{CC_g\nnorm[1]{\blambda}} \leq \eta \leq \frac{1}{C_g\nnorm[1]{\blambda}}$, and $\epsilon_i \leq \frac{1}{C_{\alpha}}y_{\min}$, the term $(1 - \eta\nnorm[1]{\blambda})(\epsilon_i - y_i)$ is strictly negative. Hence, in order to upper bound $T_i^2$, it suffices to find the lower bound for $T_i$. We have
        \begin{align*}
            T_i &= (1 - \eta\nnorm[1]{\blambda})(\epsilon_i - y_i) - \eta\e_i^\top\pts{\XX - \nnorm[1]{\blambda}\I}\pts{\bepsilon -\y}\\
            &\geq -y_i - \eta \nnorm[2]{\XX - \nnorm[1]{\blambda}\I}\nnorm[2]{\bepsilon -\y},
        \end{align*}
        where the inequality drops the positive terms $(1 - \eta\nnorm[1]{\blambda})\epsilon_i$ and $\eta\nnorm[1]{\blambda}y_i$. We again upper bound $\nnorm[2]{\XX - \nnorm[1]{\blambda}\I}$ by Corollary~\ref{cor:concentration_op_norm}. With probability at least\\ $1 - 2 \exp(-n(Cc - \ln 9))$, we have
        \begin{align*}
            T_i &\geq -y_i - \eta \cdot C\nnorm[1]{\blambda}\cdot \max \pts{\sqrt{\frac{n}{d_2}}, \frac{n}{d_{\infty}}}\nnorm[2]{\bepsilon -\y}\\
            &\geq -y_i - \frac{C}{C_g} \cdot \max \pts{\sqrt{\frac{n}{d_2}}, \frac{n}{d_{\infty}}}\nnorm[2]{\bepsilon -\y},
        \end{align*}
        by applying $\eta \leq \frac{1}{C_g\nnorm[1]{\blambda}}$. Finally, we apply the upper bounds for $\nnorm[2]{\bepsilon}$ and $\nnorm[2]{\y}$, and Assumption~\ref{asm:high_dim_data} ensures that $d_2 \geq C_0^2\frac{n^2 y_{\max}^2}{y_{\min}^2}$ and $d_{\infty} \geq C_0\frac{n^{1.5} y_{\max}}{y_{\min}}$. We have
        \begin{align*}
             T_i &\geq -y_i - \frac{C}{C_g} \max \pts{\sqrt{\frac{n}{d_2}}, \frac{n}{d_{\infty}}}\pts{\frac{\sqrt{n}}{C_{\alpha}}y_{\min} + \sqrt{n}y_{\max}}\\
             &\geq -y_i - \frac{Cy_{\min}}{C_gC_0y_{\max}}\pts{\frac{1}{C_{\alpha}}y_{\min} + y_{\max}}\\
             &\geq -y_i \pts{1 + \frac{2C}{C_gC_0}}\\
             &\geq -C_y y_i,
        \end{align*}
        with the choice of $C_y \geq 2$. Substituting $T_i^2 \leq C_y^2 y_i^2$ into Equation~\eqref{eq:single_pos_diff}, we have
        \begin{align*}
            \nnorm[2]{\bbeta_+^{(1)} -\y_+}^2 &\leq \sum_{i:y_i>0} C_y^2 y_i^2 = C_y^2 \nnorm[2]{\y_+}^2.
        \end{align*}
        As a result, we conclude that $ \nnorm[2]{\bbeta_+^{(1)} -\y_+} \leq C_y \nnorm[2]{\y_+}$ as required.

        \item We next verify that $\balpha^{(1)}$ satisfies the required norm bound. Recall that
        \begin{align*}
            \balpha^{(1)} &=\eta\pts{\y-\bepsilon+\frac{1}{\eta}\XXi\bepsilon}.
        \end{align*}
        Taking the $\ell_2$ norm and applying the triangle inequality yields
        \begin{align*}
            \nnorm[2]{\balpha^{(1)}} &=\nnorm[2]{\eta\pts{\y-\bepsilon+\frac{1}{\eta}\XXi\bepsilon}}\\
            &\leq \eta \mts{\nnorm[2]{\y} + \nnorm[2]{\bepsilon} + \frac{1}{\eta}\nnorm[2]{\XXi}\nnorm[2]{\bepsilon}}.
        \end{align*}
        We now bound each term on the right-hand side. We apply the label bound, $\nnorm[2]{\y}\leq \sqrt{n}y_{\max}$ and the construction of the initialization, $\nnorm[2]{\bepsilon}\leq \frac{\sqrt{n}}{C_{\alpha}}y_{\min}$. Moreover, Lemma~\ref{lem:concentration_eigenvalues} implies $\nnorm[2]{(\XX)^{-1}} \leq \frac{C_g}{\nnorm[1]{\blambda}}$ with probability at least $1-2e^{-n/C_g}$, and the step size condition ensures $\frac{1}{CC_g\nnorm[1]{\blambda}} \leq \eta \leq \frac{1}{C_g\nnorm[1]{\blambda}}$. Substituting these bounds, we obtain
        \begin{align*}
            \nnorm[2]{\balpha^{(1)}} &\leq \frac{1}{C_g\nnorm[1]{\blambda}}\mts{\sqrt{n}y_{\max} + \frac{\sqrt{n}}{C_{\alpha}}y_{\min} + CC_g\nnorm[1]{\blambda}\cdot\frac{C_g}{\nnorm[1]{\blambda}}\cdot \frac{\sqrt{n}}{C_{\alpha}}y_{\min}}\\
            &\leq \frac{1}{C_g\nnorm[1]{\blambda}}\pts{3\sqrt{n}y_{\max}}\\
            &\leq \frac{C_{\alpha}\sqrt{n}y_{\max}}{\nnorm[1]{\blambda}},
        \end{align*}
        with $C_{\alpha}\gtrsim \max\{C_g^2, C_yC_g\}$. Therefore, $\balpha^{(1)}$ satisfies the required norm bound.

        \item Since we have shown that $\alpha_j^{(1)} \leq -\frac{y_{\min}}{C_{\alpha}\nnorm[1]{\blambda}}$ and $\nnorm[2]{\balpha^{(1)}} \leq \frac{C_{\alpha}\sqrt{n}y_{\max}}{\nnorm[1]{\blambda}}$ for all $j\in[n]$ with $y_j < 0$, it follows from Lemma~\ref{lem:neg_stay_neg} that $\beta_j^{(1)} \leq 0$ for all $j\in[n]$ with $y_j < 0$.

        \item Since we have shown that $\beta_i^{(1)} > 0$ for all $i\in[n]$ with $y_i > 0$ and $\beta_j^{(1)} \leq 0$ for all $j\in[n]$ with $y_j < 0$, the signs of the primal variables coincide with the signs of the labels. Consequently, $\sigma(\bbeta^{(1)}) = \begin{bmatrix} \bbeta_+^{(1)} \\ \zero \end{bmatrix}$.
    \end{enumerate}
    We have shown that at iteration $t=1$, all conditions in Lemma~\ref{lem:primal_label_same_sign} are satisfied. Consequently, all positively labeled examples are active, while all negatively labeled examples are inactive. We now complete the proof by induction and characterize the gradient descent dynamics for all subsequent iterations.
    By Lemma~\ref{lem:primal_label_same_sign}, since the conditions hold at $t=1$, they also hold at $t=2$. More generally, the same lemma implies that if the conditions hold at $t=\tilde{t}$ then they continue to hold at $t=\tilde{t}+1$. This completes the induction argument.
    
    As a result, for all $t\geq 1$, the activation pattern remains fixed, i.e., $\D(\bbeta^{(t)}) = \D(\bbeta^{(1)})$, and all negative labeled examples are inactive,~i.e., $\X_-\w^{(t)}\preceq 0$. By Lemma~\ref{lem:final_phase}, the gradient descent dynamics from this point onward are equivalent to those of linear regression trained on the positively labeled examples, with initialization  $\w^{(1)}$. Finally, by Lemma~\ref{lem:conv_step_size}, the $\w^{(\infty)}$ satisfies
    \begin{align*}
        \w^{(\infty)} =\uargmin{\w\in\{\w:\X_+\w =\y_+\}} \nnorm[2]{\w - \w^{(1)}},
    \end{align*}
    where we have $\w^{(1)} = \eta\X^\top\pts{\y-\bepsilon+\frac{1}{\eta}(\XX)^{-1}\bepsilon}$.
    This completes the proof of Theorem~\ref{thm:single_relu_gd_high_dim_implicit_bias}.
\end{proof}

\subsection{Proof of Theorem~\ref{thm:single_relu_approx_to_w_star} (Implicit Bias Approximation to \texorpdfstring{$\w^\star$}{w*})}\label{app:proof_single_relu_approx_to_w_star}
In this section, we present the proof of implicit bias approximation to $\w^\star$ for single ReLU models.

\begin{proof}(Theorem~\ref{thm:single_relu_approx_to_w_star})
    We restate the definition of $\w^\star$ in Equation~\eqref{eq:single_relu_minimum_norm_sol} below.
    
    \begin{align*}
        \w^\star = {}&\uargmin{\w}\frac{1}{2}\nnorm[2]{\w}^2\\
        \text{s.t. } \w^\top\x_i &= y_i, \text{ for all } y_i > 0\nonumber\\
        \w^\top\x_j &\leq 0, \text{ for all } y_j \leq 0\nonumber.
    \end{align*}
    Recall that the gradient descent limit $\w^{(\infty)}$ satisfies the same set of constraints: it interpolates all positively labeled examples and produces negative predictions for negatively labeled examples. Consequently, both $\w^{(\infty)}$ and $\w^\star$ are feasible solutions achieving the minimum empirical risk.

    We start with showing the upper bound on $\|\w^{(\infty)}-\w^\star\|_2$. We first relate the distance between the predictors $\w^{(\infty)}$ and $\w^\star$ to the distance in their predictions, i.e., $\|\X\w^{(\infty)} - \X\w^\star\|_2$. Since both vectors lie in the span of the data $\{\x_i\}_{i=1}^n$, their difference has no component in the null space corresponding to the smallest $d-n$ eigenvalues of $\X^\top\X$. Therefore, we have
    \begin{align}
        \nnorm[2]{\X\w^{(\infty)} - \X\w^\star}^2 = \nnorm[2]{\X\pts{\w^{(\infty)} -\w^\star}}^2 &\geq \mu_{n}(\X^\top\X) \nnorm[2]{\w^{(\infty)} -\w^\star}^2\nonumber\\
        &= \mu_{n}(\XX) \nnorm[2]{\w^{(\infty)} -\w^\star}^2.\label{eq:w_star_upper}
    \end{align}
    As a result, to derive an upper bound for $\|\w^{(\infty)} -\w^\star\|_2$, it suffices to upper bound the distance between their prediction $\|\X\w^{(\infty)} - \X\w^\star\|_2$. We begin with analyzing $\w^{(\infty)}$. By Theorem~\ref{thm:single_relu_gd_high_dim_implicit_bias}, $\w^{(\infty)}$ satisfies the following:
    \begin{align*}
        \w^{(\infty)\top}\x_i &= y_i \myquad[17] \text{ for all } y_i > 0,\\
        \alpha_j^{(\infty)} &= \alpha_j^{(1)} = \eta \pts{y_j - \epsilon_j + \frac{1}{\eta}\e_j^\top\XXi\bepsilon} \quad \text{ for all } y_j < 0,
    \end{align*}
    and also all the conditions in Lemma~\ref{lem:primal_label_same_sign}. On the other hand, according to the necessary conditions in Equation~\eqref{eq:kkt_nec_cond} in Lemma~\ref{lem:relu_mni_subset}, $\w^\star$ satisfies
    \begin{align*}
        \w^\star = -\sum_{i\in S_1}\lambda^\star_i\x_i &- \sum_{j\in \tilde{S}_2}\mu^\star_j\x_j, \text{ with } \lambda^\star_i \in \R \text{ and } \mu^\star_j > 0,\\
        \w^{\star\top}\x_i &= y_i, \text{ for all } i\in S_1,\\
        \w^{\star\top}\x_j &= 0, \text{ for all } j\in \tilde{S}_2,
    \end{align*}
    where we have denoted $S_1 = \{i: y_i > 0, \text{ for all }i\in[n]\}$, $S_2 = \{j: y_j \leq 0, \text{ for all }j\in[n]\}$, $\tilde{S}_2 \subseteq S_2$ (note that $\tilde{S}_2$ can be empty) and $S=S_1\cup \tilde{S_2}$. Based on these necessary conditions, we can define $\w^\star = \X^\top\balpha^\star$ where
    \begin{align*}
        \alpha_i^\star = \condds{-\lambda_i^\star}{\quad \text{ for all } i\in S_1}{-\mu_i^\star}{\quad \text{ for all } i \in \tilde{S}_2}{0}{\quad \text{ for all } i \in S_2\cup \tilde{S}_2^{\mathsf{c}}\eqqcolon S_3 }.
    \end{align*}
    Let $\X_S \in \R^{|S| \times d}$ denote the submatrix of $\X$ consisting of the rows indexed by $S$ (taken in increasing order), and let $\y_S \in \R^{|S|}$ denote the corresponding label subvector with all negative entries replaced by zero. We have
    \begin{align*}
        \y_S &= \X_S\X_S^\top\balpha_S^\star,
    \end{align*}
    and similarly, by taking the norm and using the matrix norm lower bound of the smallest eigenvalue of $\X_S\X_S^\top$, we have
    \begin{align*}
        \nnorm[2]{\y_S} &= \nnorm[2]{\X_S\X_S^\top\balpha_S^\star}\\
        &\geq \mu_{|S|}(\X_S\X_S^\top)\nnorm[2]{\balpha_S^\star}.
    \end{align*}
    Consequently, we have
    \begin{align}
        \nnorm[2]{\balpha^\star} = \nnorm[2]{\balpha_S^\star} \leq \frac{\nnorm[2]{\y_S}}{\mu_{|S|}(\X_S\X_S^\top)} \leq \frac{\sqrt{n}y_{\max}}{\mu_{n}(\XX)} \leq \frac{C_g\sqrt{n}y_{\max}}{\nnorm[1]{\blambda}},\label{eq:alpha_star_upper}
    \end{align}
    where the second inequality follows from the variational formulation of submatrix, and the last inequality follows from Lemma~\ref{lem:concentration_eigenvalues} with probability at least $1-2e^{-n/C_g}$.
    
    We know that for all $i \in S_1$, $\w^{(\infty)\top}\x_i = \w^{\star\top}\x_i = y_i$, and $\w^{\star\top}\x_j = 0$ for all $j\in\tilde{S}_2$. Therefore, we can write
    \begin{align}
        \nnorm[2]{\X\w^{(\infty)} - \X\w^\star}^2 &= \sum_{i=1}^n \pts{\w^{(\infty)\top}\x_i - \w^{\star\top}\x_i}^2\nonumber\\
        &= \sum_{i\in \tilde{S}_2} \pts{\w^{(\infty)\top}\x_i}^2 + \sum_{i\in S_3} \pts{\w^{(\infty)\top}\x_i - \w^{\star\top}\x_i}^2.\label{eq:difference_in_prediction}
    \end{align}
    We start with upper bounding the term $(\w^{(\infty)\top}\x_i)^2$ for all $i\in \tilde{S}_2$. Since $\w^{(\infty)\top}\x_i \leq 0$ by the conditions in Lemma~\ref{lem:primal_label_same_sign}, it suffices to lower bound $\w^{(\infty)\top}\x_i$.  We have
    \begin{align*}
        \w^{(\infty)\top}\x_i &= \e_i^\top\XX\balpha^{(\infty)}\\
        &= \e_i^\top\mts{\nnorm[1]{\blambda}\I + (\XX - \nnorm[1]{\blambda}\I)}\balpha^{(\infty)}\\
        &= \nnorm[1]{\blambda}\alpha_i^{(\infty)} + \e_i^\top(\XX - \nnorm[1]{\blambda}\I)\balpha^{(\infty)}\\
        &\geq \nnorm[1]{\blambda}\alpha_i^{(\infty)} - \nnorm[2]{\XX - \nnorm[1]{\blambda}\I}\nnorm[2]{\balpha^{(\infty)}}\\
        &\geq \nnorm[1]{\blambda}\mts{\alpha_i^{(\infty)} - C\cdot \max \pts{\sqrt{\frac{n}{d_2}}, \frac{n}{d_{\infty}}}\nnorm[2]{\balpha^{(\infty)}}},
    \end{align*}
    where the last inequality applies Corollary~\ref{cor:concentration_op_norm}. Substituting the bounds of $\alpha_i^{(\infty)}$ and $\nnorm[2]{\balpha^{(\infty)}}$ from Lemma~\ref{lem:primal_label_same_sign}, we have 
    \begin{align*}
        \w^{(\infty)\top}\x_i &\geq \nnorm[1]{\blambda}\mts{-\frac{3y_{\max}}{C_g\nnorm[1]{\blambda}} - C\cdot \max \pts{\sqrt{\frac{n}{d_2}}, \frac{n}{d_{\infty}}}\frac{C_{\alpha}\sqrt{n}y_{\max}}{\nnorm[1]{\blambda}}}\\
        &\geq \nnorm[1]{\blambda}\mts{-\frac{3y_{\max}}{C_g\nnorm[1]{\blambda}} - C\cdot \frac{y_{\min}}{C_0 y_{\max}}\frac{C_{\alpha}y_{\max}}{\nnorm[1]{\blambda}}}\\
        &\geq -\frac{4}{C_g}y_{\max},
    \end{align*}
    where the inequalities above substitute $d_2 \geq C_0^2\frac{n^2 y_{\max}^2}{y_{\min}^2}$ and $d_{\infty} \geq C_0\frac{n^{1.5} y_{\max}}{y_{\min}}$ in Assumption~\ref{asm:high_dim_data} with $C_0\gtrsim C_{\alpha}^2$ and $C_{\alpha}\gtrsim \max\{C_g^2, C_yC_g\}$. Therefore, we have $(\w^{(\infty)\top}\x_i)^2 \leq \frac{16}{C_g^2} y_{\max}^2$ for all $i\in \tilde{S}_2$. Next, we upper bound the term $(\w^{(\infty)\top}\x_i - \w^{\star\top}\x_i)^2$ for all $i\in S_3$. We use the key idea that $\alpha_i^\star = 0$ for all $i\in S_3$. We have
    \begin{align*}
        \w^{(\infty)\top}\x_i - \w^{\star\top}\x_i &= \e_i^\top\XX\balpha^{(\infty)} - \e_i^\top\XX\balpha^\star\\
        &= \e_i^\top\mts{\nnorm[1]{\blambda}\I + (\XX - \nnorm[1]{\blambda}\I)}(\balpha^{(\infty)} - \balpha^\star)\\
        &= \nnorm[1]{\blambda}\alpha_i^{(\infty)} + \e_i^\top(\XX - \nnorm[1]{\blambda}\I)(\balpha^{(\infty)} - \balpha^\star)\\
        &\geq \nnorm[1]{\blambda}\alpha_i^{(\infty)} - \nnorm[2]{\XX - \nnorm[1]{\blambda}\I}\pts{\nnorm[2]{\balpha^{(\infty)}} + \nnorm[2]{\balpha^\star}}\\
        &\geq \nnorm[1]{\blambda}\mts{-\frac{3y_{\max}}{C_g\nnorm[1]{\blambda}} - C\cdot \max \pts{\sqrt{\frac{n}{d_2}}, \frac{n}{d_{\infty}}}\pts{\frac{C_{\alpha}\sqrt{n}y_{\max}}{\nnorm[1]{\blambda}} + \frac{C_g\sqrt{n}y_{\max}}{\nnorm[1]{\blambda}}}}\\
        &\geq \nnorm[1]{\blambda}\mts{-\frac{3y_{\max}}{C_g\nnorm[1]{\blambda}} - C\cdot \frac{y_{\min}}{C_0 y_{\max}}\pts{\frac{C_{\alpha}y_{\max}}{\nnorm[1]{\blambda}} + \frac{C_gy_{\max}}{\nnorm[1]{\blambda}}}}\\
        &\geq -\frac{4}{C_g}y_{\max},
    \end{align*}
    by applying the same argument and noting from Equation~\eqref{eq:alpha_star_upper} that $\nnorm[2]{\balpha^\star} \leq \frac{C_g\sqrt{n}y_{\max}}{\nnorm[1]{\blambda}}$. Substituting the upper bounds into Equation~\eqref{eq:difference_in_prediction} gives us
    \begin{align}
        \nnorm[2]{\X\w^{(\infty)} - \X\w^\star}^2 &= \sum_{i\in \tilde{S}_2} \pts{\w^{(\infty)\top}\x_i}^2 + \sum_{i\in S_3} \pts{\w^{(\infty)\top}\x_i - \w^{\star\top}\x_i}^2\nonumber\\
        &\leq \sum_{i\in \tilde{S}_2} \frac{16}{C_g^2}y_{\max}^2 + \sum_{i\in S_3} \frac{16}{C_g^2}y_{\max}^2\nonumber\\
        &= \frac{16}{C_g^2} n_- y_{\max}^2\label{eq:prediction_upper}.
    \end{align}
    Finally, putting together Equation~\eqref{eq:w_star_upper} and~\eqref{eq:prediction_upper}, we have
    \begin{align*}
        \nnorm[2]{\w^{(\infty)} -\w^\star}^2 \leq \frac{\nnorm[2]{\X\w^{(\infty)} - \X\w^\star}^2}{\mu_{n}(\XX)} \leq \frac{16n_-y_{\max}^2}{C_g\nnorm[1]{\blambda}},
    \end{align*}
    which completes the proof of the upper bound. Next, we derive the lower bound of $\|\w^{(\infty)} - \w^\star\|_2$ in a similar approach. We again start with the prediction distance, given by
    \begin{align}
        \nnorm[2]{\X\w^{(\infty)} - \X\w^\star}^2 = \nnorm[2]{\X\pts{\w^{(\infty)} -\w^\star}}^2 &\leq \mu_{1}(\X^\top\X) \nnorm[2]{\w^{(\infty)} -\w^\star}^2\nonumber\\
        &=\mu_{1}(\XX) \nnorm[2]{\w^{(\infty)} -\w^\star}^2.\label{eq:w_star_lower}
    \end{align}
    It suffices to lower bound $\|\X\w^{(\infty)} - \X\w^\star\|_2$ to get the lower bound of $\|\w^{(\infty)} - \w^\star\|_2$. By Equation~\eqref{eq:difference_in_prediction}, we have 
    \begin{align*}
        \nnorm[2]{\X\w^{(\infty)} - \X\w^\star}^2 &= \sum_{i\in \tilde{S}_2} \pts{\w^{(\infty)\top}\x_i}^2 + \sum_{i\in S_3} \pts{\w^{(\infty)\top}\x_i - \w^{\star\top}\x_i}^2.
    \end{align*}
    Therefore, we need to lower bound $(\w^{(\infty)\top}\x_i)^2$ for $i\in \tilde{S}_2$, and $(\w^{(\infty)\top}\x_i - \w^{\star\top}\x_i)^2$ for $i\in S_3$. For $\w^{(\infty)\top}\x_i$, since $\w^{(\infty)\top}\x_i < 0$, we have
    \begin{align*}
        \w^{(\infty)\top}\x_i &= \e_i^\top\XX\balpha^{(\infty)}\\
        &= \e_i^\top\mts{\nnorm[1]{\blambda}\I + (\XX - \nnorm[1]{\blambda}\I)}\balpha^{(\infty)}\\
        &= \nnorm[1]{\blambda}\alpha_i^{(\infty)} + \e_i^\top(\XX - \nnorm[1]{\blambda}\I)\balpha^{(\infty)}\\
        &\leq \nnorm[1]{\blambda}\alpha_i^{(\infty)} + \nnorm[2]{\XX - \nnorm[1]{\blambda}\I}\nnorm[2]{\balpha^{(\infty)}}\\
        &\leq \nnorm[1]{\blambda}\mts{\alpha_i^{(\infty)} + C\cdot \max \pts{\sqrt{\frac{n}{d_2}}, \frac{n}{d_{\infty}}}\nnorm[2]{\balpha^{(\infty)}}},
    \end{align*}
    where the last inequality applies Corollary~\ref{cor:concentration_op_norm}. Substituting the bounds of $\alpha_i^{(\infty)}$ and $\nnorm[2]{\balpha^{(\infty)}}$ from Lemma~\ref{lem:primal_label_same_sign}, we have
    \begin{align*}
        \w^{(\infty)\top}\x_i &\leq \nnorm[1]{\blambda}\mts{-\frac{y_{\min}}{C_{\alpha}\nnorm[1]{\blambda}} + C\cdot \max \pts{\sqrt{\frac{n}{d_2}}, \frac{n}{d_{\infty}}}\frac{C_{\alpha}\sqrt{n}y_{\max}}{\nnorm[1]{\blambda}}}\\
        &\leq \nnorm[1]{\blambda}\mts{-\frac{y_{\min}}{C_{\alpha}\nnorm[1]{\blambda}} + C\cdot \frac{y_{\min}}{C_0 y_{\max}}\frac{C_{\alpha}y_{\max}}{\nnorm[1]{\blambda}}}\\
        &\leq -(1-\frac{C\cdot C_{\alpha}^2}{C_0})\frac{y_{\min}}{C_{\alpha}},
    \end{align*}
    where the inequalities above substitute $d_2 \geq C_0^2\frac{n^2 y_{\max}^2}{y_{\min}^2}$ and $d_{\infty} \geq C_0\frac{n^{1.5} y_{\max}}{y_{\min}}$ in Assumption~\ref{asm:high_dim_data} with $C_0\gtrsim C_{\alpha}^2$. Similarly, for $\w^{(\infty)\top}\x_i - \w^{\star\top}\x_i$, we have
    \begin{align*}
        \w^{(\infty)\top}\x_i - \w^{\star\top}\x_i &= \e_i^\top\XX\balpha^{(\infty)} - \e_i^\top\XX\balpha^\star\\
        &= \e_i^\top\mts{\nnorm[1]{\blambda}\I + (\XX - \nnorm[1]{\blambda}\I)}(\balpha^{(\infty)} - \balpha^\star)\\
        &= \nnorm[1]{\blambda}\alpha_i^{(\infty)} + \e_i^\top(\XX - \nnorm[1]{\blambda}\I)(\balpha^{(\infty)} - \balpha^\star)\\
        &\leq \nnorm[1]{\blambda}\alpha_i^{(\infty)} + \nnorm[2]{\XX - \nnorm[1]{\blambda}\I}(\nnorm[2]{\balpha^{(\infty)}} + \nnorm[2]{\balpha^\star})\\
        &\leq \nnorm[1]{\blambda}\mts{-\frac{y_{\min}}{C_{\alpha}\nnorm[1]{\blambda}} + C\cdot \max \pts{\sqrt{\frac{n}{d_2}}, \frac{n}{d_{\infty}}}\pts{\frac{C_{\alpha}\sqrt{n}y_{\max}}{\nnorm[1]{\blambda}} + \frac{C_g\sqrt{n}y_{\max}}{\nnorm[1]{\blambda}}}}\\
        &\leq \nnorm[1]{\blambda}\mts{-\frac{y_{\min}}{C_{\alpha}\nnorm[1]{\blambda}} + C\cdot \frac{y_{\min}}{C_0 y_{\max}}\pts{\frac{C_{\alpha}y_{\max}}{\nnorm[1]{\blambda}} + \frac{C_gy_{\max}}{\nnorm[1]{\blambda}}}}\\
        &\leq -(1-\frac{2C\cdot C_{\alpha}^2}{C_0})\frac{y_{\min}}{C_{\alpha}},
    \end{align*}
    by applying the same argument and noting from Equation~\eqref{eq:alpha_star_upper} that $\nnorm[2]{\balpha^\star} \leq \frac{C_g\sqrt{n}y_{\max}}{\nnorm[1]{\blambda}}$. Substituting the lower bounds into Equation~\eqref{eq:difference_in_prediction} gives us
    \begin{align}
        \nnorm[2]{\X\w^{(\infty)} - \X\w^\star}^2 &= \sum_{i\in \tilde{S}_2} \pts{\w^{(\infty)\top}\x_i}^2 + \sum_{i\in S_3} \pts{\w^{(\infty)\top}\x_i - \w^{\star\top}\x_i}^2\nonumber\\
        &\geq \sum_{i\in \tilde{S}_2} (1-\frac{C\cdot C_{\alpha}^2}{C_0})^2\frac{y_{\min}^2}{C_{\alpha}^2} + \sum_{i\in S_3} (1-\frac{2C\cdot C_{\alpha}^2}{C_0})^2\frac{y_{\min}^2}{C_{\alpha}^2}\nonumber\\
        &\geq \pts{1 -\frac{2C\cdot C_{\alpha}^2}{C_0}}^2 \frac{n_- y_{\min}^2}{C_{\alpha}^2}\nonumber\\
        &= \frac{n_- y_{\min}^2}{\tilde{C}},\label{eq:prediction_lower}
    \end{align}
    where we let $\tilde{C} \coloneqq \frac{C_0^2C_{\alpha}^2}{\pts{C_0 -2C\cdot C_{\alpha}^2}^2} > 1$. Finally, putting together Equations~\eqref{eq:w_star_lower} and~\eqref{eq:prediction_lower}, we have
    \begin{align*}
        \nnorm[2]{\w^{(\infty)} -\w^\star}^2 \geq \frac{\nnorm[2]{\X\w^{(\infty)} - \X\w^\star}^2}{\mu_{1}(\XX)} \geq\frac{n_-y_{\min}^2}{\tilde{C}C_g\nnorm[1]{\blambda}}.
    \end{align*}
    This completes the proof of the lower bound.
\end{proof}

{
\subsection{Local Minimum Convergence Under a Not-All-Positive Initialization}\label{app:counter_example}
In this section, we present a simple and explicit counterexample showing that a single ReLU model initialized with a not-all-positive initialization can converge to a local minimum that is not global. 

\begin{lemma}[Local minimum convergence for not-all-positive initialization]\label{lem:counter_example}
Assume Assumptions~\ref{asm:label_bounds} and~\ref{asm:high_dim_data} hold. For a single ReLU model, a dataset $n = 2$ and $y_i > 0$ for all $i=1,2$, we choose the initialization $\w^{(0)} = \XXX \bepsilon$, where $\epsilon_1 = -\delta$ and $\epsilon_2 = \delta$, for some $0 < \delta \le \frac{1}{C} y_{\min}$. If $\x_1^\top\x_2 < 0$ and $|\x_1^\top\x_2| < \nnorm[2]{\x_2}^2$, with a step size $\eta < \frac{1}{\mu_1\pts{\XX}}$, gradient descent converges to a local minimum.
\end{lemma}
\begin{proof}
To show that gradient descent converges to a local minimum, it suffices to prove that $\beta_1^{(t)} < 0$ for all $t \ge 0$: since $y_1 > 0$, this condition implies that the first training example is not perfectly fit.
%
Hence, we prove by induction that if $\beta_1^{(t)} < 0$ for all $t \le \tilde{t}$, then $\beta_1^{(\tilde{t}+1)} < 0$. 
Note that the base case, i.e. $\beta_1^{(0)} < 0$, follows from the initialization condition $\beta_1^{(0)} = \epsilon_1 = - \delta < 0$.
Therefore, we only need to show the inductive step.
From the primal gradient update in Equation~\eqref{eq:primal_update}, we obtain the gradient update of $\beta_1^{(\tilde{t}+1)}$ as
\begin{align}
\beta_1^{(\tilde{t}+1)}
&= \beta_1^{(\tilde{t})} - \eta \nnorm[2]{\x_1}^2 \cdot \underbrace{\mathbbm{1}_{\beta_1^{(\tilde{t})} > 0}}_{=0} \cdot \bigl(\beta_1^{(\tilde{t})} - y_1\bigr)- \eta \x_1^\top \x_2 \cdot \mathbbm{1}_{\beta_2^{(\tilde{t})} > 0} \cdot \bigl(\beta_2^{(\tilde{t})} - y_2\bigr) \nonumber \\
&= \beta_1^{(0)} - \eta \x_1^\top \x_2 \sum_{t=0}^{\tilde{t}} \mathbbm{1}_{\beta_2^{(t)} > 0}\bigl(\beta_2^{(t)} - y_2\bigr). \label{eq:beta1}
\end{align}
On the other hand, for $\beta_2^{(t)}$,  since $\beta_1^{(t)} < 0$ for all $t\leq \tilde{t}$, we have
\begin{align*}
\beta_2^{(t+1)} = \beta_2^{(t)} - \eta \nnorm[2]{\x_2}^2\mathbbm{1}_{\beta_2^{(t)} > 0} \bigl(\beta_2^{(t)} - y_2\bigr) = \beta_2^{(t)}(1 - \eta \nnorm[2]{\x_2}^2\mathbbm{1}_{\beta_2^{(t)} > 0} ) + \eta \nnorm[2]{\x_2}^2y_2\mathbbm{1}_{\beta_2^{(t)} > 0},
\end{align*}
for all $t\leq \tilde{t}$. Since $\eta \nnorm[2]{\x_2}^2 < \frac{\nnorm[2]{\x_2}^2}{\mu_1\pts{\XX}} \leq 1$, we have $\beta_2^{(t+1)} > 0$ if $\beta_2^{(t)} > 0$. Since $\beta_2^{(0)} = \delta > 0$, we have shown that $\beta_2^{(t)} > 0$ for all $t\leq \tilde{t}$. Therefore, Equation~\eqref{eq:beta1} becomes
\begin{align}
    \beta_1^{(\tilde{t}+1)} = \beta_1^{(0)} - \eta \x_1^\top \x_2 \sum_{t=0}^{\tilde{t}} \bigl(\beta_2^{(t)} - y_2\bigr).\label{eq:beta15}
\end{align}


Next, we show an upper bound for $\beta_2^{(t)}$. Since $\beta_1^{(t)} < 0$ and $\beta_2^{(t)} > 0$ for all $t\leq \tilde{t}$, we can show that the gradient update for $\beta_2^{(t)}$ satisfies
\begin{align*}
\beta_2^{(t)}
&= \beta_2^{(t-1)} - \eta \nnorm[2]{\x_2}^2 \bigl(\beta_2^{(t-1)} - y_2\bigr) \\
&= \beta_2^{(t-1)} (1 - \eta \nnorm[2]{\x_2}^2) + \eta \nnorm[2]{\x_2}^2 y_2 \\
&= \beta_2^{(0)} (1 - \eta \nnorm[2]{\x_2}^2)^t 
+ \eta \nnorm[2]{\x_2}^2 y_2 \sum_{k=0}^{t-1} (1 - \eta \nnorm[2]{\x_2}^2)^k
\end{align*}
for all $t \leq \tilde{t}$.
Again, since $\eta \nnorm[2]{\x_2}^2 < 1$, the geometric series yields
\[
\sum_{k=0}^{t-1} (1 - \eta \nnorm[2]{\x_2}^2)^k 
\le \frac{1}{\eta \nnorm[2]{\x_2}^2},
\]
and therefore, we have
\begin{align}
    \beta_2^{(t)} \le \beta_2^{(0)} (1 - \eta \nnorm[2]{\x_2}^2)^t + y_2.\label{eq:counter_bound}
\end{align}
Substituting Equation~\eqref{eq:counter_bound} into Equation~\eqref{eq:beta15} and using the assumption that $\x_1^\top \x_2 < 0$, we obtain
\begin{align*}
\beta_1^{(\tilde{t}+1)}
&\le \beta_1^{(0)} - \eta \x_1^\top \x_2 \sum_{t=0}^{\tilde{t}} \beta_2^{(0)} (1 - \eta \nnorm[2]{\x_2}^2)^t \\
&\le \beta_1^{(0)} - \frac{\x_1^\top \x_2}{\nnorm[2]{\x_2}^2} \beta_2^{(0)} \\
&= -\delta + \frac{|\x_1^\top \x_2|}{\nnorm[2]{\x_2}^2} \delta < 0,
\end{align*}
where the second inequality again follows from the bound on the geometric series, and the last inequality uses the assumption that $|\x_1^\top \x_2| < \|\x_2\|_2^2$. This completes the proof of the inductive step and therefore the proof that $\beta_1^{(t)} < 0$ for all $t \geq 0$, implying that gradient descent remains stuck in a local minimum.
\end{proof}
}

According to Lemma~\ref{lem:counter_example}, a simple example is given by i.i.d. Gaussian vectors $\x_1,\x_2 \sim \mathcal{N}(\zero,\I)$ for $n=2$. By symmetry, the event $\x_1^\top \x_2 < 0$ occurs with constant probability. Moreover, in the high-dimensional regime ($d \gg n$), concentration implies that $|\x_1^\top \x_2| < \nnorm[2]{\x_2}^2$ with a high probability.

\newpage
\section{Proofs for the Two ReLU Model (\texorpdfstring{$m=2$}{m=2}) Trained with Gradient Descent}\label{app:multiple_relu_gd}

In this section, we present the proofs concerning the behavior of the $2$-ReLU model trained with gradient descent.

\subsection{Proof of Lemma~\ref{lem:multiple_relu_w_star_eqivalence} (Characterization of \texorpdfstring{$\w^\star$}{w*} )}\label{app:proof_multiple_relu_w_star_eqivalence}
\begin{proof}(Lemma~\ref{lem:multiple_relu_w_star_eqivalence})\\
    We first show that the feasible set of~\eqref{eq:multiple_relu_minimum_norm_sol} is nonempty.
    Define $\tilde{\w}_{\oplus} \coloneqq \X^\top(\XX)^{-1}\y_{\oplus}$ and $\tilde{\w}_{\ominus} \coloneqq \X^\top(\XX)^{-1}\y_{\ominus}$, where we define $y_{\oplus,i} \coloneqq \max\{y_i, 0\}$ and $y_{\ominus,i} \coloneqq -\min\{y_i, 0\}$. Then for all $i\in[n]$, we have $\sigma(\tilde{\w}_{\oplus}^\top\x_i) - \sigma(\tilde{\w}_{\ominus}^\top\x_i) = \sigma(y_{\oplus,i}) - \sigma(y_{\ominus,i}) = y_i$. Thus, $\{\tilde{\w}_{\oplus}, \tilde{\w}_{\ominus}\}$ is feasible, and the feasible set is nonempty.

    Next, we show that any optimal solution of~\eqref{eq:multiple_relu_minimum_norm_sol} corresponds to an optimal solution of~\eqref{eq:multiple_relu_w_star_eq_form}. Let $\{\w_{\oplus}^\star, \w_{\ominus}^\star\}$ be an optimal solution of~\eqref{eq:multiple_relu_minimum_norm_sol}.

    \paragraph{Case 1: $i\in S_+$ (positive labels)}~\\
    For $i\in S_+$, since $\sigma(\w_{\oplus}^{\star\top}\x_i) - \sigma(\w_{\ominus}^{\star\top}\x_i) = y_i > 0$, we have
    \begin{align*}
        \sigma(\w_{\oplus}^{\star\top}\x_i) = y_i + \sigma(\w_{\ominus}^{\star\top}\x_i) \geq y_i > 0.
    \end{align*}
    Hence, $\w_{\oplus}^{\star\top}\x_i > 0$ and $\sigma(\w_{\oplus}^{\star\top}\x_i) = \w_{\oplus}^{\star\top}\x_i$. There are two possible activation patterns:
    \begin{itemize}
        \item If $\w_{\ominus}^{\star\top}\x_i \leq 0$, then we have 
        $\sigma(\w_{\oplus}^{\star\top}\x_i) - \sigma(\w_{\ominus}^{\star\top}\x_i) = \w_{\oplus}^{\star\top}\x_i = y_i$.
        \item If $\w_{\ominus}^{\star\top}\x_i \geq 0$, then we have
        $\sigma(\w_{\oplus}^{\star\top}\x_i) - \sigma(\w_{\ominus}^{\star\top}\x_i) = \w_{\oplus}^{\star\top}\x_i - \w_{\ominus}^{\star\top}\x_i = y_i$.
    \end{itemize}
    (Note that $\w_{\ominus}^{\star\top}\x_i = 0$ is covered by both cases.)
    \paragraph{Case 2: $i\in S_-$ (negative labels)}~\\
    For $i\in S_-$, since $\sigma(\w_{\oplus}^{\star\top}\x_i) - \sigma(\w_{\ominus}^{\star\top}\x_i) = y_i < 0$, we obtain
    \begin{align*}
        \sigma(\w_{\ominus}^{\star\top}\x_i) = -y_i + \sigma(\w_{\ominus}^{\star\top}\x_i) \geq -y_i > 0,
    \end{align*}
    which implies $\w_{\ominus}^{\star\top}\x_i > 0$ and $\sigma(\w_{\ominus}^{\star\top}\x_i) = \w_{\ominus}^{\star\top}\x_i$. Again, two activation patterns are possible:
    \begin{itemize}
        \item If $\w_{\oplus}^{\star\top}\x_i \leq 0$, then we have
        $\sigma(\w_{\oplus}^{\star\top}\x_i) - \sigma(\w_{\ominus}^{\star\top}\x_i) = -\w_{\ominus}^{\star\top}\x_i = y_i$.
        \item If $\w_{\oplus}^{\star\top}\x_i \geq 0$, then we have
        $\sigma(\w_{\oplus}^{\star\top}\x_i) - \sigma(\w_{\ominus}^{\star\top}\x_i) = \w_{\oplus}^{\star\top}\x_i - \w_{\ominus}^{\star\top}\x_i = y_i$.
    \end{itemize}
    (Note that $\w_{\oplus}^{\star\top}\x_i = 0$ is covered by both cases.) 
    
    Combining the two cases (in total four patterns), there exist disjoint partitions
    \begin{align*}
        S_1 \cup S_2 = S_+, \; S_1\cap S_2 =\varnothing, \text{ and } S_3 \cup S_4 = S_-, \; S_3\cap S_4 =\varnothing,
    \end{align*}
    such that the optimal solution $\{\w_{\oplus}^\star, \w_{\ominus}^\star\}$ satisfies
    \begin{alignat}{3}
        &\w_{\oplus}^{\star\top}\x_i \phantom{-\w_{\ominus}^{\star\top}\x_i} &= y_i, \phantom{-\w_{\oplus}^{\star\top}\x_i < 0, -}\w_{\ominus}^{\star\top}\x_i\leq 0, &\quad\text{ for all } i\in S_1,\nonumber\\
        &\w_{\oplus}^{\star\top}\x_i - \w_{\ominus}^{\star\top}\x_i &= y_i, \phantom{-\w_{\oplus}^{\star\top}\x_i < 0} -\w_{\ominus}^{\star\top}\x_i \leq 0, &\quad\text{ for all } i\in S_2,\nonumber\\
        &\phantom{\w_{\oplus}^{\star\top}\x_i} - \w_{\ominus}^{\star\top}\x_i &= y_i, \phantom{-}\w_{\oplus}^{\star\top}\x_i\leq 0, \phantom{-\w_{\ominus}^{\star\top}\x_i < 0,} &\quad\text{ for all } i\in S_3,\nonumber\\
        &\w_{\oplus}^{\star\top}\x_i - \w_{\ominus}^{\star\top}\x_i &= y_i, -\w_{\oplus}^{\star\top}\x_i \leq 0, \phantom{-\w_{\ominus}^{\star\top}\x_i < 0,} &\quad\text{ for all } i\in S_4.\nonumber
    \end{alignat}
    These constraints are exactly those in~\eqref{eq:multiple_relu_w_star_eq_form}. Moreover, the feasible set of~\eqref{eq:multiple_relu_w_star_eq_form} is a subset of the feasible set of~\eqref{eq:multiple_relu_minimum_norm_sol}, since every feasible solution of~\eqref{eq:multiple_relu_w_star_eq_form} also satisfies the constraints of~\eqref{eq:multiple_relu_minimum_norm_sol} (the converse need not hold). Since $\{\w_{\oplus}^\star, \w_{\ominus}^\star\}$ is feasible for both problems and is optimal for the larger feasible set~\eqref{eq:multiple_relu_minimum_norm_sol}, it must also be optimal for the restricted problem~\eqref{eq:multiple_relu_w_star_eq_form}.
\end{proof}

\subsection{Proof of Theorem~\ref{thm:multiple_relu_gd_high_dim_implicit_bias}  (High-dimensional Implicit Bias)}\label{app:proof_multiple_relu_gd_high_dim_implicit_bias}
In this section, we present the proof of Theorem~\ref{thm:multiple_relu_gd_high_dim_implicit_bias}. For the $2$-ReLU model ($m=2$), the primal-dual gradient update in~\eqref{eq:primal_dual_update} simplifies to
\begin{subequations}
\begin{alignat}{2}
    &\text{(Primal) }\quad\quad &&\bbeta_{\oplus}^{(t+1)} = \bbeta_{\oplus}^{(t)} - \eta \X \X^\top \D(\bbeta_{\oplus}^{(t)})(h_{\bTheta^{(t)}}(\X) - \y),\label{eq:mul_primal_update_plus} \\
    &\text{(Dual) }\quad\quad &&\balpha_{\oplus}^{(t+1)} = \balpha_{\oplus}^{(t)} - \eta \D(\bbeta_{\oplus}^{(t)})(h_{\bTheta^{(t)}}(\X) - \y),\label{eq:mul_dual_update_plus}
\end{alignat}
\end{subequations}
and
\begin{subequations}
\begin{alignat}{2}
    &\text{(Primal) }\quad\quad &&\bbeta_{\ominus}^{(t+1)} = \bbeta_{\ominus}^{(t)} + \eta \X \X^\top \D(\bbeta_{\ominus}^{(t)})(h_{\bTheta^{(t)}}(\X) - \y),\label{eq:mul_primal_update_minus} \\
    &\text{(Dual) }\quad\quad &&\balpha_{\ominus}^{(t+1)} = \balpha_{\ominus}^{(t)} + \eta \D(\bbeta_{\ominus}^{(t)})(h_{\bTheta^{(t)}}(\X) - \y).\label{eq:mul_dual_update_minus}
\end{alignat}
\end{subequations}
Before proceeding to the proof, we again introduce a set of sufficient conditions under which the signs of the primal variables agree with the signs of the labels times the sign of the ReLU neuron at iteration $t$, and moreover, these conditions are preserved at iteration $t+1$. We use the results of Lemma~\ref{lem:pos_stay_pos} and Lemma~\ref{lem:neg_stay_neg} again to prove Lemma~\ref{lem:primal_label_same_sign_two}.

\begin{lemma}\label{lem:primal_label_same_sign_two}
    Under Assumptions~\ref{asm:label_bounds} and~\ref{asm:high_dim_data}, suppose the gradient descent step size satisfies $\eta \leq \frac{1}{C_g\nnorm[1]{\blambda}}$. For a $2$-ReLU model, if the following eight conditions hold at some iteration $t\geq 0$, then they also hold at iteration $t+1$. 
    \begin{enumerate}[label=\alph*.,ref=\alph*]
        \item \label{cond:beta_plus_pos} $\beta_{\oplus,i}^{(t)} > 0$ for all $i \in [n]$ with $y_i > 0$.
        \item \label{cond:beta_minus_pos} $\beta_{\ominus,j}^{(t)} > 0$ for all $j \in [n]$ with $y_j < 0$.
        \item \label{cond:alpha_plus_bounds} $-\frac{3y_{\max}}{C_g\nnorm[1]{\blambda}} \leq \alpha_{\oplus,j}^{(t)} \leq -\frac{y_{\min}}{C_{\alpha}\nnorm[1]{\blambda}}$ for all $j \in [n]$ with $y_j < 0$.
        \item \label{cond:alpha_minus_bounds} $-\frac{3y_{\max}}{C_g\nnorm[1]{\blambda}} \leq \alpha_{\ominus,i}^{(t)} \leq -\frac{y_{\min}}{C_{\alpha}\nnorm[1]{\blambda}}$ for all $i \in [n]$ with $y_i > 0$.
        \item \label{cond:sigma_beta_bounds} $\nnorm[2]{\bbeta_{\oplus, +}^{(t)} - \y_+}\leq C_y\nnorm[2]{\y_+}$, and $\nnorm[2]{\bbeta_{\ominus,-}^{(t)} + \y_-} \leq C_y\nnorm[2]{\y_-}$.
        \item \label{cond:alpha_norm_bounds} $\nnorm[2]{\balpha_{\oplus}^{(t)}} \leq \frac{C_{\alpha}\sqrt{n}y_{\max}}{\nnorm[1]{\blambda}}$ and $\nnorm[2]{\balpha_{\ominus}^{(t)}} \leq \frac{C_{\alpha}\sqrt{n}y_{\max}}{\nnorm[1]{\blambda}}$.
        \item \label{cond:beta_plus_neg} $\beta_{\oplus,j}^{(t)} \leq 0$ for all $j \in [n]$ with $y_j < 0$.
        \item \label{cond:beta_minus_neg} $\beta_{\ominus,i}^{(t)} \leq 0$ for all $i \in [n]$ with $y_i > 0$.
    \end{enumerate}
    Consequently, the set of active examples consists exactly of the positively labeled examples for the positive neuron, and the activation pattern remains unchanged, i.e., $\D(\bbeta_{\oplus}^{(t)}) = \D(\bbeta_{\oplus}^{(t+1)})$. The set of active examples consists exactly of the negatively labeled examples for the negative neuron, and the activation pattern remains unchanged, i.e., $\D(\bbeta_{\ominus}^{(t)}) = \D(\bbeta_{\ominus}^{(t+1)})$.
\end{lemma}

\begin{proof}(Lemma~\ref{lem:primal_label_same_sign_two})
    We now verify that these conditions are preserved from iteration $t$ to $t+1$.

\begin{enumerate}[label=Part (\alph*):,leftmargin=3\parindent]
    \item By conditions~\eqref{cond:beta_plus_pos},~\eqref{cond:beta_minus_pos},~\eqref{cond:sigma_beta_bounds},~\eqref{cond:beta_plus_neg} and~\eqref{cond:beta_minus_neg} at iteration $t$, we have 
    \begin{align*}
        \nnorm[2]{h_{\bTheta^{(t)}}(\X) - \y}^2 &= \nnorm[2]{\sigma(\bbeta_{\oplus}^{(t)}) - \sigma(\bbeta_{\ominus}^{(t)}) - \y}^2\\
        &= \nnorm[2]{\sigma(\bbeta_{\oplus}^{(t)}) -\begin{bmatrix} \y_+ \\ \zero \end{bmatrix} - (\sigma(\bbeta_{\ominus}^{(t)}) + \begin{bmatrix}  \zero \\ \y_- \end{bmatrix})}^2\\
        &= \nnorm[2]{\bbeta_{\oplus,+}^{(t)} -\y_+}^2 + \nnorm[2]{\bbeta_{\ominus,-}^{(t)} + \y_-}^2 \leq C_y^2\nnorm[2]{\y}^2,
    \end{align*}
    and therefore, $\nnorm[2]{h_{\bTheta^{(t)}}(\X) - \y} \leq C_y\nnorm[2]{\y}$. Together with $h_{\bTheta^{(t)}}(\x_i) = \beta_{\oplus,i}^{(t)}$ and  condition~\eqref{cond:beta_plus_pos}, the assumptions of Lemma~\ref{lem:pos_stay_pos} are satisfied for all $i$ with $y_i > 0$. Consequently, $\beta_{\oplus,i}^{(t+1)} > 0$ for all $i \in [n]$ with $y_i > 0$.
    
    \item According to the proof of part~\eqref{cond:beta_plus_pos}, we have $\nnorm[2]{h_{\bTheta^{(t)}}(\X) - \y} \leq C_y\nnorm[2]{\y}$ and\\ $-h_{\bTheta^{(t)}}(\x_j) = \beta_{\ominus,j}^{(t)}$. Together with condition~\eqref{cond:beta_minus_pos}, the assumptions of Lemma~\ref{lem:pos_stay_pos} are satisfied for all $j$ with $y_j < 0$. Consequently, we have $\beta_{\ominus,j}^{(t+1)} > 0$ for all $j \in [n]$ with $y_j < 0$.
    
    \item According to the dual gradient update in Equation~\eqref{eq:mul_dual_update_plus}, and using condition~\eqref{cond:beta_plus_neg} at iteration $t$, we have:
    \[
    \alpha_{\oplus,j}^{(t+1)} = \alpha_{\oplus,j}^{(t)} \quad \text{for all } j \in [n] \text{ with } y_j < 0.
    \]
    Therefore, condition (\ref{cond:alpha_plus_bounds}) continues to hold at iteration $t + 1$.
    
    \item According to the dual gradient update in Equation~\eqref{eq:mul_dual_update_minus}, and using conditions~\eqref{cond:beta_minus_neg} at iteration $t$, we have:
    \[
    \alpha_{\ominus,i}^{(t+1)} = \alpha_{\ominus,i}^{(t)} \quad \text{for all } i \in [n] \text{ with } y_i > 0.
    \]
    Therefore, condition (\ref{cond:alpha_minus_bounds}) continues to hold at iteration $t + 1$.
    
    \item By conditions (\ref{cond:beta_plus_pos}), (\ref{cond:beta_minus_pos}), (\ref{cond:beta_plus_neg}), and (\ref{cond:beta_minus_neg}), the gradient update at iteration $t$ for $\bbeta_{\oplus}^{(t)}$ depends only on the positively labeled examples, and the update for $\bbeta_{\ominus}^{(t)}$ depends only on the negatively labeled examples. Hence, the gradient update for an individual neuron is equivalent to gradient descent on a certain linear regression problem. Since the step size satisfies $\eta \leq \frac{1}{C_g\nnorm[1]{\blambda}}$, the linear regression squared loss is monotonically nonincreasing (as in the proof of Lemma~\ref{lem:conv_step_size}), and by condition~\eqref{cond:sigma_beta_bounds} at iteration $t$, we obtain
    \begin{align*}
        \nnorm[2]{\bbeta_{\oplus,+}^{(t+1)} - \y_+} &\leq \nnorm[2]{\bbeta_{\oplus,+}^{(t)} - \y_+}\leq C_y\nnorm[2]{\y_+},\\
        \nnorm[2]{-\bbeta_{\ominus,-}^{(t+1)} - \y_-} &\leq \nnorm[2]{-\bbeta_{\ominus,-}^{(t)} - \y_-}\leq C_y\nnorm[2]{\y_-},
    \end{align*}
    where we use $\y_+$ ($\y_-$) to denote the vector of positively labeled (negatively labeled) examples. Therefore, condition (\ref{cond:sigma_beta_bounds}) holds at iteration $t + 1$.
    
    \item Following the same argument as in Part (\ref{eq:cond4}) of Lemma \ref{lem:primal_label_same_sign}, using conditions (\ref{cond:alpha_plus_bounds}), (\ref{cond:alpha_minus_bounds}), and (\ref{cond:sigma_beta_bounds}) at iteration $t + 1$, together with the eigenvalue bounds from Lemma \ref{lem:concentration_eigenvalues}, we have
    \[
    \nnorm[2]{\balpha_{\oplus}^{(t+1)}} \leq \frac{C_{\alpha}\sqrt{n}y_{\max}}{\nnorm[1]{\blambda}}, \quad \nnorm[2]{\balpha_{\ominus}^{(t+1)}} \leq \frac{C_{\alpha}\sqrt{n}y_{\max}}{\nnorm[1]{\blambda}}
    \]
    with probability at least $1-2e^{-n/C_g}$.
    Thus, condition (\ref{cond:alpha_norm_bounds}) holds at iteration $t + 1$.
    
    \item By Lemma~\ref{lem:neg_stay_neg}, since conditions (\ref{cond:alpha_plus_bounds}) and (\ref{cond:alpha_norm_bounds}) hold at iteration $t + 1$, we conclude that $\beta_{\oplus,j}^{(t+1)} \leq 0$ for all $j \in [n]$ with $y_j < 0$. Thus, condition (\ref{cond:beta_plus_neg}) holds at iteration $t + 1$.
    
    \item Similarly, since conditions (\ref{cond:alpha_minus_bounds}) and (\ref{cond:alpha_norm_bounds}) hold at iteration $t + 1$, we have $\beta_{\ominus,i}^{(t+1)} \leq 0$ for all $i \in [n]$ with $y_i > 0$. Thus, condition (\ref{cond:beta_minus_neg}) holds at iteration $t + 1$.
\end{enumerate}
\end{proof}

Equipped with Lemma~\ref{lem:primal_label_same_sign_two}, we are ready to prove Theorem~\ref{thm:multiple_relu_gd_high_dim_implicit_bias}.

\begin{proof}(Theorem~\ref{thm:multiple_relu_gd_high_dim_implicit_bias})
The proof follows a similar structure to that of Theorem~\ref{thm:single_relu_gd_high_dim_implicit_bias} for single ReLU models, but now we must track the dynamics for both $\w_{\oplus}$ and $\w_{\ominus}$ simultaneously. Equipped with sufficient conditions under which the activation patterns are preserved in Lemma~\ref{lem:primal_label_same_sign_two}, we verify these conditions hold after the first gradient step, and use induction to characterize the full gradient descent dynamics.

We first verify that the iterate at $t = 1$ satisfies all the sufficient conditions. With the initialization
\[
\w_{\oplus}^{(0)} = \X^\top\XXi\bepsilon_{\oplus}, \quad \w_{\ominus}^{(0)} = \X^\top\XXi\bepsilon_{\ominus},
\]
we have $\bbeta_{\oplus}^{(0)} = \bepsilon_{\oplus}$ and $\bbeta_{\ominus}^{(0)} = \bepsilon_{\ominus}$. By the theorem assumptions, $0 < \epsilon_{\oplus,i} \leq \frac{1}{2C_{\alpha}}y_{\min}$ and $0 < \epsilon_{\ominus,i} \leq \frac{1}{2C_{\alpha}}y_{\min}$ for all $i \in [n]$. Using the primal gradient updates in Equations~\eqref{eq:mul_primal_update_plus} and \eqref{eq:mul_primal_update_minus}, we have
\begin{align}
    \bbeta_{\oplus}^{(1)} &= \bepsilon_{\oplus} - \eta\X\X^\top\D(\bepsilon_{\oplus})(h_{\bTheta^{(0)}}(\X) - \y)\nonumber \\
    &= \bepsilon_{\oplus} - \eta\X\X^\top\D(\bepsilon_{\oplus})(\sigma(\bepsilon_{\oplus}) - \sigma(\bepsilon_{\ominus}) - \y)\nonumber \\
    &= \bepsilon_{\oplus} - \eta\X\X^\top(\bepsilon_{\oplus} - \bepsilon_{\ominus} - \y),\label{eq:beta_oplus_one}
\end{align}
where the last equality uses the fact that $\bepsilon_{\oplus} > \mathbf{0}$ and $\bepsilon_{\ominus} > \mathbf{0}$ componentwise, so $\D(\bepsilon_{\oplus}) = \I$, $\sigma(\bepsilon_{\oplus}) = \bepsilon_{\oplus}$, and $\sigma(\bepsilon_{\ominus}) = \bepsilon_{\ominus}$. Similarly, we have
\begin{align}
    \bbeta_{\ominus}^{(1)} &= \bepsilon_{\ominus} + \eta\X\X^\top\D(\bepsilon_{\ominus})(\sigma(\bepsilon_{\oplus}) - \sigma(\bepsilon_{\ominus}) - \y)\nonumber \\
    &= \bepsilon_{\ominus} + \eta\X\X^\top(\bepsilon_{\oplus} - \bepsilon_{\ominus} - \y).\label{eq:beta_ominus_one}
\end{align}
For the dual variables, we have $\balpha_{\oplus}^{(0)} = (\X\X^\top)^{-1}\bepsilon_{\oplus}$ and $\balpha_{\ominus}^{(0)} = (\X\X^\top)^{-1}\bepsilon_{\ominus}$. The dual updates give:
\begin{align*}
    \balpha_{\oplus}^{(1)} &= \balpha_{\oplus}^{(0)} - \eta\D(\bepsilon_{\oplus})(\bepsilon_{\oplus} - \bepsilon_{\ominus} - \y) \\
    &= \XXi\bepsilon_{\oplus} - \eta(\bepsilon_{\oplus} - \bepsilon_{\ominus} - \y) \\
    &= \eta\pts{\y - \bepsilon_{\oplus} + \bepsilon_{\ominus} + \frac{1}{\eta}\XXi\bepsilon_{\oplus}},
\end{align*}
and
\begin{align*}
    \balpha_{\ominus}^{(1)} &= \balpha_{\ominus}^{(0)} + \eta\D(\bepsilon_{\ominus})(\bepsilon_{\oplus} - \bepsilon_{\ominus} - \y) \\
    &= \XXi\bepsilon_{\ominus} + \eta(\bepsilon_{\oplus} - \bepsilon_{\ominus} - \y) \\
    &= \eta\pts{-\y + \bepsilon_{\oplus} - \bepsilon_{\ominus} + \frac{1}{\eta}\XXi\bepsilon_{\ominus}}.
\end{align*}
We now verify each condition at $t = 1$.

\begin{enumerate}[label=Part (\alph*):,leftmargin=3\parindent,ref=\alph*]
    \item For all $i \in [n]$ with $y_i > 0$, we apply Lemma~\ref{lem:pos_stay_pos}. Since $\beta_{\oplus,i}^{(0)} = \epsilon_{\oplus,i} > 0$, $h_{\bTheta^{(0)}}(\x_i) = \beta_{\oplus,i}^{(0)} - \beta_{\ominus,i}^{(0)} \leq \beta_{\oplus,i}^{(0)}$, and 
    \begin{align}
        \nnorm[2]{h_{\bTheta^{(0)}}(\X) -\y} &= \nnorm[2]{\bepsilon_{\oplus} - \bepsilon_{\ominus} -\y}\nonumber\\
        &\leq \nnorm[2]{\bepsilon_{\oplus}} + \nnorm[2]{\bepsilon_{\ominus}} + \nnorm[2]{\y} \leq \frac{\sqrt{n}}{C_{\alpha}}y_{\min} + \nnorm[2]{\y} \leq C_y\nnorm[2]{\y},\label{eq:error_t0}
    \end{align}
    where we have used $C_y \geq \frac{1}{C_{\alpha}} + 1$. We conclude that $\beta_{\oplus,i}^{(1)} > 0$ for all $i$ with $y_i > 0$.

    \item For all $j \in [n]$ with $y_j < 0$, we apply Lemma~\ref{lem:pos_stay_pos}.
    Since $\beta_{\ominus,j}^{(0)} = \epsilon_{\ominus,j} > 0$, $-h_{\bTheta^{(0)}}(\x_j) = -\beta_{\oplus,j}^{(0)} + \beta_{\ominus,j}^{(0)} \leq \beta_{\ominus,j}^{(0)}$, and $\nnorm[2]{h_{\bTheta^{(0)}}(\X) -\y} \leq C_y\nnorm[2]{\y}$ by Equation~\eqref{eq:error_t0}, we conclude that $\beta_{\ominus,j}^{(1)} > 0$ for all $j\in[n]$ with $y_j < 0$.

    \item \label{cond:oplus_alpha}For all $j \in [n]$ with $y_j < 0$, we verify that $\alpha_{\oplus,j}^{(1)}$ satisfies the required upper and lower bounds. For the upper bound, we have
    \begin{align*}
        \alpha_{\oplus,j}^{(1)} &= \eta\pts{y_j - \epsilon_{\oplus,j} + \epsilon_{\ominus,j} + \frac{1}{\eta}\e_j^\top\XXi\bepsilon_{\oplus}}\\
        &= \eta\pts{y_j-\epsilon_{\oplus,j} +\epsilon_{\ominus,j} +\frac{1}{\eta}\e_j^\top\mts{\frac{1}{\nnorm[1]{\blambda}}\I + \pts{\XXi - \frac{1}{\nnorm[1]{\blambda}}\I}}\bepsilon_{\oplus}}\\
        &= \eta\pts{y_j-\epsilon_{\oplus,j}+ \epsilon_{\ominus,j} + \frac{\epsilon_{\oplus,j}}{\eta\nnorm[1]{\blambda}} + \frac{1}{\eta}\e_j^\top \pts{\XXi - \frac{1}{\nnorm[1]{\blambda}}\I}\bepsilon_{\oplus}}\\
        &\leq \eta\pts{y_j + \epsilon_{\ominus,j} + \frac{\epsilon_{\oplus,j}}{\eta\nnorm[1]{\blambda}}+ \frac{1}{\eta}\nnorm[2]{\XXi - \frac{1}{\nnorm[1]{\blambda}}\I}\nnorm[2]{\bepsilon_{\oplus}}},
    \end{align*}
    where the inequality drops the negative term $-\epsilon_{\oplus,j}$. By Corollary~\ref{cor:concentration_op_norm}, we have
    \begin{align*}
        \nnorm[2]{\XXi - \frac{1}{\nnorm[1]{\blambda}}\I} \leq \frac{C_gC}{\nnorm[1]{\blambda}} \cdot \max \pts{\sqrt{\frac{n}{d_2}}, \frac{n}{d_{\infty}}},
    \end{align*}
    with probability at least $1 - 2 \exp(-n(Cc - \ln 9))$. Moreover, by the theorem assumptions, $\epsilon_{\oplus, j} \leq \frac{1}{2C_{\alpha}}y_{\min}$, $\epsilon_{\ominus,j} \leq \frac{1}{2C_{\alpha}}y_{\min}$, and $\frac{1}{\eta} \leq CC_g\nnorm[1]{\blambda}$. Combining these bounds yields
    \begin{align}
        \alpha_{\oplus,j}^{(1)} &\leq \frac{1}{CC_g\nnorm[1]{\blambda}}\bigg(-y_{\min} + \frac{1}{2C_{\alpha}}y_{\min} + \frac{CC_g}{2C_{\alpha}}y_{\min} + C^2C_g^2\cdot \max \pts{\sqrt{\frac{n}{d_2}}, \frac{n}{d_{\infty}}} \cdot \frac{\sqrt{n}}{2C_{\alpha}}y_{\min}\bigg)\nonumber \\
        &\leq \frac{1}{CC_g\nnorm[1]{\blambda}}\pts{-y_{\min} + \frac{1}{2C_{\alpha}}y_{\min}+ \frac{CC_g}{2C_{\alpha}}y_{\min} + C^2C_g^2\cdot \frac{y_{\min}}{C_0y_{\max}} \cdot \frac{1}{2C_{\alpha}}y_{\min}}\nonumber\\
        &=  -\frac{y_{\min}}{C_{\alpha}\nnorm[1]{\blambda}}\pts{\frac{C_{\alpha}}{CC_g} - \frac{1}{2CC_g} - \frac{1}{2} - \frac{CC_gy_{\min}}{2C_0y_{\max}}}\nonumber \\
        &\leq -\frac{y_{\min}}{C_{\alpha}\nnorm[1]{\blambda}}\label{eq:oplus_alpha_upper}.
    \end{align}
    The second inequality follows from $d_2 \geq C_0^2\frac{n^2 y_{\max}^2}{y_{\min}^2}$ and $d_{\infty} \geq C_0\frac{n^{1.5} y_{\max}}{y_{\min}}$ in Assumption~\ref{asm:high_dim_data}, and the last inequality follows the relationship between constants $C_0\gtrsim C_{\alpha}^2$ and $C_{\alpha}\gtrsim \max\{C_g^2, C_yC_g\}$. For the lower bound, we have
    \begin{align}
        \alpha_{\oplus,j}^{(1)} &= \eta\pts{y_j-\epsilon_{\oplus,j}+ \epsilon_{\ominus,j} + \frac{\epsilon_{\oplus,j}}{\eta\nnorm[1]{\blambda}} + \frac{1}{\eta}\e_j^\top \pts{\XXi - \frac{1}{\nnorm[1]{\blambda}}\I}\bepsilon_{\oplus}}\nonumber \\
        &\geq \eta\pts{-y_{\max}-\epsilon_{\oplus,j} - \frac{1}{\eta}\nnorm[2]{\XXi - \frac{1}{\nnorm[1]{\blambda}}\I}\nnorm[2]{\bepsilon_{\oplus}}}\nonumber \\
        &\geq \frac{1}{C_g\nnorm[1]{\blambda}}\pts{-y_{\max} - \frac{1}{2C_{\alpha}}y_{\min} - C^2C_g^2\cdot \max \pts{\sqrt{\frac{n}{d_2}}, \frac{n}{d_{\infty}}} \cdot \frac{\sqrt{n}}{2C_{\alpha}}y_{\min}}\nonumber \\
        &\geq \frac{1}{C_g\nnorm[1]{\blambda}}\pts{-y_{\max} - \frac{1}{2C_{\alpha}}y_{\min} - C^2C_g^2\cdot \frac{y_{\min}}{C_0 y_{\max}}\cdot \frac{1}{2C_{\alpha}}y_{\min}}\nonumber\\
        &\geq -\frac{3y_{\max}}{C_g\nnorm[1]{\blambda}},\label{eq:oplus_alpha_lower}
    \end{align}
    by the same arguments. Thus, $\alpha_{\oplus,j}^{(1)}$ satisfies both the required upper and lower bounds for all $j$ with $y_j < 0$.

    \item For all $i \in [n]$ with $y_i > 0$, we verify that $\alpha_{\ominus,i}^{(1)}$ satisfies the required bounds in the approach analogous to Part~\eqref{cond:oplus_alpha}. For the upper bound, we have
    \begin{align*}
        \alpha_{\ominus,i}^{(1)} &= \eta\pts{-y_i + \epsilon_{\oplus,i} - \epsilon_{\ominus,i} + \frac{1}{\eta}\e_i^\top\XXi\bepsilon_{\ominus}} \\
        &= \eta\pts{-y_i+\epsilon_{\oplus,i} -\epsilon_{\ominus,i} +\frac{1}{\eta}\e_i^\top\mts{\frac{1}{\nnorm[1]{\blambda}}\I + \pts{\XXi - \frac{1}{\nnorm[1]{\blambda}}\I}}\bepsilon_{\ominus}}\\
        &= \eta\pts{-y_i+\epsilon_{\oplus,i}- \epsilon_{\ominus,i} + \frac{\epsilon_{\ominus,i}}{\eta\nnorm[1]{\blambda}} + \frac{1}{\eta}\e_i^\top \pts{\XXi - \frac{1}{\nnorm[1]{\blambda}}\I}\bepsilon_{\ominus}}\\
        &\leq \eta\pts{-y_i + \epsilon_{\oplus,i} + \frac{\epsilon_{\ominus,i}}{\eta\nnorm[1]{\blambda}} + \frac{1}{\eta}\nnorm[2]{\XXi - \frac{1}{\nnorm[1]{\blambda}}\I}\nnorm[2]{\bepsilon_{\ominus}}},
    \end{align*}
    where the inequality drops the negative term $-\epsilon_{\ominus,i}$. Applying the upper bound in Corollary~\ref{cor:concentration_op_norm} and the theorem assumptions $\epsilon_{\oplus, i} \leq \frac{1}{2C_{\alpha}}y_{\min}$, $\epsilon_{\ominus,i} \leq \frac{1}{2C_{\alpha}}y_{\min}$, and $\frac{1}{\eta} \leq CC_g\nnorm[1]{\blambda}$, we have
    \begin{align*}
        \alpha_{\ominus,i}^{(1)} &\leq \frac{1}{CC_g\nnorm[1]{\blambda}}\bigg(-y_{\min} + \frac{1}{2C_{\alpha}}y_{\min} + \frac{CC_g}{2C_{\alpha}}y_{\min} + C^2C_g^2\cdot \max \pts{\sqrt{\frac{n}{d_2}}, \frac{n}{d_{\infty}}} \cdot \frac{\sqrt{n}}{2C_{\alpha}}y_{\min}\bigg) \\
        &\leq -\frac{y_{\min}}{C_{\alpha}\nnorm[1]{\blambda}},
    \end{align*}
    where the second inequality follows the argument used in Equation~\eqref{eq:oplus_alpha_upper}. For the lower bound, following the same argument as in Part~\eqref{cond:oplus_alpha}, we have
    \begin{align*}
        \alpha_{\ominus,i}^{(1)} &= \eta\pts{-y_i+\epsilon_{\oplus,i}- \epsilon_{\ominus,i} + \frac{\epsilon_{\ominus,i}}{\eta\nnorm[1]{\blambda}} + \frac{1}{\eta}\e_i^\top \pts{\XXi - \frac{1}{\nnorm[1]{\blambda}}\I}\bepsilon_{\ominus}}\\
        &\geq \eta\pts{-y_{\max} - \epsilon_{\ominus,i} - \frac{1}{\eta} \nnorm[2]{\XXi - \frac{1}{\nnorm[1]{\blambda}}\I}\nnorm[2]{\bepsilon_{\ominus}}}\\
        &\geq \frac{1}{C_g\nnorm[1]{\blambda}}\pts{-y_{\max} - \frac{1}{2C_{\alpha}}y_{\min} - C^2C_g^2\cdot \max \pts{\sqrt{\frac{n}{d_2}}, \frac{n}{d_{\infty}}} \cdot \frac{\sqrt{n}}{2C_{\alpha}}y_{\min}} \\
        &\geq -\frac{3y_{\max}}{C_g\nnorm[1]{\blambda}},
    \end{align*}
    where the last inequality follows follows the argument used in Equation~\eqref{eq:oplus_alpha_lower}. Thus, $\alpha_{\ominus,i}^{(1)}$ satisfies both bounds for all $i$ with $y_i > 0$.

    \item We verify that the primal variables $\bbeta_{\oplus}^{(1)}$ corresponding to positively labeled examples minus $\y_+$ satisfy the norm bound. Specifically, we show that $\nnorm[2]{\bbeta_{\oplus,+}^{(1)} -\y_+}^2 \leq C_y^2\nnorm[2]{\y_+}^2$. According to Equation~\eqref{eq:beta_oplus_one}, we have
    \begin{align}
        \nnorm[2]{\bbeta_{\oplus,+}^{(1)} -\y_+}^2 &= \sum_{i: y_i>0} \pts{\beta_{\oplus,i}^{(1)} - y_i}^2\nonumber\\
            &= \sum_{i: y_i>0} \pts{\underbrace{\epsilon_{\oplus,i} -\eta\e_i^\top\XX\pts{\bepsilon_{\oplus} - \bepsilon_{\ominus} -\y} - y_i}_{\eqqcolon T_i}}^2.\label{eq:oplus_diff}
    \end{align}
    Next, we bound the term $T_i\coloneqq\epsilon_{\oplus,i} -\eta\e_i^\top\XX\pts{\bepsilon_{\oplus} - \bepsilon_{\ominus} -\y} - y_i$ for all $i\in[n]$ with $y_i > 0$. We have
    \begin{align*}
        T_i&=\epsilon_{\oplus,i} -\eta\e_i^\top\XX\pts{\bepsilon_{\oplus} -\bepsilon_{\ominus} -\y} - y_i\\
        &= (\epsilon_{\oplus,i} - y_i) -\eta\e_i^\top\mts{\nnorm[1]{\blambda}\I + \pts{\XX - \nnorm[1]{\blambda}\I}}\pts{\bepsilon_{\oplus} -\bepsilon_{\ominus} -\y}\\
        &= (1 - \eta\nnorm[1]{\blambda})\epsilon_{\oplus,i} + \eta\nnorm[1]{\blambda}\epsilon_{\ominus,i} - (1 - \eta\nnorm[1]{\blambda})y_i - \eta\e_i^\top\pts{\XX - \nnorm[1]{\blambda}\I}\pts{\bepsilon_{\oplus} -\bepsilon_{\ominus} -\y}.
    \end{align*}
    Since the step size assumption guarantees that $\frac{1}{CC_g\nnorm[1]{\blambda}} \leq \eta \leq \frac{1}{C_g\nnorm[1]{\blambda}}$, and $\epsilon_{\oplus,i},\epsilon_{\ominus,i}  \leq \frac{1}{2C_{\alpha}}y_{\min}$, we have 
    \begin{align*}
        (1 - \eta\nnorm[1]{\blambda})\epsilon_{\oplus,i} + \eta\nnorm[1]{\blambda}\epsilon_{\ominus,i} - (1 - \eta\nnorm[1]{\blambda})y_i &\leq \epsilon_{\oplus,i} + \eta\nnorm[1]{\blambda} \epsilon_{\ominus,i} - (1 - \eta\nnorm[1]{\blambda})y_i\\
        &\leq \pts{1 + \frac{1}{C_g}}\frac{1}{2C_{\alpha}}y_{\min} - \pts{1 - \frac{1}{C_g}}y_{\min}\\
        &< 0,
    \end{align*}
    with $C_{\alpha} \gtrsim C_g^2$. Hence, in order to upper bound $T_i^2$, it suffices to find the lower bound for $T_i$. We have
        \begin{align*}
            T_i &= (1 - \eta\nnorm[1]{\blambda})\epsilon_{\oplus,i} + \eta\nnorm[1]{\blambda}\epsilon_{\ominus,i} - (1 - \eta\nnorm[1]{\blambda})y_i - \eta\e_i^\top\pts{\XX - \nnorm[1]{\blambda}\I}\pts{\bepsilon_{\oplus} -\bepsilon_{\ominus} -\y}\\
            &\geq -y_i - \eta \nnorm[2]{\XX - \nnorm[1]{\blambda}\I}\nnorm[2]{\bepsilon_{\oplus} -\bepsilon_{\ominus} -\y},
        \end{align*}
        where the inequality drops the positive terms $(1 - \eta\nnorm[1]{\blambda})\epsilon_{\oplus,i}$, $\eta\nnorm[1]{\blambda}\epsilon_{\ominus,i}$, and $\eta\nnorm[1]{\blambda}y_i$. We again upper bound $\nnorm[2]{\XX - \nnorm[1]{\blambda}\I}$ by Corollary~\ref{cor:concentration_op_norm}. With probability at least $1 - 2 \exp(-n(Cc - \ln 9))$, we have
        \begin{align*}
            T_i &\geq -y_i - \eta \cdot C\nnorm[1]{\blambda}\cdot \max \pts{\sqrt{\frac{n}{d_2}}, \frac{n}{d_{\infty}}}\nnorm[2]{\bepsilon_{\oplus} - \bepsilon_{\ominus} -\y}\\
            &\geq -y_i - \frac{C}{C_g} \cdot \max \pts{\sqrt{\frac{n}{d_2}}, \frac{n}{d_{\infty}}}\nnorm[2]{\bepsilon_{\oplus} - \bepsilon_{\ominus} -\y},
        \end{align*}
        by applying $\eta \leq \frac{1}{C_g\nnorm[1]{\blambda}}$. Finally, we apply the upper bounds for $\nnorm[2]{\bepsilon_{\oplus}}$, $\nnorm[2]{\bepsilon_{\ominus}}$ and $\nnorm[2]{\y}$, and Assumption~\ref{asm:high_dim_data} ensures that $d_2 \geq C_0^2\frac{n^2 y_{\max}^2}{y_{\min}^2}$ and $d_{\infty} \geq C_0\frac{n^{1.5} y_{\max}}{y_{\min}}$. We have
        \begin{align*}
             T_i &\geq -y_i - \frac{C}{C_g} \max \pts{\sqrt{\frac{n}{d_2}}, \frac{n}{d_{\infty}}}\pts{\frac{\sqrt{n}}{C_{\alpha}}y_{\min} + \sqrt{n}y_{\max}}\\
             &\geq -y_i - \frac{Cy_{\min}}{C_gC_0y_{\max}}\pts{\frac{1}{C_{\alpha}}y_{\min} + y_{\max}}\\
             &\geq -y_i \pts{1 + \frac{2C}{C_gC_0}}\\
             &\geq -C_y y_i,
        \end{align*}
        with the choice of $C_y \geq 2$. Substituting $T_i^2 \leq C_y^2 y_i^2$ into Equation~\eqref{eq:oplus_diff}, we have
        \begin{align*}
            \nnorm[2]{\bbeta_{\oplus,+}^{(1)} -\y_+}^2 &\leq \sum_{i:y_i>0} C_y^2 y_i^2 = C_y^2 \nnorm[2]{\y_+}^2.
        \end{align*}
        As a result, we conclude that $ \nnorm[2]{\bbeta_{\oplus,+}^{(1)} -\y_+} \leq C_y \nnorm[2]{\y_+}$ as required.
    The same derivation holds for $\nnorm[2]{\bbeta_{\ominus, -}^{(1)} + \y_-} \leq C_y \nnorm[2]{\y_-}$ by an analogous argument. Therefore, condition~(\ref{cond:sigma_beta_bounds}) holds at $t = 1$.

    \item We verify the norm bounds on the dual variables. By the triangle inequality, we have
    \begin{align*}
        \nnorm[2]{\balpha_{\oplus}^{(1)}} &= \nnorm[2]{\eta\pts{\y - \bepsilon_{\oplus} + \bepsilon_{\ominus} + \frac{1}{\eta}\XXi\bepsilon_{\oplus}}} \\
        &\leq \eta\pts{\nnorm[2]{\y} + \nnorm[2]{\bepsilon_{\oplus}} + \nnorm[2]{\bepsilon_{\ominus}} + \frac{1}{\eta}\nnorm[2]{\XXi}\nnorm[2]{\bepsilon_{\oplus}}}.
    \end{align*}
    Using $\nnorm[2]{\y} \leq \sqrt{n}y_{\max}$, $\nnorm[2]{\bepsilon_{\oplus}}, \nnorm[2]{\bepsilon_{\ominus}} \leq \frac{\sqrt{n}}{2C_{\alpha}}y_{\min}$, $\nnorm[2]{\XXi} \leq \frac{C_g}{\nnorm[1]{\blambda}}$, and $\frac{1}{CC_g\nnorm[1]{\blambda}} \leq \eta \leq \frac{1}{C_g\nnorm[1]{\blambda}}$, we have
    \begin{align*}
        \nnorm[2]{\balpha_{\oplus}^{(1)}} &\leq \frac{1}{C_g\nnorm[1]{\blambda}}\pts{\sqrt{n}y_{\max} + \frac{\sqrt{n}}{C_{\alpha}}y_{\min} + CC_g\nnorm[1]{\blambda} \cdot \frac{C_g}{\nnorm[1]{\blambda}} \cdot \frac{\sqrt{n}}{C_{\alpha}}y_{\min}} \\
        &\leq \frac{1}{C_g\nnorm[1]{\blambda}}\pts{3\sqrt{n}y_{\max}}\\
        &\leq \frac{C_{\alpha}\sqrt{n}y_{\max}}{\nnorm[1]{\blambda}},
    \end{align*}
    with $C_{\alpha}\gtrsim \max\{C_g^2, C_yC_g\}$. The same bound holds for $\nnorm[2]{\balpha_{\ominus}^{(1)}}$. Thus, condition~(\ref{cond:alpha_norm_bounds}) holds at $t = 1$.

    \item Since we have shown that $\alpha_{\oplus,j}^{(1)} \leq -\frac{y_{\min}}{C_{\alpha}\nnorm[1]{\blambda}}$ and $\nnorm[2]{\balpha_{\oplus}^{(1)}} \leq \frac{C_{\alpha}\sqrt{n}y_{\max}}{\nnorm[1]{\blambda}}$ for all $j\in[n]$ with $y_j < 0$, it follows from Lemma~\ref{lem:neg_stay_neg} that $\beta_{\oplus,j}^{(1)} \leq 0$ for all $j\in[n]$ with $y_j < 0$.

    \item Similarly, since we have shown that $\alpha_{\ominus,i}^{(1)} \leq -\frac{y_{\min}}{C_{\alpha}\nnorm[1]{\blambda}}$ and $\nnorm[2]{\balpha_{\ominus}^{(1)}} \leq \frac{C_{\alpha}\sqrt{n}y_{\max}}{\nnorm[1]{\blambda}}$ for all $i\in[n]$ with $y_i > 0$, it follows from Lemma~\ref{lem:neg_stay_neg} that $\beta_{\ominus,i}^{(1)} \leq 0$ for all $i\in[n]$ with $y_i > 0$.
\end{enumerate}
We have shown that at iteration $t=1$ the conditions in Lemma~\ref{lem:primal_label_same_sign_two} are satisfied, and by induction, these conditions will also hold for $t \geq 1$. As a result, the positive neuron $\w_{\oplus}$ is trained with only positive examples starting from the iteration $t=1$, and it is equivalent to linear regression using only positive examples with initialization $\w_{\oplus}^{(1)} = \eta\X^\top\Bigl(\y - \bepsilon_{\oplus} + \bepsilon_{\ominus} + \frac{1}{\eta}(\X\X^\top)^{-1}\bepsilon_{\oplus}\Bigr)$. Finally, since $\w_{\oplus}$ and $\w_{\ominus}$ are trained on disjoint subsets of examples, $\w_{\oplus}^{(\infty)}$ satisfies
    \begin{align*}
        \w_{\oplus}^{(\infty)} =\uargmin{\w\in\{\w:\X_+\w =\y_+\}} \nnorm[2]{\w - \w_{\oplus}^{(1)}},
    \end{align*}
    by Lemma~\ref{lem:conv_step_size}. The same arguments apply to the negative neuron $\w_{\ominus}$ as well. This completes the proof of Theorem~\ref{thm:multiple_relu_gd_high_dim_implicit_bias}.

\end{proof}


\subsection{Proof of Theorem~\ref{thm:multiple_relu_approx_to_w_star} (Implicit Bias Approximation to \texorpdfstring{$\w^\star$}{w*})}\label{app:proof_multiple_relu_approx_to_w_star}

\begin{proof}(Theorem~\ref{thm:multiple_relu_approx_to_w_star})
    We divide the proof into four steps, and formally show the result for $\w_{\oplus}^\star$. The result for $\w_{\ominus}^\star$ follows an identical series of steps. In the first step, we derive an upper bound for $\nnorm[2]{\balpha_{\oplus}^\star}$ and $\nnorm[2]{\balpha_{\ominus}^\star}$ where $\w_{\oplus}^\star \coloneqq\X^\top\balpha_{\oplus}^\star$ and $\w_{\ominus}^\star \coloneqq \X^\top\balpha_{\ominus}^\star$, by using the optimality of the objective function in~\eqref{eq:multiple_relu_minimum_norm_sol}. In the second step, we use the KKT conditions of~\eqref{eq:multiple_relu_w_star_eq_form} to find a representation of $\{\w_{\oplus}^\star, \w_{\ominus}^\star\}$. In Steps 3 and 4, we derive the corresponding upper bound and lower bound.
    \paragraph{Step 1: Upper bounds for $\nnorm[2]{\balpha_{\oplus}^\star}$ and $\nnorm[2]{\balpha_{\ominus}^\star}$.}~\\
    $\{\w_{\oplus}^\star, \w_{\ominus}^\star\}$ is the optimal solution to~\eqref{eq:multiple_relu_minimum_norm_sol} and it achieves the minimum objective in~\eqref{eq:multiple_relu_minimum_norm_sol}. In the proof of Lemma~\ref{lem:multiple_relu_w_star_eqivalence}, we introduce $\{\tilde{\w}_{\oplus}, \tilde{\w}_{\ominus}\}$ which is also a feasible solution to~\eqref{eq:multiple_relu_minimum_norm_sol}, where $\tilde{\w}_{\oplus} \coloneqq \X^\top(\XX)^{-1}\y_{\oplus}$ and $\tilde{\w}_{\ominus} \coloneqq \X^\top(\XX)^{-1}\y_{\ominus}$, with $y_{\oplus,i} = \max\{y_i, 0\}$ and $y_{\ominus,i} = -\min\{y_i, 0\}$. Therefore, by the optimality of $\{\w_{\oplus}^\star, \w_{\ominus}^\star\}$ in the objective, we have
    \begin{align*}
        \nnorm[2]{\w_{\oplus}^\star}^2 + \nnorm[2]{\w_{\ominus}^\star}^2 &= \balpha_{\oplus}^{\star\top}\XX\balpha_{\oplus}^\star + \balpha_{\ominus}^{\star\top}\XX\balpha_{\ominus}^\star\\
        &\leq \nnorm[2]{\tilde{\w}_{\oplus}}^2 + \nnorm[2]{\tilde{\w}_{\ominus}}^2\\
        &= \y_{\oplus}^\top\XXi\y_{\oplus} + \y_{\ominus}^\top\XXi\y_{\ominus}\\
        &\leq \nnorm[2]{\XXi}\nnorm[2]{\y_{\oplus}}^2 + \nnorm[2]{\XXi}\nnorm[2]{\y_{\ominus}}^2\\
        &\leq \frac{2C_gny_{\max}^2}{\nnorm[1]{\blambda}},
    \end{align*}
    where the last inequality uses Lemma~\ref{lem:concentration_eigenvalues} with probability at least $1-2e^{-n/C_g}$, and we have $\max\{\nnorm[2]{\y_{\oplus}}^2, \nnorm[2]{\y_{\ominus}}^2\} \leq \nnorm[2]{\y}^2 \leq ny_{\max}^2$. Therefore, we have 
    \begin{align*}
        \lambda_{n}(\XX)\nnorm[2]{\balpha_{\oplus}^\star}^2 &\leq \balpha_{\oplus}^{\star\top}\XX\balpha_{\oplus}^\star\\
        &\leq \balpha_{\oplus}^{\star\top}\XX\balpha_{\oplus}^\star + \balpha_{\ominus}^{\star\top}\XX\balpha_{\ominus}^\star\\
        &\leq \frac{2C_gny_{\max}^2}{\nnorm[1]{\blambda}}.
    \end{align*}
    As a result, we have $\nnorm[2]{\balpha_{\oplus}^\star} \leq \frac{\sqrt{2}C_g\sqrt{n}y_{\max}}{\nnorm[1]{\blambda}}$. This bound applies to $\nnorm[2]{\balpha_{\ominus}^\star}$ as well via an identical argument.

    \paragraph{Step 2: KKT conditions of $\{\w_{\oplus}^\star, \w_{\ominus}^\star\}$ by Lemma~\ref{lem:multiple_relu_w_star_eqivalence}.}~\\
    Based on Lemma~\ref{lem:multiple_relu_w_star_eqivalence}, the optimal solution $\{\w_{\oplus}^\star, \w_{\ominus}^\star\}$ of~\eqref{eq:multiple_relu_minimum_norm_sol} is also the optimal solution of a convex program~\eqref{eq:multiple_relu_w_star_eq_form}. Hence, we restate the convex program in~\eqref{eq:multiple_relu_w_star_eq_form} below
    \begin{alignat*}{3}
        &\myquad[2]\w_{\oplus}^\star, \w_{\ominus}^\star &= \uargmin{\w_{\oplus}, \w_{\ominus}}\frac{1}{2}\nnorm[2]{\w_{\oplus}}^2 + \frac{1}{2}\nnorm[2]{\w_{\ominus}}^2&\\
        \text{s.t. } &\w_{\oplus}^\top\x_i \phantom{-\w_{\ominus}^\top\x_i} &= y_i, \phantom{-\w_{\oplus}^\top\x_i < 0, -}\w_{\ominus}^\top\x_i\leq 0, &\quad\text{ for all } i\in S_1,\nonumber\\
        &\w_{\oplus}^\top\x_i - \w_{\ominus}^\top\x_i &= y_i, \phantom{-\w_{\oplus}^\top\x_i < 0} -\w_{\ominus}^\top\x_i \leq 0, &\quad\text{ for all } i\in S_2,\nonumber\\
        &\phantom{\w_{\oplus}^\top\x_i} - \w_{\ominus}^\top\x_i &= y_i, \phantom{-}\w_{\oplus}^\top\x_i\leq 0, \phantom{-\w_{\ominus}^\top\x_i < 0,} &\quad\text{ for all } i\in S_3,\nonumber\\
        &\w_{\oplus}^\top\x_i - \w_{\ominus}^\top\x_i &= y_i, -\w_{\oplus}^\top\x_i \leq 0, \phantom{-\w_{\ominus}^\top\x_i < 0,} &\quad\text{ for all } i\in S_4.\nonumber
    \end{alignat*}    
    The Lagrange function in terms of $\bdelta\in \R^n$ and non-negative $\bmu\in\R_+^n$ is given by 
    \begin{align*}
        \mathcal{L}=\frac{1}{2}\nnorm[2]{\w_{\oplus}}^2 + \frac{1}{2}\nnorm[2]{\w_{\ominus}}^2 &+ \sum_{i\in S_1}\delta_i\pts{\w_{\oplus}^\top\x_i - y_i} + \sum_{i\in S_1}\mu_i\pts{\w_{\ominus}^\top\x_i}\\
        &+ \sum_{i\in S_2}\delta_i\pts{\w_{\oplus}^\top\x_i - \w_{\ominus}^\top\x_i - y_i} - \sum_{i\in S_2}\mu_i\pts{\w_{\ominus}^\top\x_i}\\
        &+ \sum_{i\in S_3}\delta_i\pts{- \w_{\ominus}^\top\x_i - y_i} + \sum_{i \in S_3}\mu_i\pts{\w_{\oplus}^\top\x_i}\\
        &+ \sum_{i\in S_4}\delta_i\pts{\w_{\oplus}^\top\x_i - \w_{\ominus}^\top\x_i - y_i}- \sum_{i\in S_4}\mu_i\pts{\w_{\oplus}^\top\x_i}.
    \end{align*}
    The KKT conditions are given below.\\
    \textbf{Stationarity:}
    \begin{align}
        \pdv{\mathcal{L}}{\w_{\oplus}} &= \w_{\oplus}^\star + \sum_{i\in S_1}\delta_i^\star\x_i + \sum_{i\in S_2}\delta_i^\star\x_i + \sum_{i\in S_3}\mu_i^\star\x_i + \sum_{i\in S_4}(\delta_i^\star - \mu_i^\star)\x_i = 0,\nonumber\\
        \pdv{\mathcal{L}}{\w_{\ominus}} &= \w_{\ominus}^\star + \sum_{i\in S_1}\mu_i^\star\x_i - \sum_{i\in S_2}(\delta_i^\star+\mu_i^\star)\x_i - \sum_{i\in S_3}\delta_i^\star\x_i - \sum_{i\in S_4}\delta_i^\star\x_i = 0,\nonumber\\
        \Leftrightarrow \w_{\oplus}^\star &= -\sum_{i\in S_1}\delta_i^\star\x_i - \sum_{i\in S_2}\delta_i^\star\x_i - \sum_{i\in S_3}\mu_i^\star\x_i + \sum_{i\in S_4}(-\delta_i^\star+\mu_i^\star)\x_i,\label{eq:multiple_relu_w_p_star}\\
        \w_{\ominus}^\star &= - \sum_{i\in S_1}\mu_i^\star\x_i + \sum_{i\in S_2}(\delta_i^\star + \mu_i^\star)\x_i + \sum_{i\in S_3}\delta_i^\star\x_i + \sum_{i\in S_4}\delta_i^\star\x_i.\label{eq:multiple_relu_w_n_star}
    \end{align}
    \textbf{Primal feasibility:}
    \begin{alignat*}{3}
        &\w_{\oplus}^{\star\top}\x_i \phantom{-\w_{\ominus}^{\star\top}\x_i} &= y_i, \phantom{-\w_{\oplus}^{\star\top}\x_i < 0, -}\w_{\ominus}^{\star\top}\x_i\leq 0, &\quad\text{ for all } i\in S_1,\nonumber\\
        &\w_{\oplus}^{\star\top}\x_i - \w_{\ominus}^{\star\top}\x_i &= y_i, \phantom{-\w_{\oplus}^{\star\top}\x_i < 0} -\w_{\ominus}^{\star\top}\x_i \leq 0, &\quad\text{ for all } i\in S_2,\nonumber\\
        &\phantom{\w_{\oplus}^{\star\top}\x_i} - \w_{\ominus}^{\star\top}\x_i &= y_i, \phantom{-}\w_{\oplus}^{\star\top}\x_i\leq 0, \phantom{-\w_{\ominus}^{\star\top}\x_i < 0,} &\quad\text{ for all } i\in S_3,\nonumber\\
        &\w_{\oplus}^{\star\top}\x_i - \w_{\ominus}^{\star\top}\x_i &= y_i, -\w_{\oplus}^{\star\top}\x_i \leq 0, \phantom{-\w_{\ominus}^{\star\top}\x_i < 0,} &\quad\text{ for all } i\in S_4.\nonumber
    \end{alignat*}    
    \textbf{Dual feasibility:}
    \begin{align*}
        \delta_i^\star &\in \R, \text{ for all } i\in [n],\\
        \mu_i^\star &\geq 0, \text{ for all }i\in [n].
    \end{align*}
    \textbf{Complementary slackness:}
    \begin{align*}
        \sum_{i\in S_1}\mu_i^\star\pts{\w_{\ominus}^{\star\top}\x_i} + \sum_{i\in S_2}\mu_i^\star\pts{-\w_{\ominus}^{\star\top}\x_i} + \sum_{i \in S_3}\mu_i^\star\pts{\w_{\oplus}^{\star\top}\x_i} + \sum_{i \in S_4}\mu_i^\star\pts{-\w_{\oplus}^{\star\top}\x_i} = 0.
    \end{align*}
    Note that the representation of $\w_{\oplus}^\star$ and $\w_{\ominus}^\star$ shares the parameters $\{\delta_i^\star$ : $i\in S_2\cup S_4\}$. As a result, since we define $\w_{\oplus}^\star =\X^\top\balpha_{\oplus}^\star$ and $\w_{\ominus}^\star =\X^\top\balpha_{\ominus}^\star$, we can write $\alpha_{\oplus,i}^\star$ and $\alpha_{\ominus,i}^\star$ in terms of $\delta_i$ and $\mu_i$ for all $i\in[n]$ by Equations~\eqref{eq:multiple_relu_w_p_star} and~\eqref{eq:multiple_relu_w_n_star} as
    \begin{align*}
        \alpha_{\oplus,i}^\star=\conddd{-\delta_i^\star}{\; \forall i\in S_1}{-\delta_i^\star}{\; \forall i\in S_2}{-\mu_i^\star}{\; \forall i\in S_3}{-\delta_i^\star+\mu_i^\star}{\; \forall i\in S_4}, \text{ and } \alpha_{\ominus,i}^\star=\conddd{-\mu_i^\star}{\; \forall i\in S_1}{\delta_i^\star + \mu_i^\star}{\; \forall i\in S_2}{\delta_i^\star}{\; \forall i\in S_3}{\delta_i^\star}{\; \forall i\in S_4}.
    \end{align*}
    
    \paragraph{Step 3: Upper bound for $\nnorm[2]{\w_{\oplus}^{(\infty)} - \w_{\oplus}^{\star}}$.}~\\
    We now generalize the proof of Theorem~\ref{thm:single_relu_approx_to_w_star}. We first relate the distance between the predictors $\w_{\oplus}^{(\infty)}$ and $\w_{\oplus}^\star$ to the distance in their predictions, i.e., $\nnorm[2]{\X\w_{\oplus}^{(\infty)} - \X\w_{\oplus}^\star}$. Since both vectors lie in the span of the data $\{\x_i\}_{i=1}^n$, their difference has no component in the null space corresponding to the smallest $d-n$ eigenvalues of $\X^\top\X$. Therefore, we have
    \begin{align}
        \nnorm[2]{\X\w_{\oplus}^{(\infty)} - \X\w_{\oplus}^\star}^2 = \nnorm[2]{\X\pts{\w_{\oplus}^{(\infty)} -\w_{\oplus}^\star}}^2 &\geq \mu_{n}(\X^\top\X) \nnorm[2]{\w_{\oplus}^{(\infty)} -\w_{\oplus}^\star}^2\nonumber\\
        &= \mu_{n}(\XX) \nnorm[2]{\w_{\oplus}^{(\infty)} -\w_{\oplus}^\star}^2.\label{eq:multiple_relu_w_star_upper}
    \end{align}
    As a result, to derive an upper bound for $\nnorm[2]{\w_{\oplus}^{(\infty)} -\w_{\oplus}^\star}$, it suffices to upper bound the distance between their predictions $\nnorm[2]{\X\w_{\oplus}^{(\infty)} - \X\w_{\oplus}^\star}$. We begin with analyzing $\w_{\oplus}^{(\infty)}$. By Theorem~\ref{thm:multiple_relu_gd_high_dim_implicit_bias}, $\w_{\oplus}^{(\infty)}$ satisfies the following
    \begin{subequations}
    \begin{align}
        \w_{\oplus}^{(\infty)\top}\x_i &= y_i \myquad[21] \text{ for all } y_i > 0,\label{eq:w_p_inf_p}\\
        \alpha_{\oplus, j}^{(\infty)} &= \alpha_{\oplus,j}^{(1)} = \eta (y_j - \epsilon_{\oplus,j} +\epsilon_{\ominus,j} + \frac{1}{\eta}\e_j^\top\XXi\bepsilon_{\oplus}) \quad \text{ for all } y_j < 0,\label{eq:w_p_inf_n}
    \end{align}
    \end{subequations}
    and also all the conditions in Lemma~\ref{lem:primal_label_same_sign_two}.

    We know that $\w_{\oplus}^{(\infty)\top}\x_i = \w_{\oplus}^{\star\top}\x_i = y_i$ for all $i\in S_1$, and $\w_{\oplus}^{(\infty)\top}\x_i = y_i$ and $\w_{\oplus}^{\star\top}\x_i = y_i + \w_{\ominus}^{\star\top}\x_i$ for all $i\in S_2$. Therefore, we can write
    \begin{align}
        \nnorm[2]{\X\w_{\oplus}^{(\infty)} - \X\w_{\oplus}^\star}^2 &= \sum_{i=1}^n \pts{\w_{\oplus}^{(\infty)\top}\x_i - \w_{\oplus}^{\star\top}\x_i}^2\nonumber\\
        &= \sum_{i\in S_2} \pts{-\w_{\ominus}^{\star\top}\x_i}^2 + \sum_{i\in S_3} \pts{\w_{\oplus}^{(\infty)\top}\x_i - \w_{\oplus}^{\star\top}\x_i}^2 + \sum_{i\in S_4} \pts{\w_{\oplus}^{(\infty)\top}\x_i - \w_{\oplus}^{\star\top}\x_i}^2.\label{eq:multiple_relu_difference_in_prediction}
    \end{align}
    We start with upper bounding the term $(-\w_{\ominus}^{\star\top}\x_i)^2$ for all $i\in S_2$. For $i\in S_2$, by the complementary slackness, we either have $-\w_{\ominus}^{\star\top}\x_i = 0$ with $\mu_i^\star \geq 0$ or $-\w_{\ominus}^{\star\top}\x_i \leq 0$ with $\mu_i^\star = 0$. In the first case, we have $(-\w_{\ominus}^{\star\top}\x_i)^2 = 0$. In the second case, we define $\tilde{S}_2\subseteq S_2$ such that $\mu_i^\star = 0$ and $-\w_{\ominus}^{\star\top}\x_i \leq 0$ for all $i\in\tilde{S}_2$, and we will show that $\tilde{S}_2$ is empty with probability at least $1 - 2\exp(-n(Cc-\ln9))$. Since $\w_{\oplus}^{\star\top}\x_i - \w_{\ominus}^{\star\top}\x_i = y_i$ for all $i\in \tilde{S}_2\subseteq S_2$, we have
    \begin{align*}
        y_i = \w_{\oplus}^{\star\top}\x_i - \w_{\ominus}^{\star\top}\x_i &= \e_i^\top\XX\pts{\balpha_{\oplus}^\star - \balpha_{\ominus}^\star}\\
        &= \nnorm[1]{\blambda}(-2\delta_i^\star - \underbrace{\mu_i^\star}_{=0}) + \e_i^\top(\XX - \nnorm[1]{\blambda}\I)\pts{\balpha_{\oplus}^\star - \balpha_{\ominus}^\star}.
    \end{align*}
    The, for all $i\in \tilde{S}_2$, we have
    \begin{align*}
        \delta_i^\star = \frac{y_i - \e_i^\top(\XX - \nnorm[1]{\blambda}\I)\pts{\balpha_{\oplus}^\star - \balpha_{\ominus}^\star}}{-2\nnorm[1]{\blambda}}.
    \end{align*}
    Based on this representation of $\delta_i^\star$, for all $i\in \tilde{S}_2$, we have
    \begin{align*}
        \w_{\ominus}^{\star\top}\x_i &= \e_i^\top\XX\balpha_{\ominus}^\star\\
        &= \nnorm[1]{\blambda}(\delta_i^\star) + \e_i^\top\pts{\XX - \nnorm[1]{\blambda}\I}\balpha_{\ominus}^\star\\
        &=-\frac{y_i}{2} +\frac{1}{2}\e_i^\top\pts{\XX - \nnorm[1]{\blambda}\I}\pts{\balpha_{\oplus}^\star + \balpha_{\ominus}^\star}\\
        &\leq -\frac{y_i}{2} + \frac{1}{2}\nnorm[2]{\XX - \nnorm[1]{\blambda}\I}\pts{\nnorm[2]{\balpha_{\oplus}^\star} + \nnorm[2]{\balpha_{\ominus}^\star}}\\
        &\leq \nnorm[1]{\blambda}\mts{-\frac{y_i}{2\nnorm[1]{\blambda}} + \frac{1}{2} C\cdot \max \pts{\sqrt{\frac{n}{d_2}}, \frac{n}{d_{\infty}}} \cdot \frac{2\sqrt{2}C_g\sqrt{n}y_{\max}}{\nnorm[1]{\blambda}}}\\
        &\leq \nnorm[1]{\blambda}\mts{-\frac{y_{\min}}{2\nnorm[1]{\blambda}} + \frac{1}{2} C\cdot \frac{y_{\min}}{C_0y_{\max}} \cdot \frac{2\sqrt{2}C_g y_{\max}}{\nnorm[1]{\blambda}}}\\
        &< 0,
    \end{align*}
    where the inequalities above apply Corollary~\ref{cor:concentration_op_norm} and the upper bound of $\nnorm[2]{\balpha_{\oplus}^\star}$, $\nnorm[2]{\balpha_{\ominus}^\star}$ from Step 1, and substitute $d_2 \geq C_0^2\frac{n^2 y_{\max}^2}{y_{\min}^2}$ and $d_{\infty} \geq C_0\frac{n^{1.5} y_{\max}}{y_{\min}}$ in Assumption~\ref{asm:high_dim_data} with $C_0\gtrsim C_{\alpha}^2$ and $C_{\alpha}\gtrsim \max\{C_g^2, C_yC_g\}$. However, $\w_{\ominus}^{\star\top}\x_i < 0$ contradicts with the condition $-\w_{\ominus}^{\star\top}\x_i \leq 0$ for $i\in \tilde{S}_2$. Therefore, $\tilde{S}_2=\varnothing$. By combining these two cases, we conclude that $\sum_{i\in S_2} \pts{-\w_{\ominus}^{\star\top}\x_i}^2 = 0$.
    
    Next, we upper bound the term $\pts{\w_{\oplus}^{(\infty)\top}\x_i - \w_{\oplus}^{\star\top}\x_i}^2$ for all $i\in S_3$. We know that $\w_{\oplus}^{(\infty)\top}\x_i < 0$ in Theorem~\ref{thm:multiple_relu_gd_high_dim_implicit_bias}. For $\w_{\oplus}^{\star\top}\x_i$ with $i\in S_3$, by the complementary slackness, we either have $\w_{\oplus}^{\star\top}\x_i = 0$ with $\mu_i^\star \geq 0$ or $\mu_i^\star = 0$ with $\w_{\oplus}^{\star\top}\x_i \leq 0$. In the first case, $\w_{\oplus}^{\star\top}\x_i = 0$, we have
    \begin{align*}
        \w_{\oplus}^{(\infty)\top}\x_i - \w_{\oplus}^{\star\top}\x_i &= \w_{\oplus}^{(\infty)\top}\x_i\\
        &= \e_i^\top\XX\balpha_{\oplus}^{(\infty)}\\
        &= \e_i^\top\mts{\nnorm[1]{\blambda}\I + \pts{\XX - \nnorm[1]{\blambda}\I}}\balpha_{\oplus}^{(\infty)}\\
        &= \nnorm[1]{\blambda}\alpha_{\oplus,i}^{(\infty)} + \e_i^\top\pts{\XX - \nnorm[1]{\blambda}\I}\balpha_{\oplus}^{(\infty)}\\
        &\geq \nnorm[1]{\blambda}\alpha_{\oplus,i}^{(\infty)} - \nnorm[2]{\XX - \nnorm[1]{\blambda}\I}\nnorm[2]{\balpha_{\oplus}^{(\infty)}}\\
        &\geq \nnorm[1]{\blambda}\mts{\alpha_{\oplus,i}^{(\infty)} - C\cdot \max \pts{\sqrt{\frac{n}{d_2}}, \frac{n}{d_{\infty}}}\nnorm[2]{\balpha_{\oplus}^{(\infty)}}},
    \end{align*}
    where the last inequality applies Corollary~\ref{cor:concentration_op_norm}. Substituting the bounds of $\alpha_{\oplus,i}^{(\infty)}$ and $\nnorm[2]{\balpha_{\oplus}^{(\infty)}}$ from Lemma~\ref{lem:primal_label_same_sign_two}, we have
    \begin{align*}
        \w_{\oplus}^{(\infty)\top}\x_i - \w_{\oplus}^{\star\top}\x_i &\geq \nnorm[1]{\blambda}\mts{-\frac{3y_{\max}}{C_g\nnorm[1]{\blambda}} - C\cdot \max \pts{\sqrt{\frac{n}{d_2}}, \frac{n}{d_{\infty}}}\frac{C_{\alpha}y_{\max}}{\nnorm[1]{\blambda}}}\\
        &\geq \nnorm[1]{\blambda}\mts{-\frac{3y_{\max}}{C_g\nnorm[1]{\blambda}} - C\cdot \frac{y_{\min}}{C_0y_{\max}}\frac{C_{\alpha}y_{\max}}{\nnorm[1]{\blambda}}}\\
        &\geq -\frac{4}{C_g}y_{\max},
    \end{align*}
    where the inequalities substitute $d_2 \geq C_0^2\frac{n^2 y_{\max}^2}{y_{\min}^2}$ and $d_{\infty} \geq C_0\frac{n^{1.5} y_{\max}}{y_{\min}}$ in Assumption~\ref{asm:high_dim_data} with $C_0\gtrsim C_{\alpha}^2$ and $C_{\alpha}\gtrsim \max\{C_g^2, C_yC_g\}$. In the second case, $\alpha_{\oplus,i}^\star = -\mu_i^\star = 0$ for $i\in S_3$, we have
    \begin{align*}
        \w_{\oplus}^{(\infty)\top}\x_i - \w_{\oplus}^{\star\top}\x_i &= \e_i^\top\XX\pts{\balpha_{\oplus}^{(\infty)} -\balpha_{\oplus}^\star}\\
        &= \e_i^\top\mts{\nnorm[1]{\blambda}\I + \pts{\XX - \nnorm[1]{\blambda}\I}}\pts{\balpha_{\oplus}^{(\infty)} -\balpha_{\oplus}^\star}\\
        &= \nnorm[1]{\blambda}\alpha_{\oplus,i}^{(\infty)} + \e_i^\top\pts{\XX - \nnorm[1]{\blambda}\I}\pts{\balpha_{\oplus}^{(\infty)} -\balpha_{\oplus}^\star}\\
        &\geq \nnorm[1]{\blambda}\alpha_{\oplus,i}^{(\infty)} - \nnorm[2]{\XX - \nnorm[1]{\blambda}\I}\pts{\nnorm[2]{\balpha_{\oplus}^{(\infty)}} + \nnorm[2]{\balpha_{\oplus}^\star}}\\
        &\geq \nnorm[1]{\blambda}\mts{-\frac{3y_{\max}}{C_g\nnorm[1]{\blambda}} - C\cdot \max \pts{\sqrt{\frac{n}{d_2}}, \frac{n}{d_{\infty}}}\pts{\frac{C_{\alpha}\sqrt{n}y_{\max}}{\nnorm[1]{\blambda}} + \frac{\sqrt{2}C_g\sqrt{n}y_{\max}}{\nnorm[1]{\blambda}}}}\\
        &\geq \nnorm[1]{\blambda}\mts{-\frac{3y_{\max}}{C_g\nnorm[1]{\blambda}} - C\cdot \frac{y_{\min}}{C_0 y_{\max}}\pts{\frac{C_{\alpha}y_{\max}}{\nnorm[1]{\blambda}} + \frac{\sqrt{2}C_gy_{\max}}{\nnorm[1]{\blambda}}}}\\
        &\geq -\frac{4}{C_g}y_{\max},
    \end{align*}
    by applying the same argument and the upper bound from Step 1 that $\nnorm[2]{\balpha_{\oplus}^\star} \leq \frac{\sqrt{2}C_g\sqrt{n}y_{\max}}{\nnorm[1]{\blambda}}$. Therefore, we have $\pts{\w_{\oplus}^{(\infty)\top}\x_i - \w_{\oplus}^{\star\top}\x_i}^2 \leq \frac{16}{C_g^2} y_{\max}^2$ for all $i\in S_3$. 
    
    Next, we upper bound the term $\pts{\w_{\oplus}^{(\infty)\top}\x_i - \w_{\oplus}^{\star\top}\x_i}^2$ for all $i\in S_4$ in a similar way compared to $S_2$. For $i\in S_4$, by the complementary slackness, we either have $-\w_{\oplus}^{\star\top}\x_i = 0$ with $\mu_i^\star \geq 0$ or $-\w_{\oplus}^{\star\top}\x_i \leq 0$ with $\mu_i^\star = 0$. In the first case, $(-\w_{\ominus}^{\star\top}\x_i)^2 = 0$, and we have $\pts{\w_{\oplus}^{(\infty)\top}\x_i - \w_{\oplus}^{\star\top}\x_i}^2 = \pts{\w_{\oplus}^{(\infty)\top}\x_i}^2$. Therefore, we can reuse the upper bound we derived in $S_3$ such that $0 \geq \w_{\oplus}^{(\infty)\top}\x_i \geq -\frac{4}{C_g}y_{\max}$. In the second case, we define $\tilde{S}_4\subseteq S_4$ such that $\mu_i^\star = 0$ and $-\w_{\oplus}^{\star\top}\x_i \leq 0$ for all $i\in\tilde{S}_4$, and we will show that $\tilde{S}_4$ is empty with probability at least $1 - 2\exp(-n(Cc-\ln9))$. Since $\w_{\oplus}^{\star\top}\x_i - \w_{\ominus}^{\star\top}\x_i = y_i$ for all $i\in \tilde{S}_4\subseteq S_4$, we have
    \begin{align*}
        y_i = \w_{\oplus}^{\star\top}\x_i - \w_{\ominus}^{\star\top}\x_i &= \e_i^\top\XX\pts{\balpha_{\oplus}^\star - \balpha_{\ominus}^\star}\\
        &= \nnorm[1]{\blambda}(-2\delta_i^\star+\underbrace{\mu_i^\star}_{=0}) + \e_i^\top(\XX - \nnorm[1]{\blambda}\I)\pts{\balpha_{\oplus}^\star - \balpha_{\ominus}^\star}.
    \end{align*}
    For all $i\in \tilde{S}_4$, we have
    \begin{align*}
        \delta_i^\star = \frac{y_i - \e_i^\top(\XX - \nnorm[1]{\blambda}\I)\pts{\balpha_{\oplus}^\star - \balpha_{\ominus}^\star}}{-2\nnorm[1]{\blambda}}.
    \end{align*}
    Based on this representation of $\delta_i^\star$, for all $i\in \tilde{S}_4$, we have
    \begin{align*}
        \w_{\oplus}^{\star\top}\x_i &= \e_i^\top\XX\balpha_{\oplus}^\star\\
        &= \nnorm[1]{\blambda}(-\delta_i^\star ) + \e_i^\top\pts{\XX - \nnorm[1]{\blambda}\I}\balpha_{\oplus}^\star\\
        &=\frac{y_i}{2} +\frac{1}{2}\e_i^\top\pts{\XX - \nnorm[1]{\blambda}\I}\pts{\balpha_{\oplus}^\star + \balpha_{\ominus}^\star}\\
        &\leq \frac{y_i}{2} + \frac{1}{2}\nnorm[2]{\XX - \nnorm[1]{\blambda}\I}\pts{\nnorm[2]{\balpha_{\oplus}^\star} + \nnorm[2]{\balpha_{\ominus}^\star}}\\
        &\leq \nnorm[1]{\blambda}\mts{\frac{y_i}{2\nnorm[1]{\blambda}} + \frac{1}{2} C\cdot \max \pts{\sqrt{\frac{n}{d_2}}, \frac{n}{d_{\infty}}} \cdot \frac{2\sqrt{2}C_g\sqrt{n}y_{\max}}{\nnorm[1]{\blambda}}}\\
        &\leq \nnorm[1]{\blambda}\mts{-\frac{y_{\min}}{2\nnorm[1]{\blambda}} + \frac{1}{2} C\cdot \frac{y_{\min}}{C_0y_{\max}} \cdot \frac{2\sqrt{2}C_g y_{\max}}{\nnorm[1]{\blambda}}}\\
        &< 0,
    \end{align*}
    where the inequalities apply Corollary~\ref{cor:concentration_op_norm} and the upper bound of $\nnorm[2]{\balpha_{\oplus}^\star}$ in Step 1, and substitute $d_2 \geq C_0^2\frac{n^2 y_{\max}^2}{y_{\min}^2}$ and $d_{\infty} \geq C_0\frac{n^{1.5} y_{\max}}{y_{\min}}$ in Assumption~\ref{asm:high_dim_data} with $C_0\gtrsim C_{\alpha}^2$ and $C_{\alpha}\gtrsim \max\{C_g^2, C_yC_g\}$. However, $\w_{\oplus}^{\star\top}\x_i < 0$ contradicts with the condition $-\w_{\oplus}^{\star\top}\x_i \leq 0$ for $i\in \tilde{S}_4$. Therefore, $\tilde{S}_4=\varnothing$. By combining these two cases, we have $\sum_{i\in S_4} \pts{\w_{\oplus}^{(\infty)\top}\x_i - \w_{\oplus}^{\star\top}\x_i}^2 = \sum_{i\in S_4} \pts{\w_{\oplus}^{(\infty)\top}\x_i}^2 \leq \sum_{i\in S_4} \frac{16}{C_g^2}y_{\max}^2$.
    
    Substituting the upper bounds into Equation~\eqref{eq:multiple_relu_difference_in_prediction} gives us
    \begin{align}
        \nnorm[2]{\X\w^{(\infty)} - \X\w^\star}^2 &= \sum_{i\in S_2} \pts{-\w_{\ominus}^{\star\top}\x_i}^2 + \sum_{i\in S_3} \pts{\w_{\oplus}^{(\infty)\top}\x_i - \w_{\oplus}^{\star\top}\x_i}^2 + \sum_{i\in S_4} \pts{\w_{\oplus}^{(\infty)\top}\x_i - \w_{\oplus}^{\star\top}\x_i}^2\nonumber\\
        &\leq \sum_{i\in S_3} \frac{16}{C_g^2}y_{\max}^2 + \sum_{i\in S_4} \frac{16}{C_g^2}y_{\max}^2\nonumber\\
        &= \frac{16}{C_g^2} n_- y_{\max}^2\label{eq:multiple_relu_prediction_upper}.
    \end{align}
    Finally, putting together Equation~\eqref{eq:multiple_relu_w_star_upper} and~\eqref{eq:multiple_relu_prediction_upper}, we have
    \begin{align*}
        \nnorm[2]{\w^{(\infty)} -\w^\star}^2 \leq \frac{\nnorm[2]{\X\w^{(\infty)} - \X\w^\star}^2}{\mu_{n}(\XX)} \leq \frac{16n_-y_{\max}^2}{C_g\nnorm[1]{\blambda}},
    \end{align*}
    which completes the proof of the upper bound.
    
    \paragraph{Step 4: Lower bound for $\nnorm[2]{\w_{\oplus}^{(\infty)} - \w_{\oplus}^{\star}}$.}~\\
    Now, we derive the lower bound of $\nnorm[2]{\w_{\oplus}^{(\infty)} - \w_{\oplus}^{\star}}$ in a similar approach. We again start with the prediction distance
    \begin{align}
        \nnorm[2]{\X\w_{\oplus}^{(\infty)} - \X\w_{\oplus}^\star}^2 = \nnorm[2]{\X\pts{\w_{\oplus}^{(\infty)} -\w_{\oplus}^\star}}^2&\leq \mu_{1}(\X^\top\X) \nnorm[2]{\w_{\oplus}^{(\infty)} -\w_{\oplus}^\star}^2\nonumber\\ &= \mu_{1}(\XX) \nnorm[2]{\w_{\oplus}^{(\infty)} -\w_{\oplus}^\star}^2.\label{eq:multiple_relu_w_star_lower}
    \end{align}
    Therefore, it suffices to lower bound $\nnorm[2]{\X\w_{\oplus}^{(\infty)} - \X\w_{\oplus}^\star}$ to get the lower bound of $\nnorm[2]{\w_{\oplus}^{(\infty)} - \w_{\oplus}^\star}$. By Equation~\eqref{eq:multiple_relu_difference_in_prediction}, we have 
    \begin{align}
        \nnorm[2]{\X\w_{\oplus}^{(\infty)} - \X\w_{\oplus}^\star}^2 &= \sum_{i\in S_2} \pts{-\w_{\ominus}^{\star\top}\x_i}^2 + \sum_{i\in S_3} \pts{\w_{\oplus}^{(\infty)\top}\x_i - \w_{\oplus}^{\star\top}\x_i}^2 + \sum_{i\in S_4} \pts{\w_{\oplus}^{(\infty)\top}\x_i - \w_{\oplus}^{\star\top}\x_i}^2\nonumber\\
        &\geq \sum_{i\in S_3} \pts{\w_{\oplus}^{(\infty)\top}\x_i - \w_{\oplus}^{\star\top}\x_i}^2 + \sum_{i\in S_4} \pts{\w_{\oplus}^{(\infty)\top}\x_i - \w_{\oplus}^{\star\top}\x_i}^2.\label{eq:multiple_relu_difference_in_prediction2}
    \end{align}
    We omit the partition in $S_2$ because we have shown that $\sum_{i\in S_2} \pts{-\w_{\ominus}^{\star\top}\x_i}^2 =0$ with probability at least $1 - 2\exp(-n(Cc-\ln9))$. Therefore, we need to lower bound $\pts{\w_{\oplus}^{(\infty)\top}\x_i - \w_{\oplus}^{\star\top}\x_i}^2$ for $i\in S_3$ and $\pts{\w_{\oplus}^{(\infty)\top}\x_i - \w_{\oplus}^{\star\top}\x_i}^2$ for $i\in S_4$. 
    
    We start with lower bounding $\pts{\w_{\oplus}^{(\infty)\top}\x_i - \w_{\oplus}^{\star\top}\x_i}^2$ for $i\in S_3$. We know that $\w_{\oplus}^{(\infty)\top}\x_i < 0$ in Theorem~\ref{thm:multiple_relu_gd_high_dim_implicit_bias}. For $\w_{\oplus}^{\star\top}\x_i$ with $i\in S_3$, by the complementary slackness, we either have $\w_{\oplus}^{\star\top}\x_i = 0$ with $\mu_i^\star \geq 0$ or $\mu_i^\star = 0$ with $\w_{\oplus}^{\star\top}\x_i \leq 0$. In the first case, $\w_{\oplus}^{\star\top}\x_i = 0$, we have
    \begin{align*}
        \w_{\oplus}^{(\infty)\top}\x_i - \w_{\oplus}^{\star\top}\x_i &= \w_{\oplus}^{(\infty)\top}\x_i\\
        &= \e_i^\top\XX\balpha_{\oplus}^{(\infty)}\\
        &= \e_i^\top\mts{\nnorm[1]{\blambda}\I + \pts{\XX - \nnorm[1]{\blambda}\I}}\balpha_{\oplus}^{(\infty)}\\
        &= \nnorm[1]{\blambda}\alpha_{\oplus,i}^{(\infty)} + \e_i^\top\pts{\XX - \nnorm[1]{\blambda}\I}\balpha_{\oplus}^{(\infty)}\\
        &\leq \nnorm[1]{\blambda}\alpha_{\oplus,i}^{(\infty)} + \nnorm[2]{\XX - \nnorm[1]{\blambda}\I}\nnorm[2]{\balpha_{\oplus}^{(\infty)}}\\
        &\leq \nnorm[1]{\blambda}\mts{\alpha_{\oplus,i}^{(\infty)} + C\cdot \max \pts{\sqrt{\frac{n}{d_2}}, \frac{n}{d_{\infty}}}\nnorm[2]{\balpha_{\oplus}^{(\infty)}}},
    \end{align*}
    where the last inequality applies Corollary~\ref{cor:concentration_op_norm}. Substituting the bounds of $\alpha_{\oplus,i}^{(\infty)}$ and $\nnorm[2]{\balpha_{\oplus}^{(\infty)}}$ from Lemma~\ref{lem:primal_label_same_sign_two}, we have
    \begin{align*}
        \w_{\oplus}^{(\infty)\top}\x_i - \w_{\oplus}^{\star\top}\x_i &\leq \nnorm[1]{\blambda}\mts{-\frac{y_{\min}}{C_{\alpha}\nnorm[1]{\blambda}} + C\cdot \max \pts{\sqrt{\frac{n}{d_2}}, \frac{n}{d_{\infty}}}\frac{C_{\alpha}\sqrt{n}y_{\max}}{\nnorm[1]{\blambda}}}\\
        &\leq \nnorm[1]{\blambda}\mts{-\frac{y_{\min}}{C_{\alpha}\nnorm[1]{\blambda}} + C\cdot \frac{y_{\min}}{C_0 y_{\max}}\frac{C_{\alpha}y_{\max}}{\nnorm[1]{\blambda}}}\\
        &\leq -(1-\frac{C\cdot C_{\alpha}^2}{C_0})\frac{y_{\min}}{C_{\alpha}},
    \end{align*}
    where the inequalities substitute $d_2 \geq C_0^2\frac{n^2 y_{\max}^2}{y_{\min}^2}$ and $d_{\infty} \geq C_0\frac{n^{1.5} y_{\max}}{y_{\min}}$ in Assumption~\ref{asm:high_dim_data} with $C_0\gtrsim C_{\alpha}^2$. In the second case, $\alpha_{\oplus,i}^\star = -\mu_i^\star = 0$ for $i\in S_3$, we have
    \begin{align*}
        \w_{\oplus}^{(\infty)\top}\x_i - \w_{\oplus}^{\star\top}\x_i &= \e_i^\top\XX\pts{\balpha_{\oplus}^{(\infty)} -\balpha_{\oplus}^\star}\\
        &= \e_i^\top\mts{\nnorm[1]{\blambda}\I + \pts{\XX - \nnorm[1]{\blambda}\I}}\pts{\balpha_{\oplus}^{(\infty)} -\balpha_{\oplus}^\star}\\
        &= \nnorm[1]{\blambda}\alpha_{\oplus,i}^{(\infty)} + \e_i^\top\pts{\XX - \nnorm[1]{\blambda}\I}\pts{\balpha_{\oplus}^{(\infty)} -\balpha_{\oplus}^\star}\\
        &\leq \nnorm[1]{\blambda}\alpha_{\oplus,i}^{(\infty)} + \nnorm[2]{\XX - \nnorm[1]{\blambda}\I}\pts{\nnorm[2]{\balpha_{\oplus}^{(\infty)}} + \nnorm[2]{\balpha_{\oplus}^\star}}\\
        &\leq \nnorm[1]{\blambda}\mts{-\frac{y_{\min}}{C_{\alpha}\nnorm[1]{\blambda}} + C\cdot \max \pts{\sqrt{\frac{n}{d_2}}, \frac{n}{d_{\infty}}}\pts{\frac{C_{\alpha}\sqrt{n}y_{\max}}{\nnorm[1]{\blambda}} + \frac{\sqrt{2}C_g\sqrt{n}y_{\max}}{\nnorm[1]{\blambda}}}}\\
        &\leq \nnorm[1]{\blambda}\mts{-\frac{y_{\min}}{C_{\alpha}\nnorm[1]{\blambda}} + C\cdot \frac{y_{\min}}{C_0 y_{\max}}\pts{\frac{C_{\alpha}y_{\max}}{\nnorm[1]{\blambda}} + \frac{\sqrt{2}C_gy_{\max}}{\nnorm[1]{\blambda}}}}\\
        &\leq -(1-\frac{2C\cdot C_{\alpha}^2}{C_0})\frac{y_{\min}}{C_{\alpha}},
    \end{align*}
    by applying the same argument and the upper bound from Step 1 that $\nnorm[2]{\balpha_{\oplus}^\star} \leq \frac{\sqrt{2}C_g\sqrt{n}y_{\max}}{\nnorm[1]{\blambda}}$. Therefore, we have $\pts{\w_{\oplus}^{(\infty)\top}\x_i - \w_{\oplus}^{\star\top}\x_i}^2 \geq \pts{1 -\frac{2C\cdot C_{\alpha}^2}{C_0}}^2 \frac{ y_{\min}^2}{C_{\alpha}^2}$, for all $i\in S_3$.
    
    Next, we lower bound the term $\pts{\w_{\oplus}^{(\infty)\top}\x_i - \w_{\oplus}^{\star\top}\x_i}^2$ for all $i\in S_4$. In Step 3, we already showed the two cases in $\w_{\oplus}^{\star\top}\x_i$ with $i\in S_4$  by the complementary slackness. In the first case, $(-\w_{\ominus}^{\star\top}\x_i)^2 = 0$, and we have $\pts{\w_{\oplus}^{(\infty)\top}\x_i - \w_{\oplus}^{\star\top}\x_i}^2 = \pts{\w_{\oplus}^{(\infty)\top}\x_i}^2$. Therefore, we can reuse the lower bound we derived in $S_3$ such that $\w_{\oplus}^{(\infty)\top}\x_i \leq -(1-\frac{C\cdot C_{\alpha}^2}{C_0})\frac{y_{\min}}{C_{\alpha}}$. In the second case, we have shown that $\tilde{S}_4 =\varnothing$. By concluding two cases, we have $\sum_{i\in S_4} \pts{\w_{\oplus}^{(\infty)\top}\x_i - \w_{\oplus}^{\star\top}\x_i}^2 = \sum_{i\in S_4} \pts{\w_{\oplus}^{(\infty)\top}\x_i}^2 \geq \sum_{i\in S_4}\pts{1 -\frac{2C\cdot C_{\alpha}^2}{C_0}}^2 \frac{ y_{\min}^2}{C_{\alpha}^2}$.
    
    Substituting the lower bounds into Equation~\eqref{eq:multiple_relu_difference_in_prediction2} gives us
    \begin{align}
        \nnorm[2]{\X\w_{\oplus}^{(\infty)} - \X\w_{\oplus}^\star}^2 &\geq \sum_{i\in S_3} \pts{\w_{\oplus}^{(\infty)\top}\x_i - \w_{\oplus}^{\star\top}\x_i}^2 + \sum_{i\in S_4} \pts{\w_{\oplus}^{(\infty)\top}\x_i - \w_{\oplus}^{\star\top}\x_i}^2\nonumber\\
        &\geq \sum_{i\in S_3} \pts{1 -\frac{2C\cdot C_{\alpha}^2}{C_0}}^2 \frac{ y_{\min}^2}{C_{\alpha}^2} + \sum_{i\in S_4} \pts{1 -\frac{2C\cdot C_{\alpha}^2}{C_0}}^2 \frac{ y_{\min}^2}{C_{\alpha}^2}\nonumber\\
        &= \frac{n_- y_{\min}^2}{\tilde{C}}\label{eq:multiple_relu_prediction_lower},
    \end{align}
    where we let $\tilde{C} \coloneqq \frac{C_0^2C_{\alpha}^2}{\pts{C_0 -2C\cdot C_{\alpha}^2}^2} > 1$. Finally, putting together Equation~\eqref{eq:multiple_relu_w_star_lower} and~\eqref{eq:multiple_relu_prediction_lower}, we have
    \begin{align*}
        \nnorm[2]{\w_{\oplus}^{(\infty)} -\w_{\oplus}^\star}^2 \geq \frac{\nnorm[2]{\X\w_{\oplus}^{(\infty)} - \X\w_{\oplus}^\star}^2}{\mu_{1}(\XX)} \geq \frac{n_-y_{\min}^2}{\tilde{C}C_g\nnorm[1]{\blambda}}.
    \end{align*}
    This completes the proof of the lower bound.
\end{proof}
\newpage

\section{Implicit Bias of Multiple ReLU Models (\texorpdfstring{$m>2$}{m>2}) Under Gradient Descent }\label{sec:multiple_relu_gd_main}

In this section, we extend our analysis to multiple ReLU models trained with $m>2$ neurons under stronger assumptions on the initialization. We consider models of the form:  $h_{\bTheta}(\x) \coloneqq h_{\{ \w_k\}_{k=1}^m}(\x) = \sum_{k=1}^m s_k\sigma(\w_k^\top \x)$, where $\w_k\in\R^d$ are the model weights and there are at least one positive neuron and one negative neuron. The parameter set is hence denoted by $\bTheta = \{\w_k\}_{k=1}^m$. The empirical risk is defined in~\eqref{eq:empirical_risk} as
\begin{align*}
    \Risk(\bTheta) = \frac{1}{2} \sum_{i=1}^n \bigl(h_{\bTheta}(\x_i) - y_i\bigr)^2 = \frac{1}{2} \sum_{i=1}^n \pts{\sum_{k=1}^m s_k\sigma(\w_k^\top \x_i)- y_i}^2.
\end{align*}
Here, we fix $s_k\in\{\pm 1\}$ and only train the hidden weights $\{\w_k\}_{k=1}^m$.

\subsection{Gradient Descent Updates and Convergence}

The gradient of the empirical risk in~\eqref{eq:empirical_risk} with respect to $\w_k$ is given in~\eqref{eq:gd_update_wk} as
\begin{align}
    \w_k^{(t+1)}
    &= \w_k^{(t)} - \eta \nabla_{\w_k} \Risk(\bTheta^{(t)}) = \w_k^{(t)} - \eta s_k\X^\top \D(\X\w_k^{(t)})\bigl(h_{\bTheta^{(t)}}(\X) - \y\bigr).\label{eq:gd_update_m}
\end{align}
The primal-dual gradient update in~\eqref{eq:primal_dual_update} is given by
\begin{subequations}\label{eq:primal_dual_update2}
\begin{alignat}{2}
    &\text{(Primal) }\myquad[5] &&\bbeta_k^{(t+1)} = \bbeta_k^{(t)} - \eta s_k\X \X^\top \D(\bbeta_k^{(t)})(h_{\bTheta^{(t)}}(\X) - \y),\myquad[6]\label{eq:primal_update2} \\
    &\text{(Dual) }\myquad[5] &&\balpha_k^{(t+1)} = \balpha_k^{(t)} - \eta s_k\D(\bbeta^{(t)})(h_{\bTheta^{(t)}}(\X) - \y)\label{eq:dual_update2}.\myquad[6]
\end{alignat}
\end{subequations}
Next, we consider a regime in which, after some time $t_0$, each neuron activates on a fixed subset of training examples, and this activation pattern remains unchanged throughout the subsequent dynamics. Moreover, these active subsets are disjoint across different neurons. That is, for every training example, at most one neuron is active, while each neuron may be active on a subset of examples. In this regime, each neuron effectively reduces to a linear model trained only on its own active examples. 
\begin{lemma} \label{lem:multiple_relu_m_reduce_linear}
    Consider a multiple ReLU model $h_{\bTheta}$. For each neuron $k\in[m]$, suppose there exists iteration $t_0 \ge 0$ such that 
    \begin{enumerate}
        \item At time $t_0$, the subset of examples on which the $k$-th neuron is active is disjoint from the subsets activated by all other neurons, i.e., $\D(\X\w_k^{(t_0)}) \D(\X\w_{\ell}^{(t_0)}) = \zero_{n\times n}$ for any $\ell\neq k$.
        \item The activation pattern of the $k$-th remains unchanged after time $t_0$, i.e., $\D(\X\w_k^{(t_0)})=\D(\X\w_k^{(t)})$ for all $t \geq t_0$.
    \end{enumerate}
    Then, for all $t \geq t_0$, and each $k\in[m]$, the gradient descent dynamics of the $k$-th neuron are equivalent to gradient descent applied to a linear model, initialized at $\w_k^{(t_0)}$, and trained using only the subset of samples satisfying  $\x_i^\top \w_k^{(t_0)} > 0$.
\end{lemma}
The proof of Lemma~\ref{lem:multiple_relu_m_reduce_linear} is provided in Appendix~\ref{app:proof_multiple_relu_m_reduce_linear}.


\subsection{Minimum-\texorpdfstring{$\ell_2$}{l2}-norm Solution of Multiple ReLU Models}

The minimum-$\ell_2$-norm solution for the multiple ReLU regression in~\eqref{eq:w_star_def} is given by
\begin{align}\label{eq:multiple_relu_m_minimum_norm_sol}
    &\{\w_k^\star\}_{k=1}^m = {} \uargmin{\{\w_k\}_{k=1}^m}\frac{1}{2}\sum_{k=1}^m \nnorm[2]{\w_k}^2\\
    \text{s.t. } &\sum_{k=1}^m s_k\sigma(\w_k^\top \x_i) = y_i, \text{ for all } i\in[n].\nonumber
\end{align}

\subsection{High-dimensional Implicit Bias of Multiple ReLU Models}

In this section, we characterize the implicit bias of multiple ReLU models trained by gradient descent in the high-dimensional regime. We identify a setup in which each neuron is only active toward a fixed and disjoint subset of training examples, where the labels $y_i$ of these examples have the same sign as the neuron’s sign $s_k$.

To formalize this setup, we introduce an assignment vector $\ba \in[m]^n$, where each entry $a_i \in [m]$ indicates which neuron is responsible for example $i$.
\begin{assumption}\label{asm:disjoint_subset}
    For a multiple ReLU model, we assume that there exists an assignment vector $\ba\in[m]^n$ such that for each example $i\in[n]$, $a_i = k$, for some neuron $k$ satisfying $s_k\cdot y_i > 0$. For each neuron $k\in[m]$, define a diagonal matrix $\A_k\in\R^{n\times n}$ with diagonal entries
    \begin{align*}
        (\A_k)_{ii} =\cond{0,}{\text{ if } a_i = k, \text{ or } s_k\cdot y_i < 0}{-\mathrm{sign}(y_i),}{\text{ otherwise}}.
    \end{align*}
\end{assumption}
Assumption~\ref{asm:disjoint_subset} is used to design a proper initialization that ensures that the gradient descent can converge to the desired regime.
In this regime, we show that if a neuron’s primal variable $\beta_{k, i}$ is positive and the sign of the neuron agrees with the label (i.e., $s_k\cdot y_i > 0$), then the corresponding example remains active throughout training. Conversely, if the associated dual variable $\alpha_{k,j}$ stays sufficiently negative, it remains frozen and is no longer updated.

\begin{theorem}\label{thm:multiple_relu_gd_m_high_dim_implicit_bias}
    Under Assumptions~\ref{asm:label_bounds},~\ref{asm:high_dim_data} and~\ref{asm:disjoint_subset}, suppose we choose initialization\\ $\w_k^{(0)} = \X^\top(\XX)^{-1}\pts{\frac{1}{C_g}\A_k\y +\bepsilon_k}$, where $0 < \epsilon_{k,i} \leq \frac{1}{C_{\alpha}m}y_{\min}$ for all $k\in[m]$ and $i\in[n]$, and the gradient descent step size satisfies $\frac{1}{CC_g\nnorm[1]{\blambda}}\leq \eta \leq \frac{1}{C_g\nnorm[1]{\blambda}}$. Then, the gradient descent limit $\w_k^{(\infty)}$ for multiple ReLU models coincides with the solution obtained by training a linear model on disjoint subsets of examples, initialized at $\w_k^{(1)}$ with probability at least $1 - 2 \exp(-
    cn)$. Formally, we have $\w_k^{(\infty)} =\uargmin{\w\in\{\w:\X_{S_k}\w =\y_{S_k}\}} \nnorm[2]{\w - \w_k^{(1)}}$ and $\X_{S_k^{\mathsf{c}}} \w_k^{(\infty)} \preceq \zero$, where $S_k\coloneqq\{i\in[n]:a_i=k\}$.
\end{theorem}
The full proof is provided in Appendix~\ref{app:proof_multiple_relu_gd_m_high_dim_implicit_bias}. Note that the initialization, constructed by the matrices $\A_k$, ensures that each training example $i$ is activated by exactly one neuron that matches its sign—namely, the $a_i$-th neuron. All other neurons with the same sign remain inactive on this example.

\subsection{Approximation to Minimum-\texorpdfstring{$\ell_2$}{l2}-norm Solution in High Dimensions}\label{sec:multiple_relu_m_approx_w_star}

In this section, we show that in high dimensions, the implicit bias solution for multiple ReLU models derived in Theorem~\ref{thm:multiple_relu_gd_m_high_dim_implicit_bias} is close to the corresponding minimum-$\ell_2$-norm solution $\{\w_k^\star\}_{k=1}^m$ defined in~\eqref{eq:multiple_relu_m_minimum_norm_sol}.

\begin{theorem}\label{thm:multiple_relu_m_approx_to_w_star}
    Under Assumptions~\ref{asm:label_bounds},~\ref{asm:high_dim_data} and~\ref{asm:disjoint_subset}, suppose we choose initialization\\ $\w_k^{(0)} = \X^\top(\XX)^{-1}\pts{\frac{1}{C_g}\A_k\y + \bepsilon_k}$, where $0 < \epsilon_{k,i} \leq \frac{1}{C_{\alpha}m}y_{\min}$ for all $k\in[m]$ and $i\in[n]$, and the gradient descent step size satisfies $\frac{1}{CC_g\nnorm[1]{\blambda}}\leq\eta \leq \frac{1}{C_g\nnorm[1]{\blambda}}$. Then, we have\\ $\sqrt{\sum_{k=1}^m\nnorm[2]{\w_{k}^{(\infty)} - \w_{k}^{\star}}^2} \leq \sqrt{\frac{4C_gC_{\alpha}^2mny_{\max}^2}{\nnorm[1]{\blambda}}}$ with probability at least $1 - 2 \exp(-
    cn)$. 
\end{theorem}
The proof is deferred to Appendix~\ref{app:proof_multiple_relu_m_approx_to_w_star}. Note that since the minimum-$\ell_2$-norm solution $\{\w_k^\star\}_{k=1}^m$ is more involved to characterize, Theorem~\ref{thm:multiple_relu_m_approx_to_w_star} only provides an upper bound for the approximation of the implicit bias to $\{\w_k^\star\}_{k=1}^m$. A more fine-grained characterization, as well as a deeper understanding of the role of overparameterization, is left for future work.
\newpage

\section{Proofs for Multiple ReLU Models (\texorpdfstring{$m>2$}{m>2}) Trained with Gradient Descent}\label{app:multiple_relu_gd_m}

In this section, we present the proofs concerning the behavior of the multiple ReLU model trained with gradient descent.

\subsection{Proof of Lemma~\ref{lem:multiple_relu_m_reduce_linear} (Gradient Descent Convergence)}\label{app:proof_multiple_relu_m_reduce_linear}
\begin{proof}(Lemma~\ref{lem:multiple_relu_m_reduce_linear})
    This proof is analogous to Lemma~\ref{lem:final_phase}. The key idea is to show that once the activation pattern becomes fixed after some iteration $t_0 \geq 0$, the gradient descent dynamics of each neuron are equivalent to those of a linear model trained on a fixed subset of examples.
    
    Fix a neuron $k\in[m]$. Consider the linear model
    \begin{align*}
        h(\x) =  s_k\w^\top\x,
    \end{align*}
    where $\w \in\R^d$ is the linear model parameter (also called weight). Let $S_k^{(t_0)}\subseteq [n]$ denote the active set of the $k$-th neuron at iteration $t_0$, defined by $S_k^{(t_0)} \coloneqq \{ i\in [n] : \x_i^\top\w_k^{(t_0)} > 0\}$. We define the empirical risk with the linear model using only the examples in $S_k^{(t_0)}$ as
    \begin{align*}
        \Risk_{S_k^{(t_0)}}(\w) = \frac{1}{2}\sum_{i\in S_k^{(t_0)}} ( s_k\w^\top\x_i - y_i)^2.
    \end{align*}
    The gradient descent update for this linear model is then given by
    \begin{align}
        \w^{(t+1)} &= \w^{(t)} - \eta\nabla\Risk_{S_k^{(t_0)}}(\w^{(t)})\nonumber\\
        &= \w^{(t)} -\eta s_k\sum_{i\in S_k^{(t_0)}} ( s_k\w^{(t)\top}\x_i - y_i)\x_i\label{eq:final_phase_linear_m}.
    \end{align}
    On the other hand, the original gradient descent update of the multiple ReLU model in Equation~\eqref{eq:gd_update_m} tells us that
    \begin{align*}
       \w_k^{(t+1)} = \w_k^{(t)} - \eta s_k\X^\top \D(\X\w_k^{(t)})\bigl(h_{\bTheta^{(t)}}(\X) - \y\bigr).
    \end{align*}
    Under the first assumption in the lemma, $\D(\X\w_k^{(t_0)}) \D(\X\w_{\ell}^{(t_0)}) = \zero_{n\times n}$ for any $\ell\neq k$, the activation patterns of different neurons are disjoint at iteration $t_0$. Consequently, for any $i\in S_k^{(t_0)}$, only the $k$-th neuron is active, and therefore we have
    \begin{align*}
        h_{\bTheta^{(t_0)}}(\x_i) = \sum_{k=1}^m s_k\sigma(\w_k^\top\x_i) = s_k\w_k^\top\x_i.
    \end{align*}
    Moreover, by the second assumption of the lemma, the activation pattern of the $k$-th neuron remains unchanged after iteration $t_0$, i.e., $\D(\X\w_k^{(t_0)})=\D(\X\w_k^{(t)})$ for all $t \geq t_0$. Hence, for all $t\geq t_0$, the diagonal entries of $\D(\X\w_k^{(t)})$ satisfy $D_{ii} = \mathbbm{1}_{i\in S_k^{(t_0)}}$ for all $i\in[n]$. Therefore, for all $t\geq t_0$, the gradient update of the $k$-th neuron in the multiple ReLU model is given by
    \begin{align*}
        \w_k^{(t+1)} &= \w_k^{(t)} - \eta s_k\X^\top \D(\X\w_k^{(t)})(h_{\bTheta^{(t)}}(\X) - \y)\\
        &=\w^{(t)} -\eta s_k\sum_{i\in S_k^{(t_0)}} (s_k\w_k^{(t)\top}\x_i - y_i)\x_i.
    \end{align*}
    This update is identical to the gradient descent update of the linear model in Equation~\eqref{eq:final_phase_linear_m}. Hence, for all $t \geq t_0$, the gradient descent dynamics of the $k$-th neuron in the multiple ReLU model are equivalent to those of a linear model trained using only the examples in $S_k^{(t_0)}$. This completes the proof of the lemma.
\end{proof}

\newpage

\subsection{Proof of Theorem~\ref{thm:multiple_relu_gd_m_high_dim_implicit_bias} (High-dimensional Implicit Bias)}\label{app:proof_multiple_relu_gd_m_high_dim_implicit_bias}
In this section, we present the proof of Theorem~\ref{thm:multiple_relu_gd_m_high_dim_implicit_bias}. Before proceeding to the proof, we again introduce a set of sufficient conditions under which the active pattern for a neuron at iteration $t$ will be preserved at iteration $t+1$. Similar to the single ReLU model and $2$-ReLU model cases, our analysis relies on Lemma~\ref{lem:pos_stay_pos} and Lemma~\ref{lem:neg_stay_neg} to characterize the dynamics of primal and dual variables. Using these results, we establish Lemma~\ref{lem:primal_label_same_sign_m}, which characterizes that the active sets of all neurons remain unchanged across gradient descent iterations.

\begin{lemma}\label{lem:primal_label_same_sign_m}
    Under Assumption~\ref{asm:label_bounds},~\ref{asm:high_dim_data} and~\ref{asm:disjoint_subset}, suppose the gradient descent step size satisfies $\eta \leq \frac{1}{C_g\nnorm[1]{\blambda}}$. For a multiple ReLU model, if the following five conditions hold at some iteration $t\geq 0$, then they also hold at iteration $t+1$. 
    \begin{enumerate}[label=\alph*.,ref=\alph*]
        \item \label{cond:beta_plus_pos_m} $\beta_{a_i,i}^{(t)} > 0$, for all $i \in [n]$.
        \item \label{cond:alpha_plus_bounds_m} $-\frac{3y_{\max}}{C_g\nnorm[1]{\blambda}} \leq \alpha_{k,j}^{(t)} \leq -\frac{y_{\min}}{C_{\alpha}\nnorm[1]{\blambda}}$, for all $j \in [n]$ with $k \neq a_j$.
        \item \label{cond:sigma_beta_bounds_m} $\nnorm[2]{\bbeta_{k,S_k}^{(t)} - s_k\y_{S_k}} \leq C_y\nnorm[2]{\y_{S_k}}$, for all $k\in[m]$.
        \item \label{cond:alpha_norm_bounds_m} $\nnorm[2]{\balpha_{k}^{(t)}} \leq \frac{C_{\alpha}\sqrt{n}y_{\max}}{\nnorm[1]{\blambda}}$, for all $k\in[m]$.
        \item \label{cond:beta_plus_neg_m} $\beta_{k,j}^{(t)} \leq 0$, for all $j \in [n]$ with $k\neq a_j$.
    \end{enumerate}
    Consequently, the activation pattern of each neuron remains unchanged from iteration $t$ to $t+1$.
    In the above, we define $S_k \coloneqq \{ i\in [n] : a_i = k\}$, and for any vector $\bv\in\R^n$, we use $\bv_{S_k}$ to denote the subvector of entries indexed by $S_k$.
\end{lemma}

\begin{proof}(Lemma~\ref{lem:primal_label_same_sign_m})
    We now verify that these conditions are preserved from iteration $t$ to $t+1$.

\begin{enumerate}[label=Part (\alph*):,leftmargin=3\parindent]
    \item By conditions~\eqref{cond:beta_plus_pos_m},~\eqref{cond:sigma_beta_bounds_m} and~\eqref{cond:beta_plus_neg_m} at iteration $t$, we have
    \begin{align*}
        \nnorm[2]{h_{\bTheta^{(t)}}(\X) - \y}^2 &= \nnorm[2]{\sum_{k=1}^m s_k\sigma(\bbeta_{k}^{(t)}) - \y}^2 = \sum_{k=1}^m\nnorm[2]{ s_k\pts{\bbeta_{k,S_k}^{(t)} - s_k\y_{S_k}}}^2 \leq C_y^2\nnorm[2]{\y}^2,
    \end{align*}
    where the last inequality uses the fact that the sets $\{S_k\}_{k=1}^m$ are disjoint.
    Also, we have $s_{a_i} h_{\bTheta^{(t)}}(\x_i) = \beta_{a_i,i}^{(t)}$ from conditions~\eqref{cond:beta_plus_pos_m} and~\eqref{cond:beta_plus_neg_m}. Together with condition~\eqref{cond:beta_plus_pos}, the assumptions of Lemma~\ref{lem:pos_stay_pos} are satisfied for all $i\in [n]$. Consequently, $\beta_{a_i,i}^{(t+1)} > 0$ for all $i \in [n]$.
    
    \item According to the dual gradient update in Equation~\eqref{eq:dual_update2}, and using condition~\eqref{cond:beta_plus_neg_m} at iteration $t$, we have:
    \[
    \alpha_{k,j}^{(t+1)} = \alpha_{k,j}^{(t)} \quad \text{for all } j \in [n] \text{ with } k\neq a_j.
    \]
    Therefore, condition (\ref{cond:alpha_plus_bounds_m}) continues to hold at iteration $t + 1$.
    
    \item By conditions (\ref{cond:beta_plus_pos_m}) and (\ref{cond:beta_plus_neg_m}), the gradient update at iteration $t$ for $\bbeta_{k}^{(t)}$ depends only on the examples in the subset $S_k$. Hence, the gradient update for an individual neuron is equivalent to a linear regression gradient descent. As similarly argued in the proof of Lemma~\ref{lem:conv_step_size}, since the step size satisfies $\eta \leq \frac{1}{C_g\nnorm[1]{\blambda}}$, the linear regression squared loss is monotonically nonincreasing, and by condition~\eqref{cond:sigma_beta_bounds_m} at iteration $t$, we obtain 
    \begin{align*}
        \nnorm[2]{\bbeta_{k,S_k}^{(t+1)} - s_k\y_{S_k}} &\leq \nnorm[2]{\bbeta_{k,S_k}^{(t)} - s_k\y_{S_k}}\leq C_y\nnorm[2]{\y_{S_k}}.
    \end{align*}
    Therefore, condition (\ref{cond:sigma_beta_bounds_m}) holds at iteration $t + 1$.
    
    \item Following the same argument as in Part (\ref{eq:cond4}) of Lemma \ref{lem:primal_label_same_sign}, using conditions (\ref{cond:alpha_plus_bounds_m}) and (\ref{cond:sigma_beta_bounds_m}) at iteration $t + 1$, together with the eigenvalue bounds from Lemma \ref{lem:concentration_eigenvalues}, we can establish that
    \[
    \nnorm[2]{\balpha_{k}^{(t+1)}} \leq \frac{C_{\alpha}\sqrt{n}y_{\max}}{\nnorm[1]{\blambda}} \text{ for all } k\in[m],
    \]
    with probability at least $1 -2e^{-n/C_g}$. Thus, condition (\ref{cond:alpha_norm_bounds_m}) holds at iteration $t + 1$.
    
    \item By Lemma~\ref{lem:neg_stay_neg}, since conditions (\ref{cond:alpha_plus_bounds_m}) and (\ref{cond:alpha_norm_bounds_m}) hold at iteration $t + 1$, we conclude that $\beta_{k,j}^{(t+1)} \leq 0$ for all $j \in [n]$ with $k \neq a_j$. Thus, condition (\ref{cond:beta_plus_neg_m}) holds at iteration $t + 1$.
    
\end{enumerate}
\end{proof}

Equipped with Lemma~\ref{lem:primal_label_same_sign_m}, we are ready to prove Theorem~\ref{thm:multiple_relu_gd_m_high_dim_implicit_bias}.

\begin{proof}(Theorem~\ref{thm:multiple_relu_gd_m_high_dim_implicit_bias})
    The proof follows a similar structure to that of Theorem~\ref{thm:single_relu_gd_high_dim_implicit_bias} for single ReLU models, but now we must track the dynamics of all the neurons $\{\w_k\}_{k=1}^m$ simultaneously. Equipped with sufficient conditions under which the activation patterns are preserved in Lemma~\ref{lem:primal_label_same_sign_m}, we verify these conditions hold after the first gradient step, and use induction to characterize the full gradient descent dynamics.
    
We first verify that the iterate at $t = 1$ satisfies all the sufficient conditions. With the initialization
\begin{align*}
    \w_k^{(0)} =\X^\top\XXi\pts{\frac{1}{C_g}\A_k\y + \bepsilon_k},
\end{align*}
we have $\bbeta_k^{(0)} = \frac{1}{C_g}\A_k\y + \bepsilon_k$. Recalling the definition of $\A_k$ in Assumption~\ref{asm:disjoint_subset}, we have 
\begin{align}
    \bbeta_{k,i}^{(0)} = \cond{\epsilon_{a_i,i},}{\text{ if } a_i = k, \text{ or } s_k\cdot y_i < 0}{-\frac{|y_i|}{C_g} + \epsilon_{k,i},}{\text{ otherwise}},\label{eq:beta_zero}
\end{align}
for all $k\in[m]$ and $i\in[n]$. Since the theorem assumption ensures $\epsilon_{k,i} \leq \frac{1}{C_{\alpha}m}y_{\min}$ and $C_{\alpha} \gtrsim C_g^2$, we have $-\frac{|y_i|}{C_g} + \epsilon_{k,i} < 0$. Therefore, we obtain
\begin{align}
    h_{\bTheta^{(0)}}(\x_i) = \sum_{k=1}^m s_k\sigma\pts{\bbeta_{k,i}^{(0)}} = s_{a_i}\epsilon_{a_i,i} - s_{a_i} \sum_{k:s_k\cdot y_i < 0}\epsilon_{k,i},\label{eq:h_zero}
\end{align}
for all $i \in [n]$. Therefore, using the primal gradient update in Equation~\eqref{eq:primal_update2}, we obtain
    \begin{align}
        \bbeta_k^{(1)} &= \bbeta_k^{(0)} - \eta s_k\X \X^\top \D(\bbeta_k^{(0)})(h_{\bTheta^{(0)}}(\X) - \y)\nonumber\\
        &= \X \X^\top\mts{\underbrace{\eta\pts{s_k\D(\bbeta_k^{(0)})\pts{\y-h_{\bTheta^{(0)}}(\X)}+\frac{1}{\eta}\XXi\bbeta_k^{(0)}}}_{\eqqcolon\balpha_k^{(1)}}},\label{eq:alpha_one}
    \end{align}
    according to the primal-dual formulation $\bbeta_k^{(1)} =\XX\balpha_k^{(1)}$ in Equation~\eqref{eq:primal_dual_def}. In the below, we show that at iteration $t=1$, the variables $\bbeta_k^{(1)}$ and $\balpha_k^{(1)}$ satisfy all the conditions in Lemma~\ref{lem:primal_label_same_sign_m}.
\begin{enumerate}[label=Part (\alph*):,leftmargin=3\parindent]
    \item For all $i \in [n]$, we show that $\beta_{a_i,i}^{(1)} > 0$ by applying Lemma~\ref{lem:pos_stay_pos}. According to Equation~\eqref{eq:beta_zero} and Equation~\eqref{eq:h_zero}, we have $\beta_{a_i,i}^{(0)} = \epsilon_{a_i,i} > 0$ and $s_{a_i}\cdot h_{\bTheta^{(0)}}(\x_i) = \epsilon_{a_i,i} -\sum_{k:s_k\cdot y_i < 0} \epsilon_{k,i} \leq \beta_{a_i,i}^{(0)}$. Moreover, we have
    \begin{align*}
        \nnorm[2]{h_{\bTheta^{(0)}}(\X) - \y} \leq  \nnorm[2]{h_{\bTheta^{(0)}}(\X)} + \nnorm[2]{\y} \leq \sum_{k=1}^m \nnorm[2]{\bepsilon_k} + \nnorm[2]{\y} &\leq \frac{\sqrt{n}}{C_{\alpha}}y_{\min} + \nnorm[2]{\y} \leq C_y\nnorm[2]{\y},
    \end{align*}
    with $C_y\geq 1 + \frac{1}{C_{\alpha}}$. All the conditions of Lemma~\ref{lem:pos_stay_pos} are satisfied, and therefore, $\beta_{a_i,i}^{(1)}>0$ for all $i\in[n]$.

    \item For all $j \in [n]$ with $k \neq a_j$, we verify that $\alpha_{k,j}^{(1)}$ satisfies the required upper and lower bounds. We need to discuss two cases: 1) $\beta_{k, j}^{(0)} = \epsilon_{k,j} > 0$ with $s_k\cdot y_j < 0$, and 2) $\beta_{k, j}^{(0)} = -\frac{|y_j|}{C_g}+\epsilon_{k,j} < 0$ with $s_k\cdot y_j > 0$.\\
    \textbf{Case 1):} For $\beta_{k, j}^{(0)} = \epsilon_{k,j} > 0$ with $s_k\cdot y_j < 0$, we work from Equation~\eqref{eq:alpha_one} to get
    \begin{align}
        \alpha_{k, j}^{(1)} &= \eta\e_j^\top\pts{s_k\D(\bbeta_k^{(0)})\pts{\y-h_{\bTheta^{(0)}}(\X)}+\frac{1}{\eta}\XXi\bbeta_k^{(0)}}\nonumber\\
        &\stackrel{(\mathrm{i})}{=} \eta\Bigg(-|y_j| - s_k\pts{s_{a_j}\epsilon_{a_j,j} - s_{a_j} \sum_{k:s_k\cdot y_j < 0}\epsilon_{k,j}} + \frac{1}{\eta}\e_j^\top\XXi\pts{\frac{1}{C_g}\A_k\y + \bepsilon_k}\Bigg)\nonumber\\
        &= \eta\pts{-|y_j| +\epsilon_{a_j,j} -  \sum_{k:s_k\cdot y_j < 0}\epsilon_{k,j} + \frac{1}{\eta}\e_j^\top\XXi\pts{\frac{1}{C_g}\A_k\y + \bepsilon_k}}\nonumber\\
        &= \eta\Bigg(-|y_j| +\epsilon_{a_j,j} -  \sum_{k:s_k\cdot y_j < 0}\epsilon_{k,j} + \frac{1}{\eta} \e_j^\top\mts{\frac{1}{\nnorm[1]{\blambda}}\I + \pts{\XXi - \frac{1}{\nnorm[1]{\blambda}}\I}}\pts{\frac{1}{C_g}\A_k\y + \bepsilon_k}\Bigg)\nonumber\\
        &\stackrel{(\mathrm{ii})}{=} \eta\Bigg(-|y_j| +\epsilon_{a_j,j} -  \sum_{k:s_k\cdot y_j < 0}\epsilon_{k,j} + \frac{\epsilon_{k,j}}{\eta\nnorm[1]{\blambda}} + \frac{1}{\eta}\e_j^\top\pts{\XXi - \frac{1}{\nnorm[1]{\blambda}}\I}\pts{\frac{1}{C_g}\A_k\y + \bepsilon_k}\Bigg)\label{eq:c1_eq}
    \end{align}
    where equality (i) substitutes $h_{\bTheta^{(0)}}(\x_j) = s_{a_j}\epsilon_{a_j,j} - s_{a_j} \sum_{k:s_k\cdot y_j < 0}\epsilon_{k,j}$ from Equation~\eqref{eq:beta_zero} and $\bbeta_k^{(0)} = \frac{1}{C_g}\A_k\y + \bepsilon_k$, and equality (ii) applies $\pts{\A_k}_{jj} = 0$ for $k\neq a_j$ with $s_k\cdot y_j < 0$. For the upper bound, we have
    \begin{align}
        \alpha_{k, j}^{(1)} &\leq \eta\pts{-|y_j| +\epsilon_{a_j,j} + \frac{\epsilon_{k,j}}{\eta\nnorm[1]{\blambda}} + \frac{1}{\eta} \nnorm[2]{\XXi - \frac{1}{\nnorm[1]{\blambda}}\I}\nnorm[2]{\frac{1}{C_g}\A_k\y + \bepsilon_k}},\label{eq:c1_upper}
    \end{align}
    by dropping negative terms $-  \sum_{k:s_k\cdot y_j < 0}\epsilon_{k,j}$. Next, by applying the upper bound in Corollary~\ref{cor:concentration_op_norm} and the upper bounds for $\nnorm[2]{\y}$ and $\nnorm[2]{\bepsilon_k}$, we have
    \begin{align*}
        \alpha_{k, j}^{(1)} &\leq \eta\bigg(-|y_j| +\epsilon_{a_j,j} + \frac{\epsilon_{k,j}}{\eta\nnorm[1]{\blambda}} +  \frac{C_g}{\eta\nnorm[1]{\blambda}} C \cdot \max \pts{\sqrt{\frac{n}{d_2}}, \frac{n}{d_{\infty}}}\pts{\frac{\sqrt{n}y_{\max}}{C_g} + \frac{\sqrt{n}}{C_{\alpha}m}y_{\min}}\bigg)\\
        &\leq \eta\pts{-|y_j| +\epsilon_{a_j,j} + \frac{\epsilon_{k,j}}{\eta\nnorm[1]{\blambda}} + \frac{C_g}{\eta\nnorm[1]{\blambda}} C \cdot \frac{y_{\min}}{C_0 y_{\max}}\pts{\frac{y_{\max}}{C_g} + \frac{y_{\min}}{C_{\alpha}m}}},
    \end{align*}
    where the second inequality substitutes $d_2 \geq C_0^2\frac{n^2 y_{\max}^2}{y_{\min}^2}$ and $d_{\infty} \geq C_0\frac{n^{1.5} y_{\max}}{y_{\min}}$ in Assumption~\ref{asm:high_dim_data}. Finally, by using the step size assumption $\frac{1}{\eta} \leq CC_g\nnorm[1]{\blambda}$ and the theorem assumption $\epsilon_{k,j} \leq \frac{1}{C_{\alpha}m}y_{\min}$, we have
    \begin{align*}
        \alpha_{k, j}^{(1)} &\leq \frac{1}{CC_g\nnorm[1]{\blambda}}\pts{-y_{\min} + \frac{y_{\min}}{C_{\alpha}m} + \frac{CC_gy_{\min}}{C_{\alpha}m} +  \frac{2C^2C_g^2}{C_0}y_{\min}} \\
        &\leq -\frac{y_{\min}}{C_{\alpha}\nnorm[1]{\blambda}},
    \end{align*}
    with $C_0\gtrsim C_{\alpha}^2$ and $C_{\alpha}\gtrsim \max\{C_g^2, C_yC_g\}$. For the lower bound, starting from Equation~\eqref{eq:c1_eq}, we have
    \begin{align*}
        \alpha_{k,j}^{(1)} &\geq \eta\pts{-|y_j| - \sum_{k:s_k\cdot y_j < 0}\epsilon_{k,j} - \frac{1}{\eta}\nnorm[2]{\XXi - \frac{1}{\nnorm[1]{\blambda}}\I}\nnorm[2]{\frac{1}{C_g}\A_k\y + \bepsilon_k}}\\
        &\stackrel{(\mathrm{i})}{\geq} \eta\Bigg(-|y_j| - \sum_{k:s_k\cdot y_j < 0}\epsilon_{k,j} - \frac{C_g}{\eta\nnorm[1]{\blambda}} C \cdot \max \pts{\sqrt{\frac{n}{d_2}}, \frac{n}{d_{\infty}}}\pts{\frac{\sqrt{n}y_{\max}}{C_g} + \frac{\sqrt{n}}{C_{\alpha}m}y_{\min}}\Bigg) \\
        &\stackrel{(\mathrm{ii})}{\geq} \frac{1}{C_g\nnorm[1]{\blambda}}\pts{-y_{\max} - \frac{y_{\min}}{C_{\alpha}} - C^2C_g^2 \cdot \frac{y_{\min}}{C_0 y_{\max}}\pts{\frac{y_{\max}}{C_g} + \frac{y_{\min}}{C_{\alpha}m}}} \\
        &\stackrel{(\mathrm{iii})}{\geq} -\frac{3y_{\max}}{C_g\nnorm[1]{\blambda}},
    \end{align*}
    where inequality (i) applies the upper bound in Corollary~\ref{cor:concentration_op_norm} and the upper bounds for $\nnorm[2]{\y}$ and $\nnorm[2]{\bepsilon_k}$, inequalities (ii) applies $d_2 \geq C_0^2\frac{n^2 y_{\max}^2}{y_{\min}^2}$ and $d_{\infty} \geq C_0\frac{n^{1.5} y_{\max}}{y_{\min}}$ in Assumption~\ref{asm:high_dim_data}, and inequality (iii) follows by the constant relationship that $C_0\gtrsim C_{\alpha}^2$ and $C_{\alpha}\gtrsim \max\{C_g^2, C_yC_g\}$. Thus, for $\beta_{k, j}^{(0)} = \epsilon_{k,j} > 0$ with $s_k\cdot y_j < 0$, $\alpha_{k,j}^{(1)}$ satisfies both the required upper and lower bounds.\\
    \textbf{Case 2):} For $\beta_{k, j}^{(0)} = -\frac{|y_j|}{C_g}+\epsilon_{k,j} < 0$ with $s_k\cdot y_j > 0$, we work from Equation~\eqref{eq:alpha_one} to get
    \begin{align*}
        \alpha_{k, j}^{(1)} &= \eta\e_j^\top\pts{s_k\D(\bbeta_k^{(0)})\pts{\y-h_{\bTheta^{(0)}}(\X)}+\frac{1}{\eta}\XXi\bbeta_k^{(0)}}\nonumber\\
        &= \e_j^\top\XXi\pts{\frac{1}{C_g}\A_k\y + \bepsilon_k},
    \end{align*}
    where we substitute $\bbeta_k^{(0)} = \frac{1}{C_g}\A_k\y + \bepsilon_k$ and $\beta_{k, j}^{(0)} = -\frac{1}{C_g}|y_j|+\epsilon_{k,j} < 0$, and this eliminates the first term, since $D_{jj} = 0$. Then, $\alpha_{k,j}^{(1)}$ can further be written as
    \begin{align}
        \alpha_{k,j}^{(1)} &= \e_j^\top\mts{\frac{1}{\nnorm[1]{\blambda}}\I + \pts{\XXi - \frac{1}{\nnorm[1]{\blambda}}\I}}\pts{\frac{1}{C_g}\A_k\y + \bepsilon_k}\nonumber\\
        &= \frac{1}{\nnorm[1]{\blambda}}\pts{-\frac{|y_j|}{C_g} + \epsilon_{k,j}} + \e_j^\top \pts{\XXi - \frac{1}{\nnorm[1]{\blambda}}\I} \pts{\frac{1}{C_g}\A_k\y + \bepsilon_k}.\label{eq:op_norm_rev}
    \end{align}
    For the upper bound, we have
    \begin{align*}
        \alpha_{k,j}^{(1)} &\leq \frac{1}{\nnorm[1]{\blambda}}\pts{-\frac{1}{C_g}|y_j| + \epsilon_{k,j}} + \nnorm[2]{\XXi - \frac{1}{\nnorm[1]{\blambda}}\I} \nnorm[2]{\frac{1}{C_g}\A_k\y + \bepsilon_k}\\
        &\stackrel{(\mathrm{i})}{\leq} \frac{1}{\nnorm[1]{\blambda}}\pts{-\frac{|y_j|}{C_g} + \epsilon_{k,j} + C_gC \cdot \max \pts{\sqrt{\frac{n}{d_2}}, \frac{n}{d_{\infty}}} \pts{\frac{\sqrt{n}y_{\max}}{C_g} + \frac{\sqrt{n}}{C_{\alpha}}y_{\min}}}\\
        &\stackrel{(\mathrm{ii})}{\leq} \frac{1}{\nnorm[1]{\blambda}}\pts{-\frac{y_{\min}}{C_g} + \frac{y_{\min}}{C_{\alpha}m}+ C_gC \cdot \frac{y_{\min}}{C_0 y_{\max}}\pts{\frac{y_{\max}}{C_g} + \frac{1}{C_{\alpha}}y_{\min}}}\\
        &\leq -\frac{y_{\min}}{C_{\alpha}\nnorm[1]{\blambda}},
    \end{align*}
    where inequality (i) applies the upper bound in Corollary~\ref{cor:concentration_op_norm} and the upper bounds for $\nnorm[2]{\y}$ and $\nnorm[2]{\bepsilon_k}$, inequalities (ii) substitutes $d_2 \geq C_0^2\frac{n^2 y_{\max}^2}{y_{\min}^2}$ and $d_{\infty} \geq C_0\frac{n^{1.5} y_{\max}}{y_{\min}}$ in Assumption~\ref{asm:high_dim_data}. The last inequality follows by $C_0\gtrsim C_{\alpha}^2$ and $C_{\alpha}\gtrsim \max\{C_g^2, C_yC_g\}$.
    For the lower bound, we work from Equation~\eqref{eq:op_norm_rev} to get
    \begin{align*}
        \alpha_{k,j}^{(1)} &\geq \frac{1}{\nnorm[1]{\blambda}}\pts{-\frac{|y_j|}{C_g}} - \nnorm[2]{\XXi - \frac{1}{\nnorm[1]{\blambda}}\I} \nnorm[2]{\frac{1}{C_g}\A_k\y + \bepsilon_k}\\
        &\geq \frac{1}{\nnorm[1]{\blambda}}\pts{-\frac{|y_j|}{C_g}- C_gC \cdot \max \pts{\sqrt{\frac{n}{d_2}}, \frac{n}{d_{\infty}}} \pts{\frac{\sqrt{n}y_{\max}}{C_g} + \frac{\sqrt{n}}{C_{\alpha}}y_{\min}}}\\
        &\geq \frac{1}{\nnorm[1]{\blambda}}\pts{-\frac{y_{\max}}{C_g}  - C_gC \cdot \frac{y_{\min}}{C_0 y_{\max}}\pts{\frac{y_{\max}}{C_g} + \frac{y_{\min}}{C_{\alpha}}}}\\
        &\geq -\frac{3y_{\max}}{C_g\nnorm[1]{\blambda}},
    \end{align*}
    by the same argument. Thus, for $\beta_{k, j}^{(0)} = -\frac{|y_j|}{C_g}+\epsilon_{k,j} < 0$ with $s_k\cdot y_j > 0$, $\alpha_{k,j}^{(1)}$ satisfies both the required upper and lower bounds. This completes the proof of this part.

    \item We verify that the primal variables $\bbeta_{k,S_k}^{(1)}$ corresponding to active examples minus $\y_{S_k}$ satisfy the norm bound. Specifically, we show that $\nnorm[2]{\bbeta_{k,S_k}^{(1)} -s_k\y_{S_k}}^2 \leq C_y^2\nnorm[2]{\y}^2$. According to Equation~\eqref{eq:alpha_one}, we have
    \begin{align}
        \nnorm[2]{\bbeta_{k,S_k}^{(1)} -s_k\y_{S_k}}^2 &= \sum_{i:a_i = k} \pts{\beta_{k,i}^{(1)} - s_k y_i}^2\nonumber\\
        &= \sum_{i:a_i = k} \pts{\beta_{k,i}^{(0)} - \eta s_k\e_i^\top\X \X^\top \D(\bbeta_k^{(0)})(h_{\bTheta^{(0)}}(\X) - \y) - s_ky_i}^2\nonumber\\
        &= \sum_{i:a_i = k} \pts{\underbrace{\epsilon_{a_i,i} - \eta s_{a_i}\e_i^\top\X \X^\top \D(\bbeta_{a_i}^{(0)})(h_{\bTheta^{(0)}}(\X) - \y) - |y_i|}_{\eqqcolon T_i}}^2,\label{eq:k_diff}
    \end{align}
    where we substitute $\beta_{k,i}^{(0)} = \beta_{a_i,i}^{(0)} = \epsilon_{a_i,i}$ for $k=a_i$, and $s_{a_i} \cdot y_i > 0$. Next, we bound the term\\ $T_i\coloneqq \epsilon_{a_i,i} - \eta s_{a_i}\e_i^\top\X \X^\top \D(\bbeta_{a_i}^{(0)})(h_{\bTheta^{(0)}}(\X) - \y) - |y_i|$ for all $i\in[n]$. We have
        \begin{align*}
            T_i &=\epsilon_{a_i,i} - \eta s_{a_i}\e_i^\top\X \X^\top \D(\bbeta_{a_i}^{(0)})(h_{\bTheta^{(0)}}(\X) - \y) - |y_i|\\
            &= (\epsilon_{a_i,i} - |y_i|) -\eta s_{a_i}\e_i^\top\mts{\nnorm[1]{\blambda}\I + \pts{\XX - \nnorm[1]{\blambda}\I}}\D(\bbeta_{a_i}^{(0)})\pts{h_{\bTheta^{(0)}}(\X) -\y}\\
            &= (\epsilon_{a_i,i} - |y_i|) -\eta s_{a_i}\nnorm[1]{\blambda}\pts{s_{a_i}\epsilon_{a_i,i} - s_{a_i} \sum_{k:s_k\cdot y_i < 0}\epsilon_{k,i} - y_i}\\
            &\myquad[8] - \eta s_k\e_i^\top\pts{\XX - \nnorm[1]{\blambda}\I}\D(\bbeta_k^{(0)})\pts{h_{\bTheta^{(0)}}(\X) -\y}\\
            &= (1-\eta\nnorm[1]{\blambda})\epsilon_{a_i,i} - (1 -\eta\nnorm[1]{\blambda})|y_i| + \eta \nnorm[1]{\blambda}\sum_{k:s_k\cdot y_i < 0}\epsilon_{k,i}\\
            &\myquad[8] - \eta s_k\e_i^\top\pts{\XX - \nnorm[1]{\blambda}\I}\D(\bbeta_k^{(0)})\pts{h_{\bTheta^{(0)}}(\X) -\y},
        \end{align*}
        by applying $h_{\bTheta^{(0)}}(\x_i) = s_{a_i}\epsilon_{a_i,i} - s_{a_i} \sum_{k:s_k\cdot y_i < 0}\epsilon_{k,i}$ from Equation~\eqref{eq:h_zero}. Since the step size assumption guarantees that $\frac{1}{CC_g\nnorm[1]{\blambda}} \leq \eta \leq \frac{1}{C_g\nnorm[1]{\blambda}}$, and $\epsilon_{k,i} \leq \frac{1}{C_{\alpha}m}y_{\min}$, we have 
        \begin{align*}
            (1-\eta\nnorm[1]{\blambda})\epsilon_{a_i,i} - (1 -\eta\nnorm[1]{\blambda})|y_i| &+ \eta \nnorm[1]{\blambda}\sum_{k:s_k\cdot y_i < 0}\epsilon_{k,i}\\
            &\leq \epsilon_{a_i,i} - (1 - \eta\nnorm[1]{\blambda})|y_i| + \eta\nnorm[1]{\blambda} \sum_{k:s_k\cdot y_i < 0}\epsilon_{k,i} \\
            &\leq \pts{\frac{1}{m} + \frac{1}{C_g}}\frac{1}{C_{\alpha}}y_{\min} - \pts{1 - \frac{1}{C_g}}y_{\min}\\
            &< 0,
        \end{align*}
        with $C_{\alpha} \gtrsim C_g^2$. Hence, in order to upper bound $T_i^2$, it suffices to find the lower bound for $T_i$. We have
        \begin{align*}
            T_i &= (1-\eta\nnorm[1]{\blambda})\epsilon_{a_i,i} - (1 -\eta\nnorm[1]{\blambda})|y_i| + \eta \nnorm[1]{\blambda}\sum_{k:s_k\cdot y_i < 0}\epsilon_{k,i}\\
            &\myquad[8] - \eta s_k\e_i^\top\pts{\XX - \nnorm[1]{\blambda}\I}\D(\bbeta_k^{(0)})\pts{h_{\bTheta^{(0)}}(\X) -\y}\\
            &\geq -|y_i| - \eta \nnorm[2]{\XX - \nnorm[1]{\blambda}\I}\nnorm[2]{h_{\bTheta^{(0)}}(\X) -\y},
        \end{align*}
        where the inequality drops the positive terms $(1-\eta\nnorm[1]{\blambda})\epsilon_{a_i,i}$, $\eta\nnorm[1]{\blambda}|y_i|$, and\\$\eta \nnorm[1]{\blambda}\sum_{k:s_k\cdot y_i < 0}\epsilon_{k,i}$. We again upper bound $\nnorm[2]{\XX - \nnorm[1]{\blambda}\I}$ by Corollary~\ref{cor:concentration_op_norm}. With probability at least $1 - 2 \exp(-n(Cc - \ln 9))$, we have
        \begin{align*}
            T_i &\geq -|y_i| - \eta \cdot C\nnorm[1]{\blambda}\cdot \max \pts{\sqrt{\frac{n}{d_2}}, \frac{n}{d_{\infty}}}\nnorm[2]{h_{\bTheta^{(0)}}(\X) -\y}\\
            &\geq -|y_i| - \frac{C}{C_g} \cdot \max \pts{\sqrt{\frac{n}{d_2}}, \frac{n}{d_{\infty}}}\nnorm[2]{h_{\bTheta^{(0)}}(\X) -\y},
        \end{align*}
        by applying $\eta \leq \frac{1}{C_g\nnorm[1]{\blambda}}$. Finally, we apply the upper bound that $\nnorm[2]{h_{\bTheta^{(0)}}(\X)} \leq \sum_{k=1}^m\nnorm[2]{\bepsilon_k} \leq \frac{\sqrt{n}}{C_{\alpha}}y_{\min}$ and $\nnorm[2]{\y}\leq \sqrt{n}y_{\max}$, and Assumption~\ref{asm:high_dim_data} ensures that $d_2 \geq C_0^2\frac{n^2 y_{\max}^2}{y_{\min}^2}$ and $d_{\infty} \geq C_0\frac{n^{1.5} y_{\max}}{y_{\min}}$. We have
        \begin{align*}
             T_i &\geq -|y_i| - \frac{C}{C_g} \max \pts{\sqrt{\frac{n}{d_2}}, \frac{n}{d_{\infty}}}\pts{\frac{\sqrt{n}}{C_{\alpha}}y_{\min} + \sqrt{n}y_{\max}}\\
             &\geq -|y_i| - \frac{Cy_{\min}}{C_gC_0y_{\max}}\pts{\frac{1}{C_{\alpha}}y_{\min} + y_{\max}}\\
             &\geq -|y_i| \pts{1 + \frac{2C}{C_gC_0}}\\
             &\geq -C_y |y_i|,
        \end{align*}
        with the choice of $C_y \geq 2$. Substituting $T_i^2 \leq C_y^2 y_i^2$ into Equation~\eqref{eq:k_diff}, we have
        \begin{align*}
            \nnorm[2]{\bbeta_{k,S_k}^{(1)} -s_k\y_{S_k}}^2 &\leq \sum_{i:a_i=k} C_y^2 y_i^2 = C_y^2 \nnorm[2]{\y_{S_k}}^2.
        \end{align*}
        As a result, we conclude that $ \nnorm[2]{\bbeta_{k,S_k}^{(1)} -s_k\y_{S_k}} \leq C_y \nnorm[2]{\y_{S_k}}$ as required.

    \item We verify the norm bounds on the dual variables. By the triangle inequality, we work from Equation~\eqref{eq:alpha_one} to get
    \begin{align*}
        \nnorm[2]{\balpha_{k}^{(1)}} &= \nnorm[2]{\eta\pts{s_k\D(\bbeta_k^{(0)})\pts{\y-h_{\bTheta^{(0)}}(\X)}+\frac{1}{\eta}\XXi\bbeta_k^{(0)}}} \\
        &\leq \eta \mts{\nnorm[2]{\y} + \nnorm[2]{h_{\bTheta^{(0)}}(\X)} + \frac{1}{\eta}\nnorm[2]{\XXi}\nnorm[2]{\bbeta_k^{(0)}}}\\
        &= \eta \mts{\nnorm[2]{\y} + \nnorm[2]{h_{\bTheta^{(0)}}(\X)} + \frac{1}{\eta}\nnorm[2]{\XXi}\nnorm[2]{\frac{1}{C_g}\A_k\y + \bepsilon_k}},
    \end{align*}
    by substituting $\bbeta_k^{(0)} = \frac{1}{C_g}\A_k\y + \bepsilon_k$. Using $\nnorm[2]{\y} \leq \sqrt{n}y_{\max}$, $\nnorm[2]{h_{\bTheta^{(0)}}(\X)} \leq \sum_{k=1}^m\nnorm[2]{\bepsilon_k} \leq \frac{\sqrt{n}}{C_{\alpha}}y_{\min}$, $\nnorm[2]{\XXi} \leq \frac{C_g}{\nnorm[1]{\blambda}}$, $\epsilon_{k,i} \leq \frac{1}{C_{\alpha}m}y_{\min}$, and $\frac{1}{CC_g\nnorm[1]{\blambda}} \leq \eta \leq \frac{1}{C_g\nnorm[1]{\blambda}}$, we have
    \begin{align*}
        \nnorm[2]{\balpha_{k}^{(1)}} &\leq \frac{1}{C_g\nnorm[1]{\blambda}}\mts{\sqrt{n}y_{\max} +  \frac{\sqrt{n}}{C_{\alpha}}y_{\min} + CC_g\nnorm[1]{\blambda}\cdot\frac{C_g}{\nnorm[1]{\blambda}}\cdot \pts{\frac{\sqrt{n}y_{\max}}{C_g} + \frac{\sqrt{n}}{C_{\alpha}m}y_{\min}}}\\
        &\leq \frac{1}{\nnorm[1]{\blambda}}\pts{3\sqrt{n}y_{\max}}\\
        &\leq \frac{C_{\alpha}\sqrt{n}y_{\max}}{\nnorm[1]{\blambda}},
    \end{align*}
    with $C_{\alpha}\gtrsim \max\{C_g^2, C_yC_g\}$. Thus, condition~(\ref{cond:alpha_norm_bounds_m}) holds at $t = 1$.

    \item Since we have shown that $\alpha_{k,j}^{(1)} \leq -\frac{y_{\min}}{C_{\alpha}\nnorm[1]{\blambda}}$ and $\nnorm[2]{\balpha_{k}^{(1)}} \leq \frac{C_{\alpha}\sqrt{n}y_{\max}}{\nnorm[1]{\blambda}}$ for all $j\in[n]$ with $k\neq a_j$, by Lemma~\ref{lem:neg_stay_neg}, it follows that $\beta_{k,j}^{(1)} \leq 0$ for all $j\in[n]$ with $k \neq a_j$.

\end{enumerate}
We have shown that at iteration $t=1$ the conditions in Lemma~\ref{lem:primal_label_same_sign_m} are satisfied, and by induction, these conditions will also hold for $t \geq 1$. As a result, $\w_k$ is trained with only predefined active examples starting from the iteration $t=0$, and it is equivalent to linear regression using only active examples with initialization $\w_{k}^{(1)}= \eta\X^\top\pts{s_k\D(\bbeta_k^{(0)})\pts{\y-h_{\bTheta^{(0)}}(\X)}+\frac{1}{\eta}\XXi\bbeta_k^{(0)}}$. Finally, since $\w_{k}$ is trained on disjoint subset of examples by Assumption~\ref{asm:disjoint_subset}, by Lemma~\ref{lem:conv_step_size}, $\w_{k}^{(\infty)}$ satisfies
    \begin{align*}
        \w_{k}^{(\infty)} =\uargmin{\w\in\{\w:\X_{S_k}\w =\y_{S_k}\}} \nnorm[2]{\w - \w_{k}^{(1)}}.
    \end{align*}
    This completes the proof of Theorem~\ref{thm:multiple_relu_gd_m_high_dim_implicit_bias}.
\end{proof}
\subsection{Proof of Theorem~\ref{thm:multiple_relu_m_approx_to_w_star} (Implicit Bias Approximation to \texorpdfstring{$\w^\star$}{w*})}\label{app:proof_multiple_relu_m_approx_to_w_star}

\begin{proof}(Theorem~\ref{thm:multiple_relu_m_approx_to_w_star})
    We restate the definition of $\w^\star$ in Equation~\eqref{eq:multiple_relu_m_minimum_norm_sol}.
    \begin{align}\label{eq:multiple_relu_m_approx_to_w_star_rep}
        &\{\w_k^\star\}_{k=1}^m = {} \uargmin{\{\w_k\}_{k=1}^m}\frac{1}{2}\sum_{k=1}^m \nnorm[2]{\w_k}^2\\
    \text{s.t. } &\sum_{k=1}^m s_k\sigma(\w_k^\top \x_i) = y_i, \text{ for all } i\in[n].\nonumber
    \end{align}
    Recall that the gradient descent limit $\{\w_k^{(\infty)}\}_{k=1}^m$ satisfies the same set of constraints: it interpolates all examples. Consequently, both $\{\w_k^{(\infty)}\}_{k=1}^m$ and $\{\w_k^\star\}_{k=1}^m$ are feasible solutions to~\eqref{eq:multiple_relu_m_approx_to_w_star_rep}. We show that the norm difference between $\w_k^{(\infty)}$ and $\w_k^\star$ can be upper bounded by 2 times the norm of $\w_k^{(\infty)}$.
    \begin{align*}
        \sum_{k=1}^m\nnorm[2]{\w_k^{(\infty)} - \w_k^\star}^2\leq 2\sum_{k=1}^m\nnorm[2]{\w_k^{(\infty)}}^2 + 2\sum_{k=1}^m\nnorm[2]{\w_k^\star}^2 \leq 4 \sum_{k=1}^m\nnorm[2]{\w_k^{(\infty)}}^2,
    \end{align*}
    where it follows the definition of~\eqref{eq:multiple_relu_m_approx_to_w_star_rep}. By Lemma~\ref{lem:primal_label_same_sign_m}, we have the upper bound for $\nnorm[2]{\w_k^{(\infty)}}^2$ as
    \begin{align*}
        \nnorm[2]{\w_k^{(\infty)}}^2 = \balpha_k^{(\infty)\top}\XX\balpha_k^{(\infty)} \leq \mu_1(\XX)\nnorm[2]{\balpha_k^{(\infty)}}^2 \leq C_g\nnorm[1]{\blambda} \cdot \frac{C_{\alpha}^2ny_{\max}^2}{\nnorm[1]{\blambda}^2} =\frac{C_gC_{\alpha}^2ny_{\max}^2}{\nnorm[1]{\blambda}}.
    \end{align*}
    As a result, we have
    \begin{align*}
        \sum_{k=1}^m\nnorm[2]{\w_k^{(\infty)} - \w_k^\star}^2\leq \frac{4C_gC_{\alpha}^2mny_{\max}^2}{\nnorm[1]{\blambda}}.
    \end{align*}
\end{proof}
\newpage
\section{Simulations} \label{app:simulation}

In this section, we present exploratory visualizations of the evolution of the primal variables at iteration checkpoints in settings that violate the assumptions underlying our theoretical results.

\subsection{Moderate-Dimensional Data and Single ReLU Model}

\begin{figure}[H]
\centering     
\subfigure[Initialization 1 ($t=0$).]{\includegraphics[width=65mm]{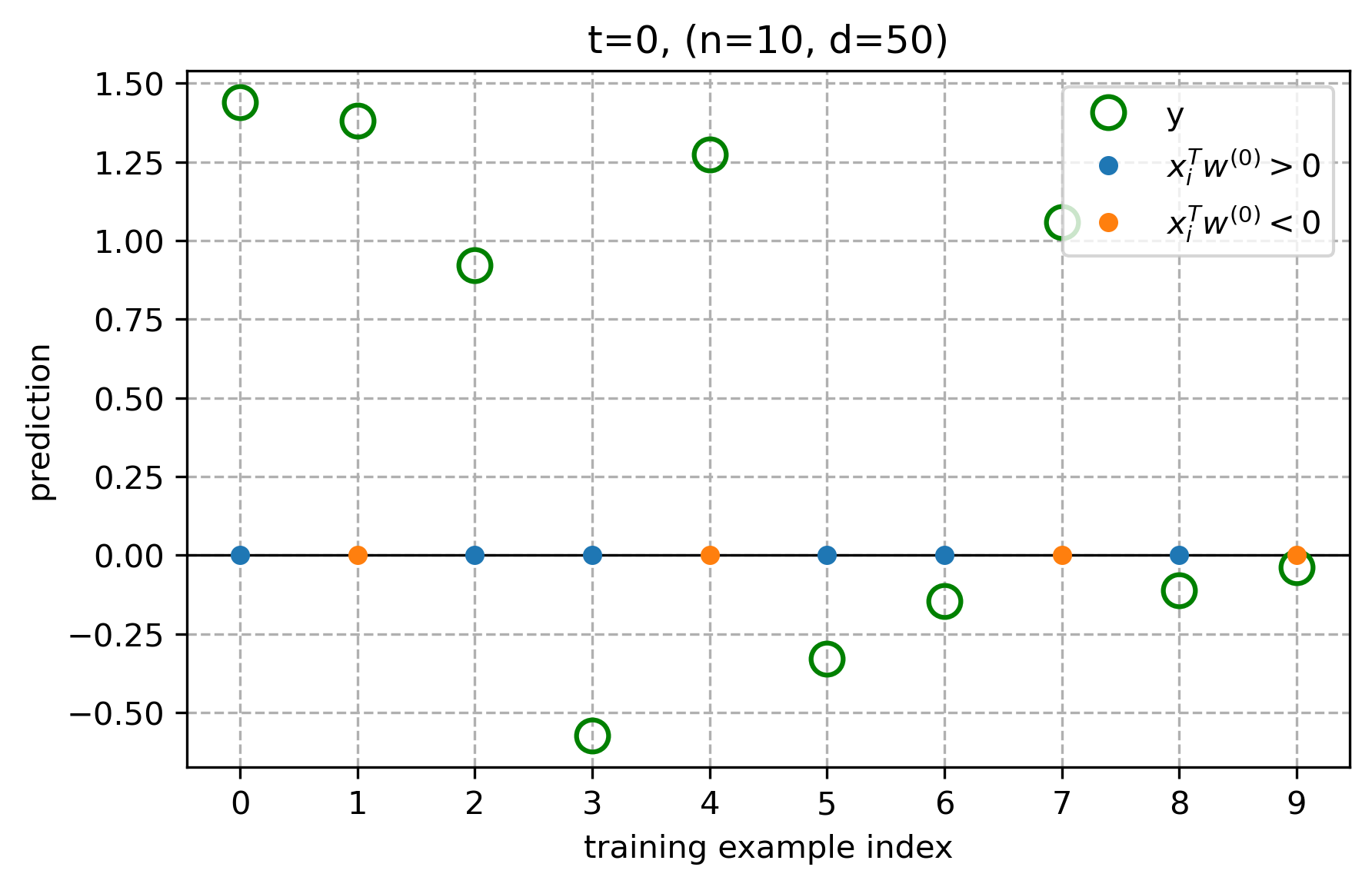}}
\subfigure[Initialization 2 ($t=0$).]{\includegraphics[width=65mm]{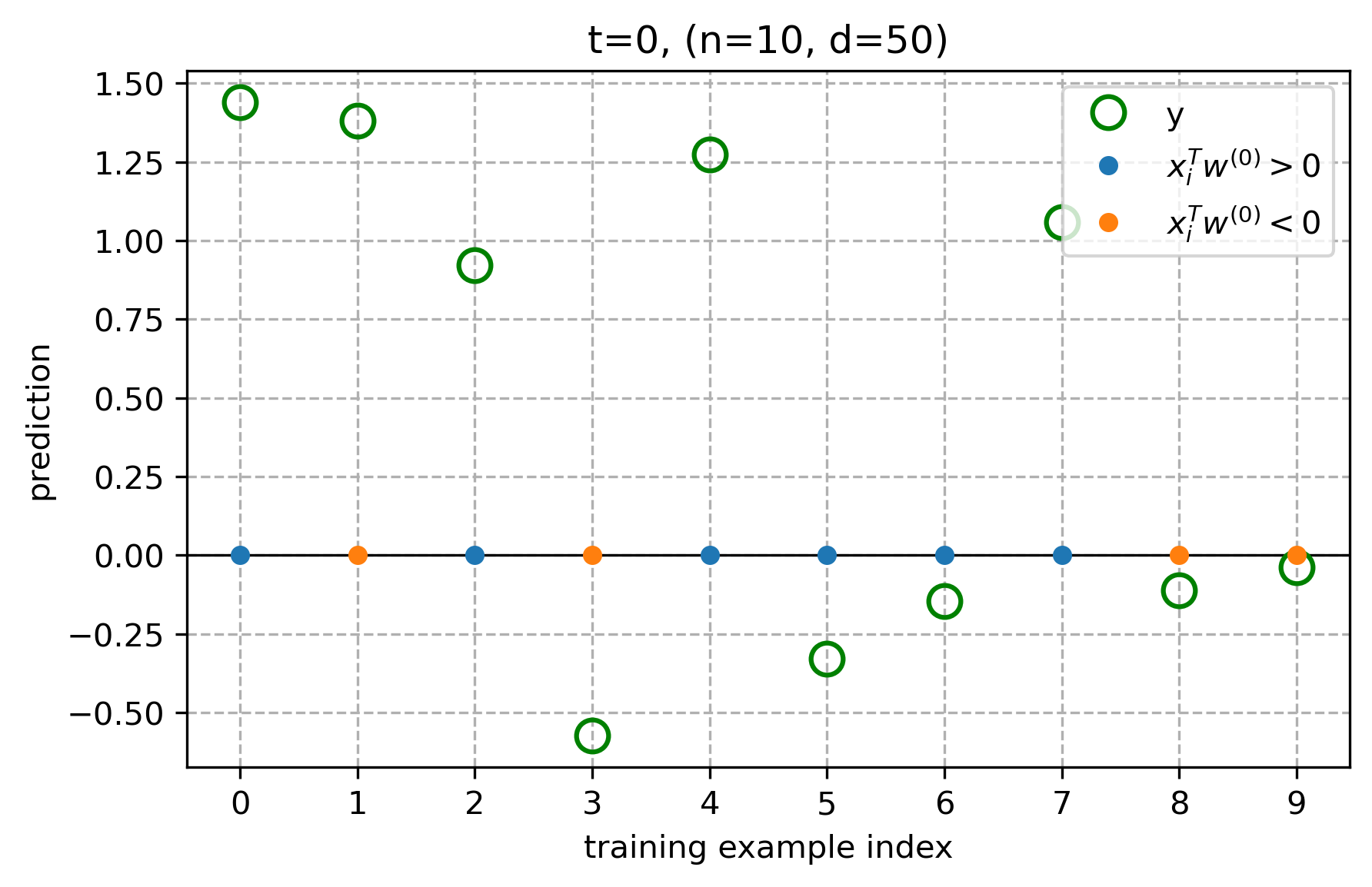}}
\subfigure[Initialization 1 ($t=19$).]{\includegraphics[width=65mm]{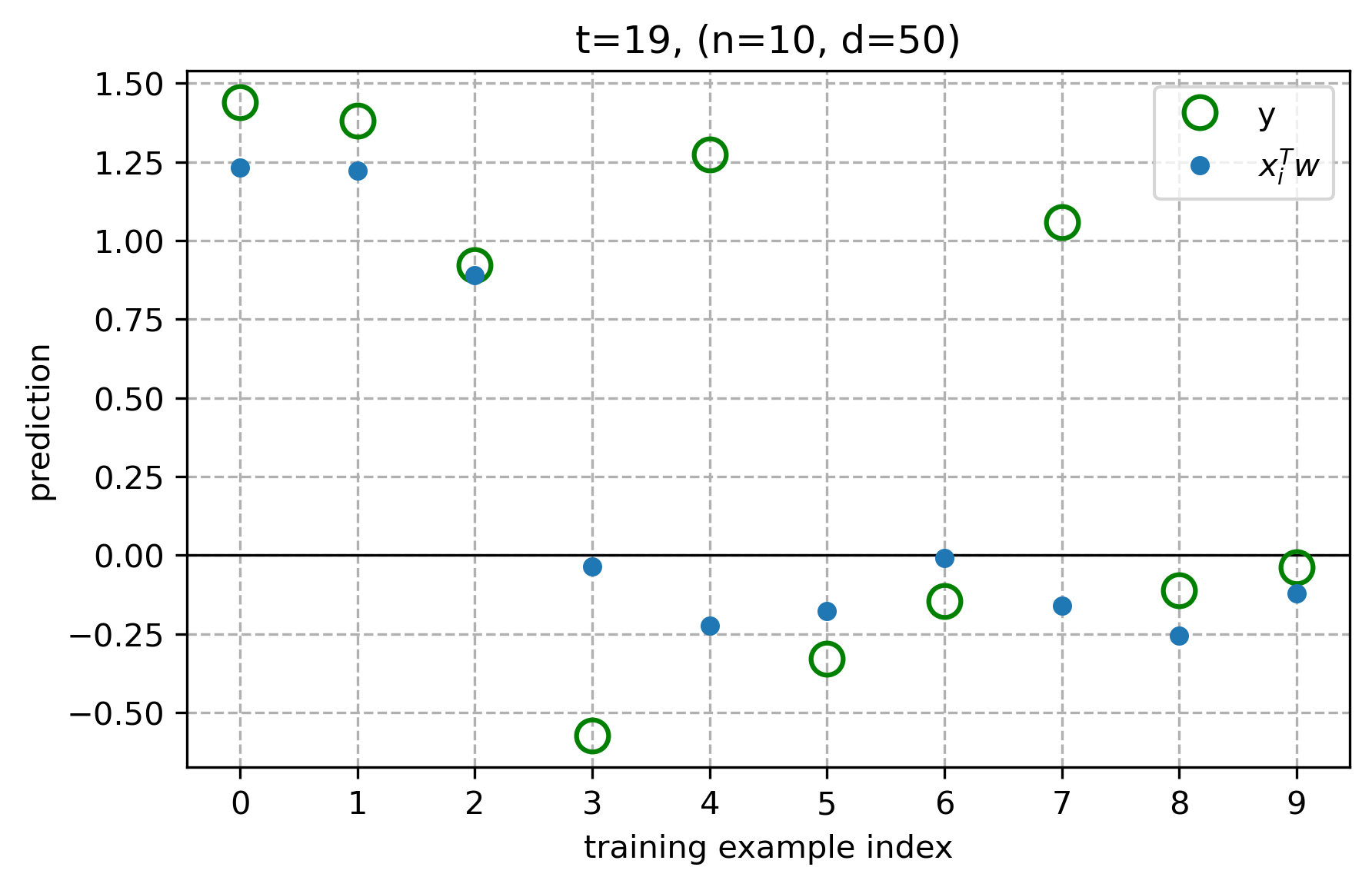}}
\subfigure[Initialization 2 ($t=19$).]{\includegraphics[width=65mm]{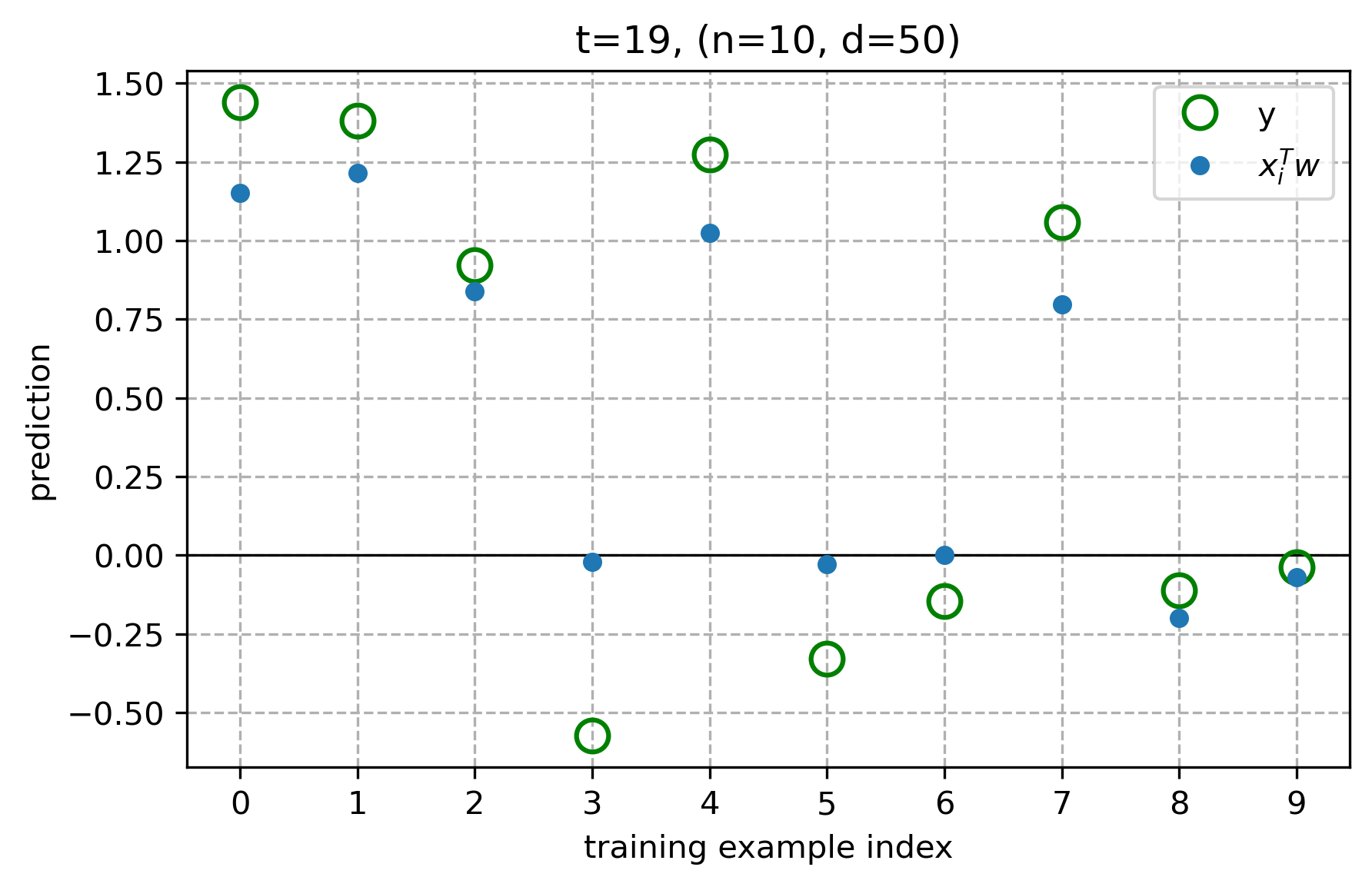}}
\subfigure[Initialization 1 ($t=80$).]{\includegraphics[width=65mm]{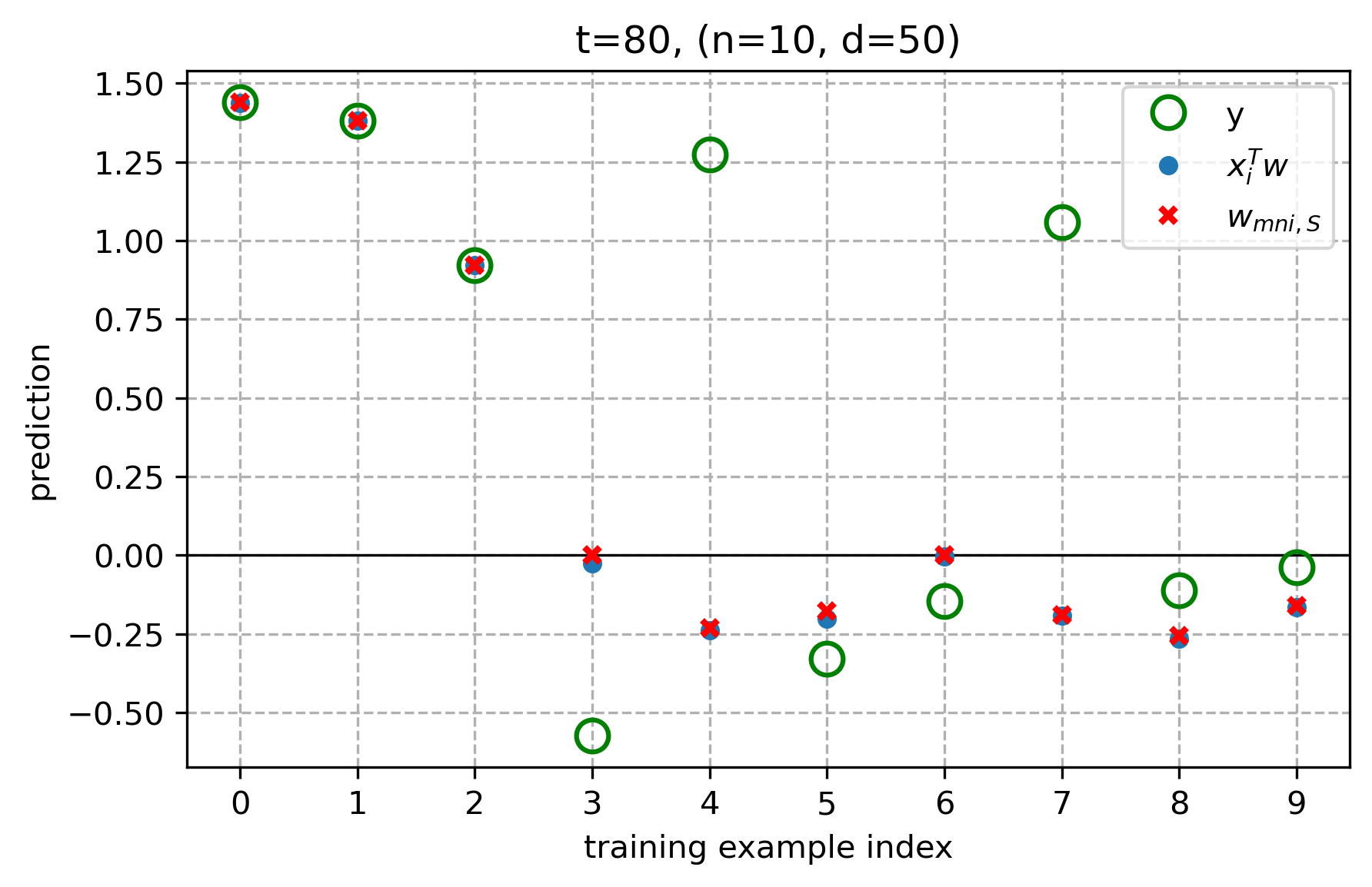}}
\subfigure[Initialization 2 ($t=80$).]{\includegraphics[width=65mm]{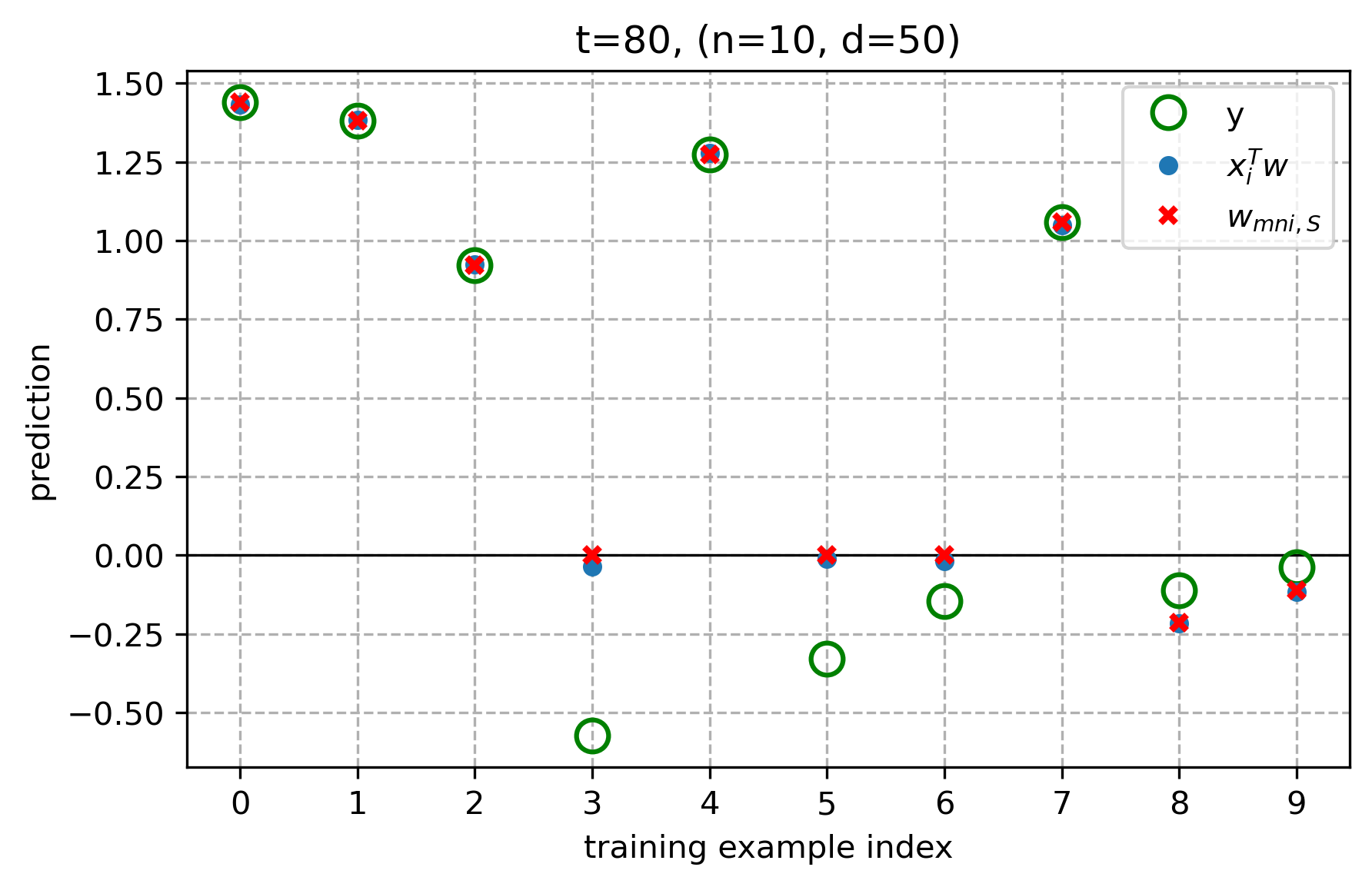}}
\caption{We illustrate the prediction dynamics of gradient descent for a single ReLU model under different random initializations when $d$ is comparable with $n$. In both cases, with sufficiently small step size, the final solution converges to a linear minimum-$\ell_2$-norm interpolator on some subset of the training examples,~i.e. of the form $\w_{\text{linear-MNI},S} = \X_S^\top(\X_S\X_S^\top)^{-1}\tilde{\y}_S$, where $\tilde{y}_{S,i}=\max\{y_i,0\}$. In contrast to the high-dimensional regime, \emph{different initializations lead to different subsets $S$}, indicating that ReLU training implicitly performs an example ``selection'' process, that is initialization-dependent, rather than fitting all positively-labeled samples. The experiment uses $n=10$, $d=50$, $\x\sim\mathcal{N}(\zero,\I)$, $y\sim\mathcal{N}(0,1)$, $\w^{(0)}\sim\mathcal{N}(\zero,2\times10^{-6}\I)$, and $\eta= 10^{-4}$.}\label{fig:approx_mni_s}
\end{figure}

\subsection{Gradient Descent Dynamics of Two ReLU Models}
\vspace{-2em}
\phantom{placeholder}
\begin{figure}[H]
\centering     
\subfigure[Two ReLU model gradient descent dynamic ($t=0$).]{\includegraphics[width=0.68\textwidth]{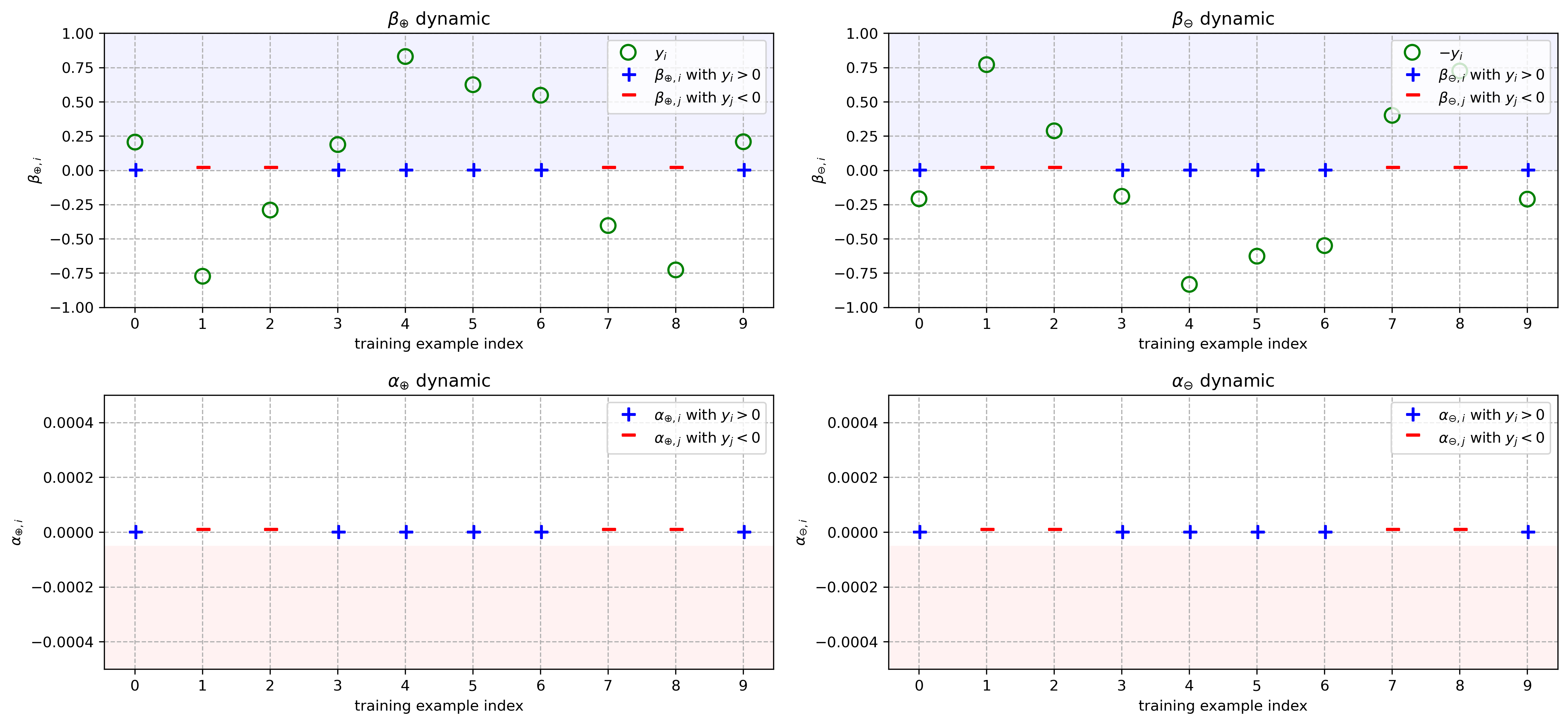}}
\subfigure[Two ReLU model gradient descent dynamic ($t=1$).]{\includegraphics[width=0.68\textwidth]{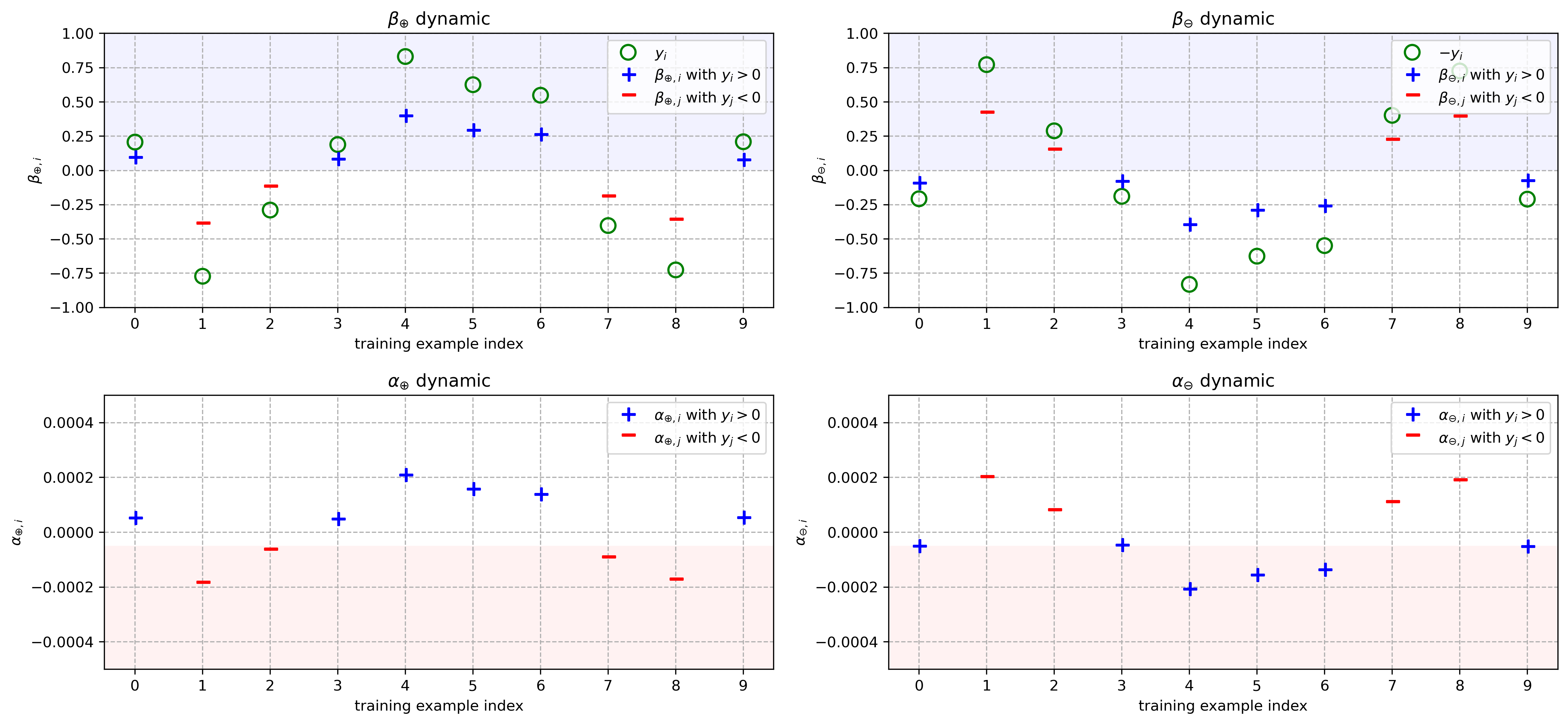}}
\subfigure[Two ReLU model gradient descent dynamic ($t=17$).]
{\includegraphics[width=0.68\textwidth]{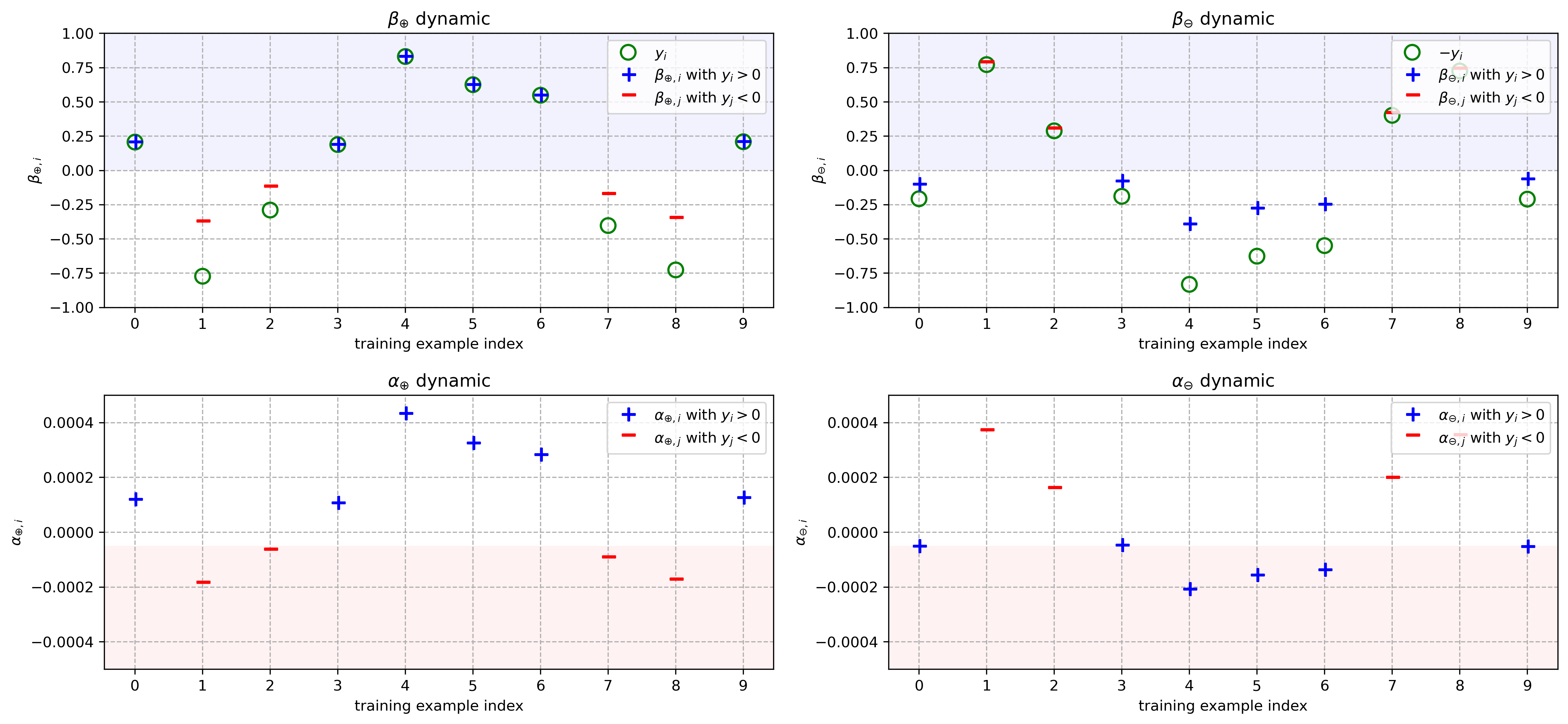}}
\caption{Simulation illustrating Theorem~\ref{thm:multiple_relu_gd_high_dim_implicit_bias}.
In the high-dimensional regime and under our ``all-positive'' initialization, after the first gradient step, examples with positive labels remain active while examples with negative labels become inactive, consistent with Lemma~\ref{lem:primal_label_same_sign_two}. The blue region shows primal variables that remain positive over training, whereas the red region corresponds to dual variables that are sufficiently negative and remain unchanged. As training proceeds, $\w_{\oplus}$ fits all positively labeled examples and $\w_{\ominus}$ fits all negatively labeled examples. The experiment uses $n=10$, $d=2000$, features $\x\sim\mathcal{N}(\zero,\I)$, and labels satisfying $|y|\sim\mathcal{U}(0.1,1)$ with $\mathrm{sign}(y)$ uniformly distributed over $\{\pm 1\}$.}\label{fig:good_init_high_dim}
\end{figure}

\begin{figure}[H]
\centering     
\subfigure[Two ReLU model gradient descent dynamic ($t=0$).]{\includegraphics[width=0.74\textwidth]{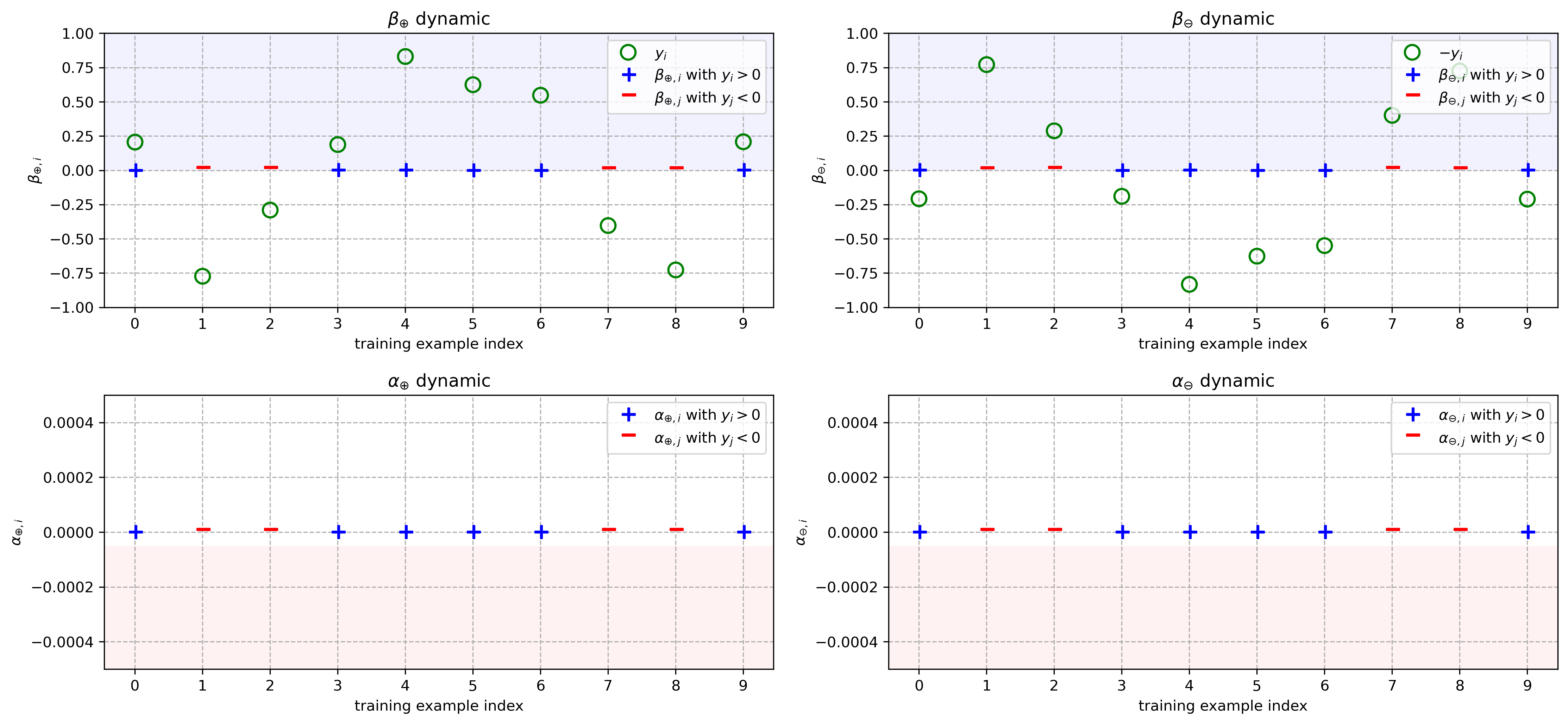}}
\subfigure[Two ReLU model gradient descent dynamic ($t=1$).]{\includegraphics[width=0.74\textwidth]{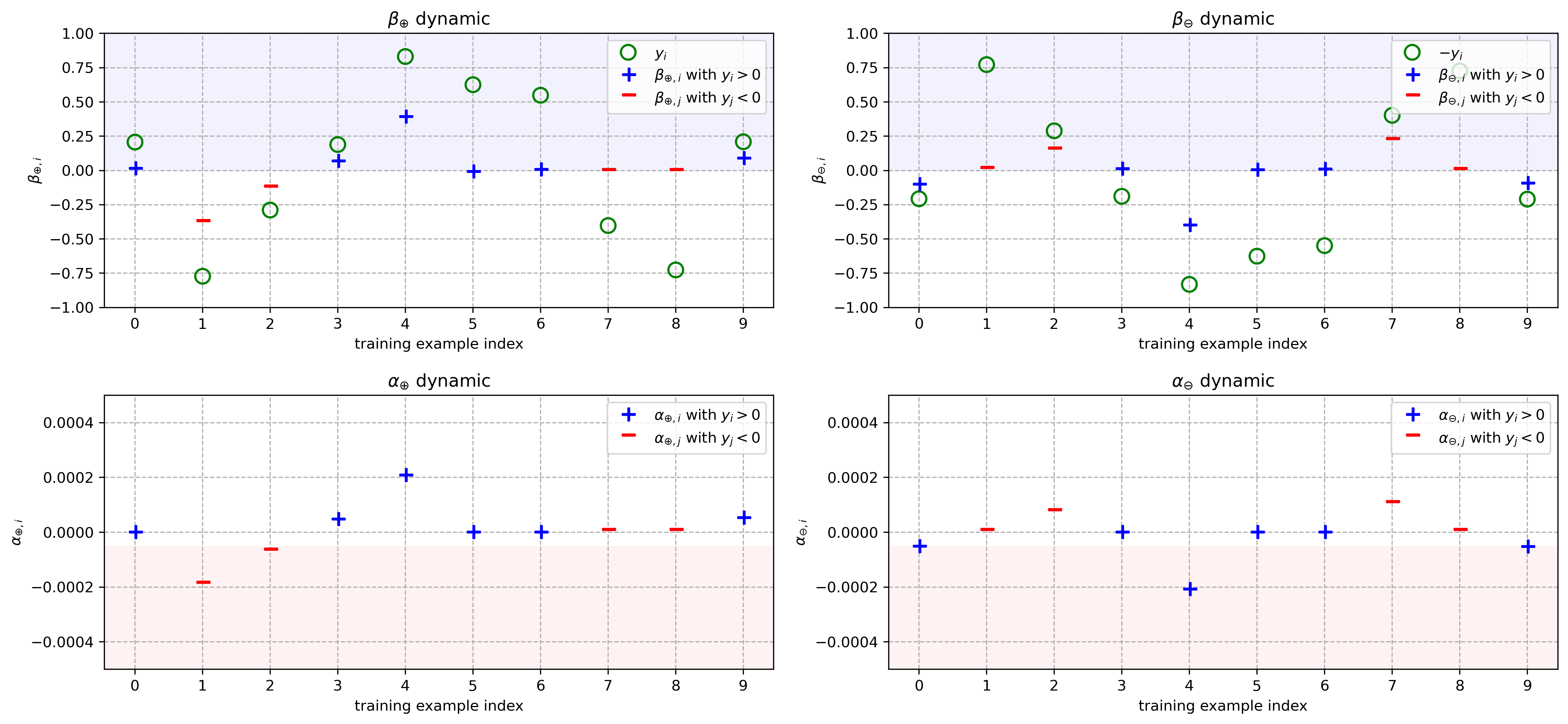}}
\subfigure[Two ReLU model gradient descent dynamic ($t=40$).]
{\includegraphics[width=0.74\textwidth]{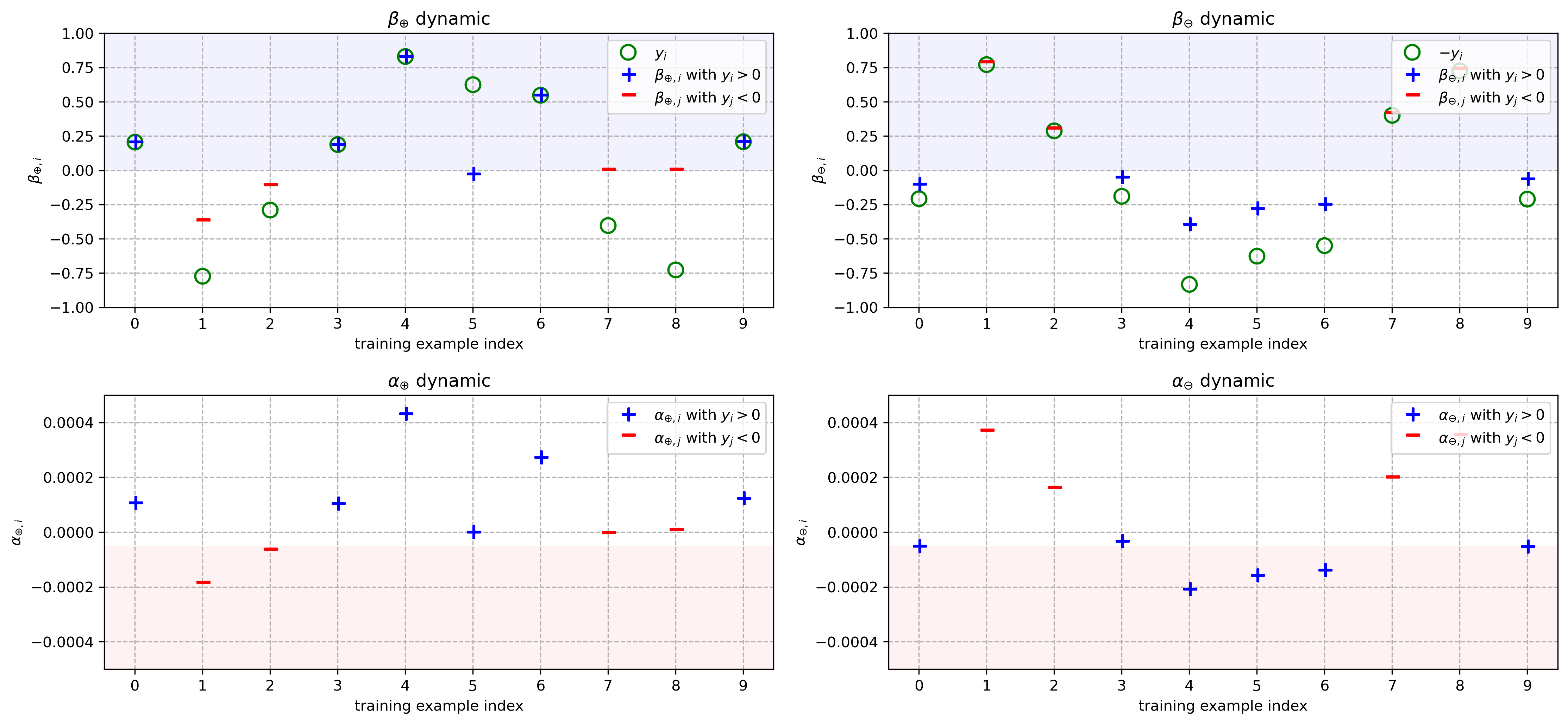}}
\caption{Simulation with random initialization in the high-dimensional regime, which violates our initialization assumption in Theorem~\ref{thm:multiple_relu_gd_high_dim_implicit_bias}.
Under random initialization, the sufficient conditions of Lemma~\ref{lem:primal_label_same_sign_two} are violated at the first gradient step. As a result, positively labeled examples do not all remain in the active (blue) regime (e.g., example no. 5), nor do negatively labeled examples consistently enter the inactive (red) regime (e.g., example no. 7). \emph{Consequently, during training, this model fails to converge to a global minimum.} The experiment uses $n=10$, $d=2000$, features $\x\sim\mathcal{N}(\zero,\I)$, and labels satisfying $|y|\sim\mathcal{U}(0.1,1)$ with $\mathrm{sign}(y)$ uniformly distributed over $\{\pm 1\}$.}\label{fig:bad_init_high_dim}
\end{figure}

\begin{figure}[H]
\centering     
\subfigure[Two ReLU model gradient descent dynamic ($t=0$).]{\includegraphics[width=0.74\textwidth]{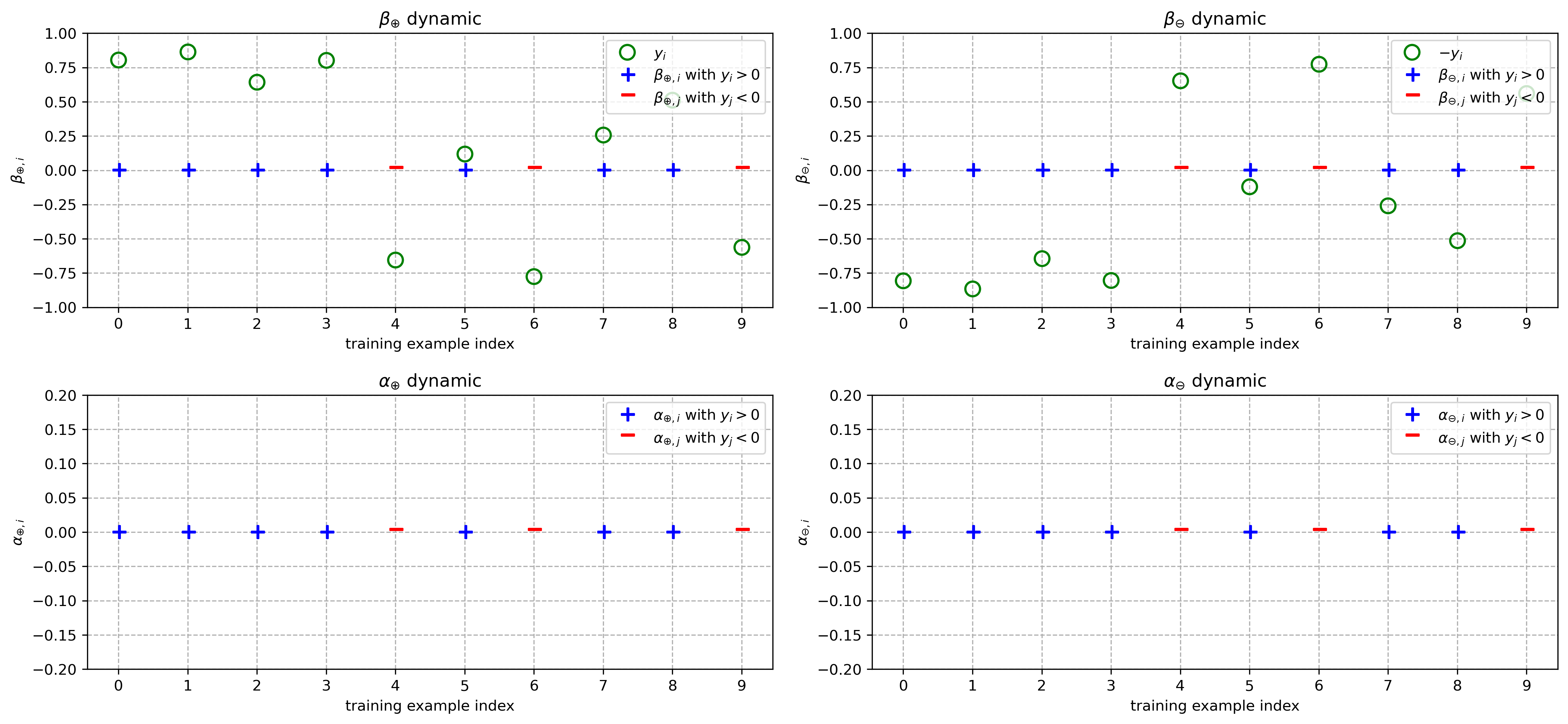}}
\subfigure[Two ReLU model gradient descent dynamic ($t=1$).]{\includegraphics[width=0.74\textwidth]{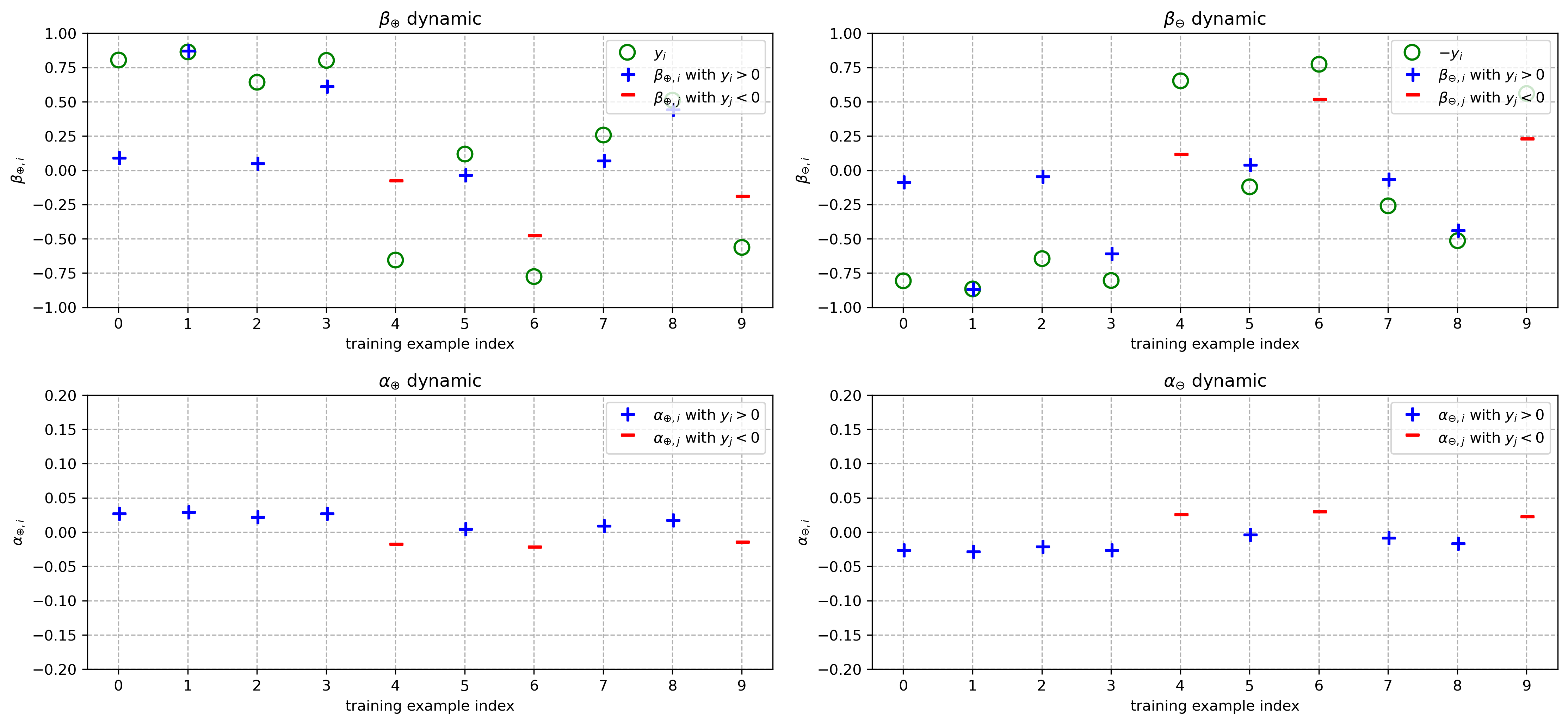}}
\subfigure[Two ReLU model gradient descent dynamic ($t=40$).]
{\includegraphics[width=0.74\textwidth]{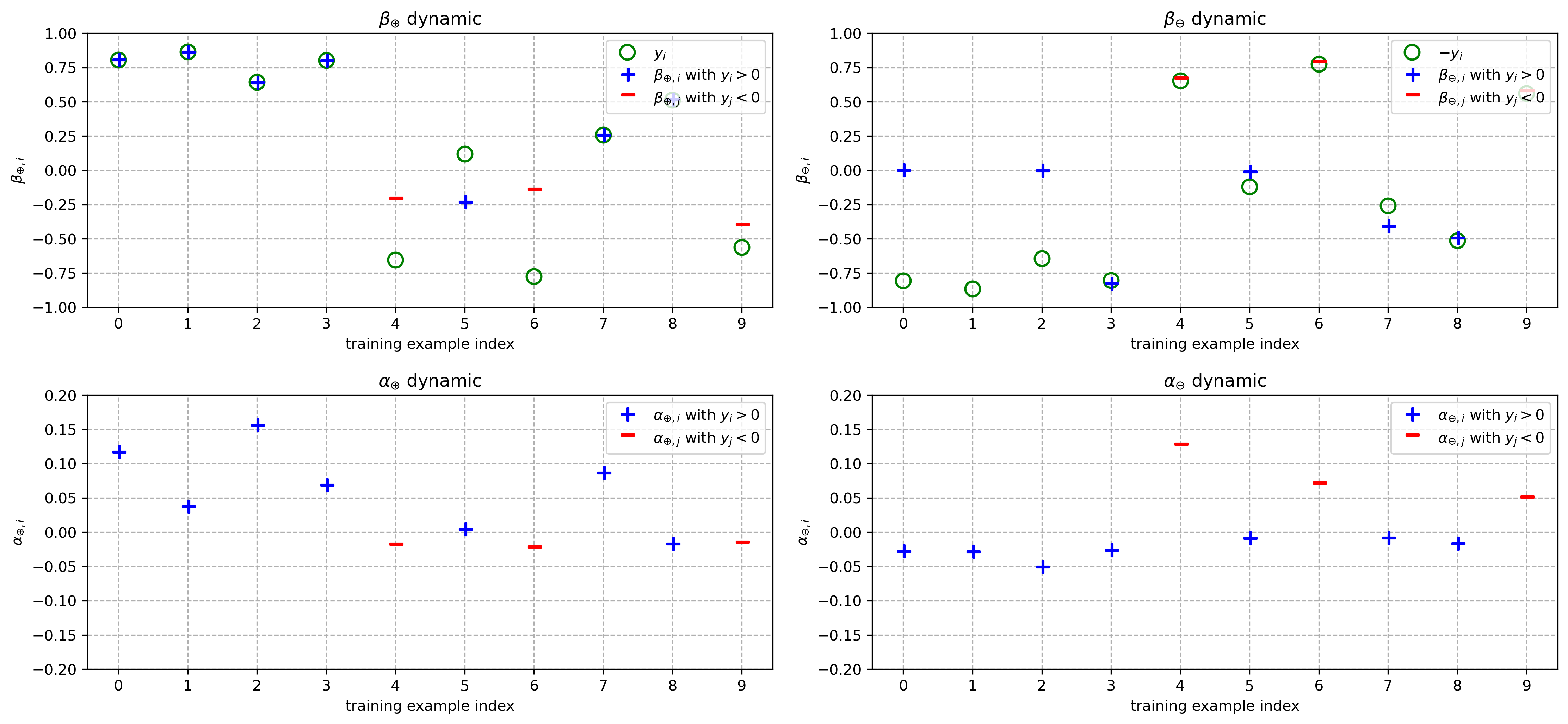}}
\caption{Simulation with all-positive initialization outside the high-dimensional regime.
When the data dimension is not sufficiently large, the feature vectors are no longer approximately orthogonal. As a result, the clear separation into active (blue) and inactive (red) regimes observed in Figures~\ref{fig:good_init_high_dim} and~\ref{fig:bad_init_high_dim} disappears. Consequently, the gradient dynamics become highly coupled across examples and are no longer analytically tractable using our high-dimensional arguments. The experiment uses $n=10$, $d=15$, features $\x\sim\mathcal{N}(\zero,\I)$, and labels satisfying $|y|\sim\mathcal{U}(0.1,1)$ with $\mathrm{sign}(y)$ uniformly distributed over $\{\pm 1\}$.}
\end{figure}\label{fig:good_init_low_dim}

\subsection{Gradient Descent Dynamics of Multiple ReLU Models}
\begin{figure}[H]
\centering     
\subfigure[Multiple ReLU model gradient descent dynamic ($t=0$).]{\includegraphics[width=0.95\textwidth]{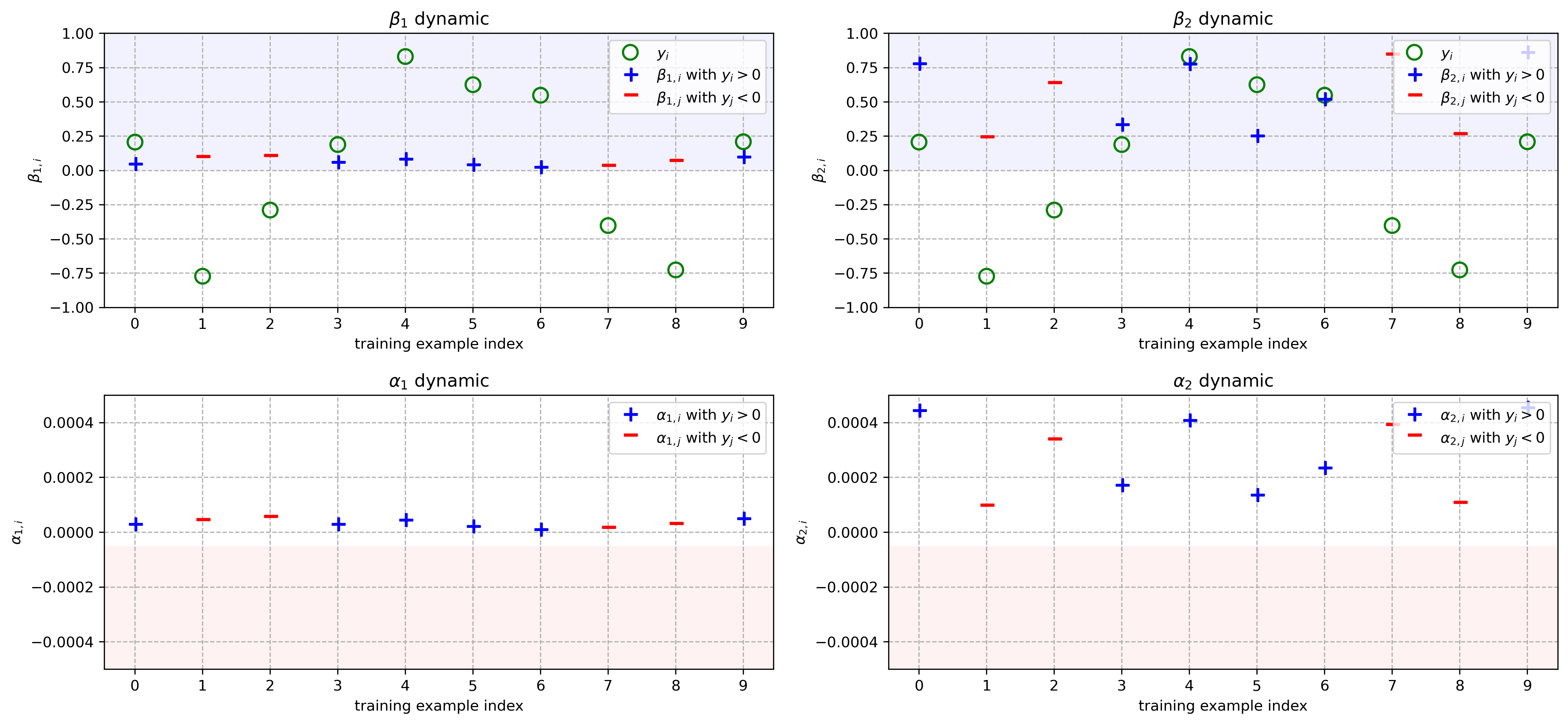}}
\subfigure[Multiple ReLU model gradient descent dynamic ($t=1$).]{\includegraphics[width=0.95\textwidth]{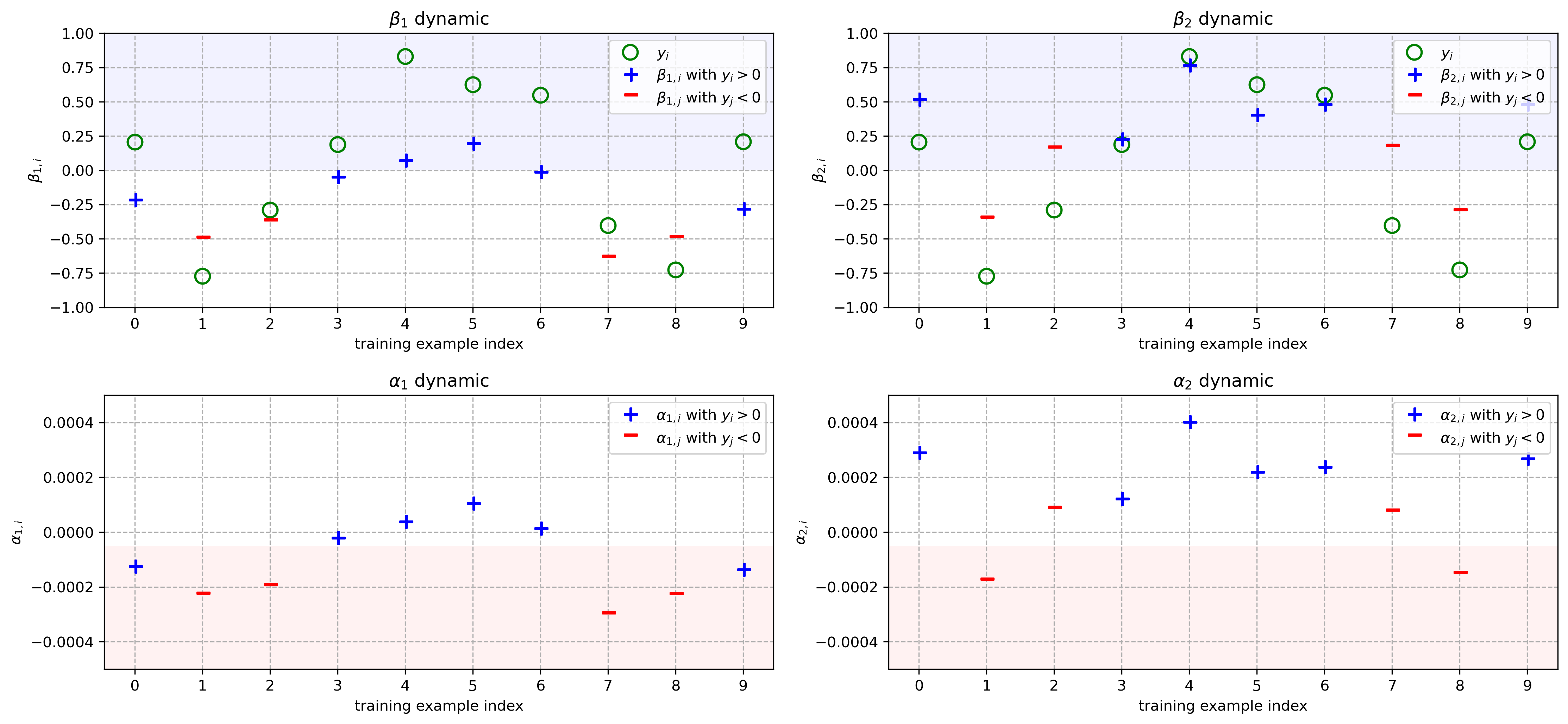}}
\caption{Failure of stable activation patterns in multiple ReLU models.
We illustrate the training dynamics of a multiple ReLU model when multiple neurons share the same sign. In this setting, the sufficient conditions of Lemma~\ref{lem:primal_label_same_sign_m} are violated, and positive primal variables do not necessarily remain in the active (blue) regime throughout training (e.g. training example no. 0). As a result, the activation pattern becomes unstable, and the resulting primal dynamics are no longer tractable. The experiment uses $n=10$, $d=2000$, $m=4$, with neuron signs $s_1=s_2=1$ and $s_3=s_4=-1$, features $\x\sim\mathcal{N}(\zero,\I)$, and labels satisfying $|y|\sim\mathcal{U}(0.1,1)$ with $\mathrm{sign}(y)$ uniformly distributed over $\{\pm 1\}$.}
\end{figure}


\end{document}